\documentclass[12pt]{article}
\usepackage[margin=1in]{geometry}

\usepackage[authoryear, sort&compress, round]{natbib}
\usepackage[utf8]{inputenc} 
\usepackage[T1]{fontenc}    
\usepackage{hyperref}       
\usepackage{url}            
\usepackage{booktabs, multicol, multirow}       
\usepackage{amsfonts}       
\usepackage{amsmath}
\usepackage{nicefrac}       
\usepackage{microtype}      
\usepackage{mathtools}

\usepackage{upgreek,morenotations2,rotating}

\usepackage{arydshln}

\usepackage[toc,page]{appendix}

\title{\papertitle}

\author{
  Vivien Tran-Thien \qquad Richard Nock \\
 Google\\
{\normalsize $\{$vtranthien,richardnock$\}$@google.com} \\
}
\date{}

\begin{document}

\maketitle

\begin{abstract}
  Speculative decoding is a successful technique speeding up inference of a target autoregressive language model via a fast drafter model. Lossy speculative decoding allows a drift with respect to the target to further improve speed. Interestingly, it has been observed experimentally that the resulting model can \textit{also} beat the target \textit{quality-wise}. Our paper formally proves how such a feat is possible with a formal approach to lossy speculative decoding called \textit{mentored decoding}. To get there, we connect inference to a celebrated ML training theory, \textit{boosting}, and proceed via the generalization of mentored decoding to the whole set of $f$-divergences. We uncover key properties of mentored decoding, among which (i) the particularly appealing geometric nature of the total variation case, (ii) simple approximations for any $f$-divergence in direct relation with boosting compliance, and (iii) a \textit{divergence independent} $O(n)$ space and $O(\mathrm{sort}(n))$ time data structure built on drafter and target outputs, which allows to query the optimal parameters of the dual problem in $O(\log n)$ time and constructing optimal mentored distributions in $O(n)$ time for any $f$-divergence.
\end{abstract}

\section{Introduction}\label{sec-int}

Large language model (LLM) inference is often memory-bandwidth bound during sequential token generation. \textit{Speculative decoding} \citep{leviathan2023fast,chen2023accelerating} (SD) accelerates this process by using a smaller, faster draft model to propose candidate tokens, which are subsequently verified in parallel by the target model. Crucially, SD constrains the \textit{output distribution} to be the same as the target's, which inherently constrains the overall acceptance probability as a tight function of drafter and target. This fundamental limit is not an artifact of SD: \citet{sun2023spectr} proved that SD achieves the optimal acceptance rate under exact target distribution matching. To lift this cap, \citet{Tran-Thien_2023} first framed \textit{mentored decoding} (MD) as a constrained optimization problem maximizing draft acceptance subject to bounded divergence between target and the output. The target, authorizing deviations with respect to its output as long as they do not substantially diverge, becomes the \textit{mentor} in MD and the final output, which mixes tokens from both models, is a composite output distribution from an ensemble model. Initially, \citet{Tran-Thien_2023} used the reverse Kullback-Leibler divergence as divergence measure. To the best of our knowledge, this was the first formal attempt to alleviate SD's acceptance probability cap, even when heuristic proposals started in fact to flourish from the introduction of \textit{lenient} SD in \citet{leviathan2023fast}.

It is hard to exaggerate the experimental success of SD \citep{kim2023speculative,cai2024medusa,li2024eagle,fu2024break,yang2023inference,he2024rest, wang2025diversed, hao2026cactus} (and many others, See Section \ref{sec-rel}). Among the chorus of approval for speeding up inference, distinct voices later started to emerge, either on the fact that the drafter, even when smaller than the target, can occasionally produce high quality tokens that are then underutilized \citep{DBLP:conf/icml/LiaoXD0MSSX25}, or, more importantly, that the \textit{combination} of models achieved in the composite output can in fact beat the target on \textit{quality metrics} as well \citep{DBLP:conf/aaai/QinHPCS25, li2026targetimitationcollaborationspeculative, zhong2025speeding}. While the technical leads in the formal analysis of SD/MD inference speed-up alone are already scarce \citep{leviathan2023fast,Tran-Thien_2023, sun2023spectr, yin2024theoretical, pankratov-alistarh-2026-speculative}, there is to our knowledge no such analysis combining the possibility of speeding up inference to that of improving any quality metric on the output. \\
\noindent\textbf{Our paper} proposes the first analysis of this kind, on joint inference efficiency and model quality properties of mentored decoding as originally designed in \citet{Tran-Thien_2023}, demonstrating in particular how the MD setting achieves connections with one of machine learning (ML)'s seminal training framework especially suited to analyze the quality of model combinations: Boosting \citep{sfBF}. Our contribution to get there is threefold: (i) we substantially improve the state of the art understanding of MD, (ii) we design and analyze a new boosting approach for the connection, and (iii) we design and analyze efficient algorithms to operate this connection on the MD side.\\
\noindent\textbf{On improving MD understanding}, we use as a warmup the particular case of the total variation divergence. \citet{yin2024theoretical} partially covered the case but left aside the characterization of the set of optimal solutions. It turns out that it has absolutely remarkable properties. First, a deceptively simple geometric appeal: it is the intersection of the $n$-dimensional hyperrectangle defined by the drafter and target coordinates with the probability simplex and activating the divergence constraint. Second, a remarkable extent: this set is big enough to contain the optimal solutions for \textit{all} strictly convex divergences. Its properties bestow optimal solutions with unique appealing geometric and computational features, yielding extremely simple optimal solutions like the convex combination used in several papers \citep{yin2024theoretical,wang2025diversed,zhong2025speeding}. We then characterize the general solution for any $f$-divergence. In particular, for any strictly convex $f$, the optimal mentored distribution is unique and takes an exceptionally simple, intuitive coordinate-wise clamping form: $\pi^* = \max\{\alpha q, \min\{p, \beta q\}\}$ with $0 \le \alpha < 1 < \beta$, tracing a one-dimensional trajectory in the simplex connecting target $q$ to 
drafter $p$. Remarkably, this trajectory is independent from $f$. Additionally, for any generator $f$ differentiable in $z=1$, the curve giving the threshold $f$-divergence as a function of the optimal acceptance probability is \textit{always of right-derivative 0} at speculative decoding's "minimal" acceptance probability. Hence, it is always possible to at least reasonably improve SD's acceptance probability at negligible divergence cost to the target.\\
\noindent\textbf{On the connection with Boosting}, we first design a multiclass extension of the self-normalized boosting algorithm of \citet{DBLP:journals/ai/NockN07}, simpler and more efficient than AdaBoost yet giving rates that compete with the state of the art \citep{10.1214/aos/1024691352}. Mentored decoding being an inference technique, we develop two distinct paths connecting it with boosting. The first path is general and relies on a novel use of boosting, showing how the composite \textit{outputs} of mentored decoding "hides" a combination of \textit{models} that exhibits boosting properties. Since the MD's output depends on the drafter and target's output distributions, we ultimately deliver boosting compliance for \textit{all} related combinations of models, depending on these distributions and also on boosting's key parameter: the \textit{edge} of the drafter and target's last layers. This makes it possible to evaluate how well drafter and target "complement" each other from the output quality's standpoint, offering a concrete criterion to then select drafter and / or target from a pool of already available models -- that now abound in repositories of public and private spaces. Our second path connecting mentored decoding and boosting is specific to the total variation divergence, for which the conveniences of the set of optimal solutions make it possible to carve at reduced formal cost the distribution corresponding to the boosted ensemble of drafter and target directly in the optimal set of mentored decoding.\\
\noindent\textbf{From the standpoint of algorithms}, another remarkable invariant emerges at the level of generality of all $f$-divergences: we show that there exists a simple \textit{divergence independent} "breakpoint" data structure of size $\leq n$ (=the vocabulary size) which then allows to compute the optimal per-token acceptance and resampling probabilities, for \textit{any} $f$; the computation of this data structure takes $O(\mathrm{sort}(n))$, i.e. the complexity of sorting $n$ reals. While solving a non-linear constrained optimization problem per token might seem 
computationally demanding, the practical runtime overhead is in fact negligible. First, our data structure reduces the optimization to a single pass query over $O(\log n)$ precomputed breakpoints. This can then be used to approximately find the optimal parameters in $O(1)$ -- i.e. with guarantees on the divergence --, and this can also be used to find the \textit{exact} optimal parameters for the dual problem in $O(1)$ -- i.e. minimize the $f$-divergence subject to lowerbounded acceptance probability --. Second, in modern LLM inference where top-$k$ truncation is standard, the optimization domain reduces naturally from $n$ to just $k$ candidates.\\
\noindent\textbf{Our paper is organized as follows}: the next Section \ref{sec-rel} summarizes related work. Then, follow three key parts of our paper, organized so that readers familiar with only one of the two frameworks used (lossy speculative decoding and boosting) may easily process the part on which they are most familiar and then connect with the other one: Section \ref{sec-men-dec} presents the main results on the mentored decoding side, Section \ref{sec-boost} presents the boosting side and its connection to mentored decoding, finally Section \ref{sec-algo-prop-md} presents the algorithmic sides of the theory discussed. A following Section \ref{sec-disc} discusses additional topics related to mentored decoding and boosting, and a last Section \ref{sec-conc} concludes with avenues for future research. Our paper is self-contained: all proofs are given either in the main body of the paper or in an Appendix starting page \pageref{sec-app-proofs}.

\section{Related Work}\label{sec-rel}

On the pure \textbf{speculative decoding} side, i.e. \textit{lossless decoding}, \textit{Blockwise Parallel Decoding} \citep{stern2018blockwise} pioneered interleaving fast draft sequence generation with parallel target verification to accelerate greedy sequence-to-sequence decoding. \citet{xia2022speculative} refined this approach and coined the term \textit{speculative decoding}, drawing analogy to speculative execution in computer architecture. \citet{leviathan2023fast,chen2023accelerating} independently generalized the framework to multinomial sampling, establishing the standard rejection-sampling formulation described in Section~\ref{sec-men-dec}. \citet{sun2023spectr} proved that this formulation achieves the optimal acceptance rate under exact target distribution matching. Following its inception, speculative decoding has evolved along several dimensions. A first one moved towards better aligned or faster draft models: Self-speculative decoding \citep{kim2023speculative,zhang2024draft,liu2024kangaroo,gloeckle2024better,cai2024medusa} eliminates the need for a separate draft model by adding lightweight prediction heads, skipping transformer layers, pruning sub-networks, or early-exiting from the target model. \textit{EAGLE} and its variants \citep{li2024eagle,li2024eagle2,li2026eagle} perform autoregressive drafting over target hidden feature representations rather than discrete tokens. DistillSpec \citep{zhou2024distillspec} aligns draft models to target models during training by minimizing $f$-divergence objectives. Recently, \textit{DFlash} \citep{chen2026dflash} proposed non-autoregressive block diffusion models for low-latency draft prediction. A second one moved towards model-free drafting: \citet{fu2024break} introduced \textit{Lookahead Decoding}, generating candidate n-grams via parallel Jacobi fixed-point iteration without relying on a draft model. \textit{LLMA} \citep{yang2023inference} copies recurring n-gram patterns directly from the input prompt or reference documents, while \textit{REST} \citep{he2024rest} retrieves candidate phrases from external datastores. A third one moved towards multi-draft and tree verification: Rather than proposing a single linear sequence of candidate tokens, tree-based speculation generates candidate trees verified in parallel using tree-attention masks \citep{miao2024specinfer,chen2024sequoia,li2024eagle2}. \citet{sun2023spectr} proposed \textit{SpecTr}, using optimal transport to verify multiple drafts. \citet{hu2025towards} established that optimal multi-draft speculative decoding (MDSD) reduces via total unimodularity to subset selection, introducing \textit{Greedy Draft Selection} as an efficient and theoretically grounded candidate selection strategy. It has been observed that the performances of speculative decoding depend on many factors \citep{liu2026speculative}. In deep contrast with the work, essentially experimental, that flourished after the seminal work of \citet{leviathan2023fast,chen2023accelerating}, the theory side of speculative decoding has remained in close contact with the seminal work, with essentially one exception digging in the expected number of tokens successfully predicted \citep{pankratov-alistarh-2026-speculative}.\\

\noindent Because lossless speculative decoding strictly preserves the target distribution, its acceptance rate is fundamentally bounded by the divergence between drafter and target. To further increase throughput, several works have explored relaxing this exact-matching constraint towards \textit{lossy speculative decoding}. \citet{xia2026practical} provide an empirical benchmark of most of these lossy decoding strategies. Many approaches are fundamentally heuristic in nature. In their seminal paper, \citet{leviathan2023fast} introduced \textit{Lenient Speculative Decoding}, making the per-token acceptance probability dependent on a factor that skews it. Subsequent work proposed various heuristic acceptance criteria: \textit{Typical Acceptance Sampling} \citep{cai2024medusa} accepts candidate tokens based on entropy heuristics; \textit{Fuzzy Speculative Decoding} \citep{holsman2025fuzzy} unconditionally accepts draft tokens whenever the step-level divergence between drafter and target falls below a scalar threshold $T$; \textit{Speculative Contrastive Decoding} \citep{yuan2024speculative} incorporates contrastive penalties to steer generation away from draft errors; and \citet{narasimhan2025faster} adapt speculative verification to model cascading deferral rules. Other work include using big models or more than two models \citep{byun2025model,li2026targetimitationcollaborationspeculative}, adding a linear head on top of the target called a \textit{judge} -- being another example of last layer retraining -- \citep{bachmann2025judge}, completing the process with information from prefill \citep{wang-etal-2025-alignment}, completing the process with guessing appropriate draft length \citep{zhang-etal-2025-draft}, etc. \citep{holsman2025fuzzy}. \\

\noindent The first work \textbf{formalizing} the problem of mentored decoding as a constrained optimization problem maximizing draft acceptance subject to bounded (reverse Kullback-Leibler) divergence is  \citet{Tran-Thien_2023}. \citet{yin2024theoretical} analyzed the problem under Total Variation distance, characterizing the linear Pareto frontier. Inspired by this result, \textit{DIVERSED} \citep{wang2025diversed} introduced dynamic ensemble verification by sampling from a convex combination between drafter and target, like \citet{zhong2025speeding}. Under forward Kullback-Leibler divergence, \textit{Cactus} \citep{hao2026cactus} optimizes candidate acceptance via a second-order Taylor approximation on the sampled token's coordinate, though without globally controlling the divergence of the resulting joint output distribution.\\

\noindent Finally, the papers observing that allowing some drift can beat the target's own metrics are experimental \citep{DBLP:conf/aaai/QinHPCS25, li2026targetimitationcollaborationspeculative, zhong2025speeding}, as to our knowledge there is no formal work on the subject.

\section{Mentored decoding}\label{sec-men-dec}

We first provide some definitions needed for this Section. $n$ is the vocabulary size, $\Delta_n$ is the $n$-probability simplex, $[n] \defeq \{1, 2, ..., n\}$. Bold faces like $\ve{z}$ denote vectors, and their coordinates are denoted like $z_i$. Binary relations between vectors of the same dimension are coordinate-wise: $\ve{a} \leq \ve{b}$ means $a_i \leq b_i, \forall i \in [n]$. For any $\ve{a}, \ve{b} \in \mathbb{R}^n$ such that $\ve{a} \leq \ve{b}$, we let $[\ve{a}, \ve{b}] \defeq \prod_{i\in [n]} [a_i, b_i]$. A prompt to the drafter and target models yields two distributions $\ve{p}\in \Delta_n$ (drafter) and $\ve{q}\in \Delta_n$ (target). Speculative and mentored decoding operate by generating multiple draft outputs and checking acceptance in parallel with the target. Checking a token is probabilistic and relies on a vector of acceptance probabilities $\ve{r} \in [0,1]^n$; if rejected, a resampling distribution $\ve{s}\in \Delta_n$ resamples a new token. Then, the same algorithm resumes until complete sequence generation. We refer e.g. to \citet{leviathan2023fast,Tran-Thien_2023} for more details on the algorithmic side. We also define the $f$-divergences between $\ve{\pi}\in \Delta_n$ and $\ve{q}\in \Delta_n$  as \citep{fdiv-AliSilvey-1966,cEI}:
\begin{eqnarray}
  D_f(\ve{\pi}\| \ve{q}) & \defeq & \sum_i q_i f\left(\frac{\pi_i}{q_i}\right),\label{eq-def-f-div}
\end{eqnarray}
where the \textit{generator}
\begin{eqnarray}
  f: \mathbb{R}^+ \rightarrow \mathbb{R} \mbox{ is convex and such that } f(1) = 0.\label{eq-const-f}
\end{eqnarray}

\subsection{One problem, two parameterizations}\label{subsec-one-two}
Without further ado, we define the core inference problem on which we focus.
\begin{definition}\label{def-f-md-2}
  For any $f$ as per \eqref{eq-const-f}, $\ve{p}, \ve{q} \in \Delta_n, D \geq 0$, the $f$-mentored decoding (MD) problem is defined as find
 \begin{align}
      \textsc{md}^2_f(\ve{p}, \ve{q}; D) \defeq \arg\min _{\ve{r} \in [0,1]^n, \ve{s} \in  \Delta_n} -\ve{p}^\top \ve{r}  \quad \mbox{s.t. } D_{f} (\ve{p} \odot \ve{r} + (1-\ve{p}^\top \ve{r})\cdot\ve{s} \|\ve{q}) \leq D, \label{def-pb-f-md-2-general}\tag{\textbf{$f$-MD-2}}
 \end{align}
 where $\odot$ is Hadamard product.
\end{definition}
This problem was introduced by \citet{Tran-Thien_2023} with the specific choice $f_{\mathrm{rKL}}(z) \defeq -\ln(z)$, the reverse-KL divergence. Note that if $D = 0$, the problem formalizes speculative decoding (SD). \eqref{def-pb-f-md-2-general} is a direct parameterization of MD: we directly seek the couple of vector of acceptance probabilities $\ve{r}$ and resampling distribution $\ve{s}$. A convenient result that we now state and prove is that this problem admits an equivalent parameterization with a single parameter, the mentored distribution itself. Let us define it:
\begin{align}
      \textsc{md}^1_f(\ve{p}, \ve{q}; D) \defeq \arg\min _{\ve{\pi} \in \Delta_n} D_{\mathrm{TV}} (\ve{\pi} \|\ve{p})  \quad \mbox{s.t. } D_{f} (\ve{\pi} \|\ve{q}) \leq D,  \label{def-pb-f-md-1-general}\tag{\textbf{$f$-MD-1}}
\end{align}
where $D_{\mathrm{TV}}$ denotes the total variation divergence, whose generator is $f_{\mathrm{TV}} (z) \defeq |z-1|/2$.
\begin{lemma}\label{lem-equiv-sol}
  Let $\oslash$ denote the coordinate-wise division. For any $\ve{\pi} \in \textsc{md}^1_f(\ve{p}, \ve{q}; D)$, we have $(\ve{r}, \ve{s}) \in \textsc{md}^2_f(\ve{p}, \ve{q}; D)$ where $\ve{r}, \ve{s}$ are defined as:
\begin{eqnarray}
  \ve{r} \defeq \min \{\ve{1}, \ve{\pi} \oslash \ve{p}\} ; \quad \ve{s} \defeq \left\{
  \begin{array}{ccl}
    \frac{1}{1-\ve{p}^\top\ve{r}} \cdot \left(\ve{\pi} - \ve{p}\odot\ve{r}\right) = \frac{1}{1-\ve{p}^\top\ve{r}} \cdot \max\{\ve{0}, \ve{\pi} - \ve{p}\}& \mbox{ if } & \ve{p}^\top\ve{r} < 1\\
    \mbox{any } \ve{s} \in \Delta_n & \mbox{ if } & \ve{p}^\top\ve{r} = 1
    \end{array}\right.\label{eq-r-s-from-pi} .
\end{eqnarray}
Respectively, for any $(\ve{r}, \ve{s}) \in \textsc{md}^2_f(\ve{p}, \ve{q}; D)$, we have $\ve{\pi} \in \textsc{md}^1_f(\ve{p}, \ve{q}; D)$ with
\begin{eqnarray}
  \ve{\pi} & \defeq & \ve{p} \odot \ve{r} + (1-\ve{p}^\top \ve{r})\cdot\ve{s}. \label{eq-pi-from-r-s}
\end{eqnarray}
Finally, the corresponding objective functions are related by $\ve{p}^\top \ve{r} = 1 - D_{\mathrm{TV}} (\ve{\pi} \|\ve{p})$.
\end{lemma}
Proof in Appendix, Section \ref{sec-proof-lem-equiv-sol}. Lemma \ref{lem-equiv-sol} allows us to work with whichever parameterization fits best to context; we denote $f$-MD as the general problem of $f$ mentored decoding, with whichever parameterization. The overall probability to accept a token in the MD setting is defined as
\begin{eqnarray*}
  \pacc(MD) & \defeq & \ve{p}^\top \ve{r}.
\end{eqnarray*}
The expression is the same for SD, only in this case we would have the "hidden" constraint bound $D$ to be zero. We make the following assumptions regarding mentored decoding.
  \begin{assumption}\label{assum-mda}
We assume $0< D < D_{\mathrm{TV}} (\ve{p} \|\ve{q})$, $\ve{p} \neq \ve{q}$, $\ve{p} > \ve{0}$ and $\ve{q} > \ve{0}$.
\end{assumption}
Note the weakness of those statements: if it does not hold that $0< D < D_{\mathrm{TV}} (\ve{p} \|\ve{q})$, the problem is trivial ($D$ being non-negative, either $\ve{\pi} = \ve{p}$ or $\ve{\pi} = \ve{q}$ is optimal); similarly if $p_i = q_i$ for some $i\in [n]$ then any optimum trivially meets $\pi_i = p_i = q_i$ so the coordinate can be dropped from solving \eqref{def-pb-f-md-1-general}, notwithstanding a replacement of the unit mass constraint of the probability simplex by a $1-q_i$ mass constraint. Finally, the assumptions $\ve{p} > \ve{0}$ and $\ve{q} > \ve{0}$ are reasonable for LLMs, or any neural net architecture in which the last layer is a softmax.  Finally, note that we should theoretically add the technical assumption that $f$ be proper in \eqref{eq-const-f}, but it is in fact always met for any $f$ relevant to our context, and even more given Assumption \ref{assum-mda} with which we can always restrict all of our analysis on a closed interval of the real line.

\begin{figure*}
  \centering
  \resizebox{\columnwidth}{!}{\begin{tabular}{c?c?c}\Xhline{2pt}
                                \includegraphics[trim=0bp 0bp 0bp 0bp,clip,width=0.3\columnwidth]{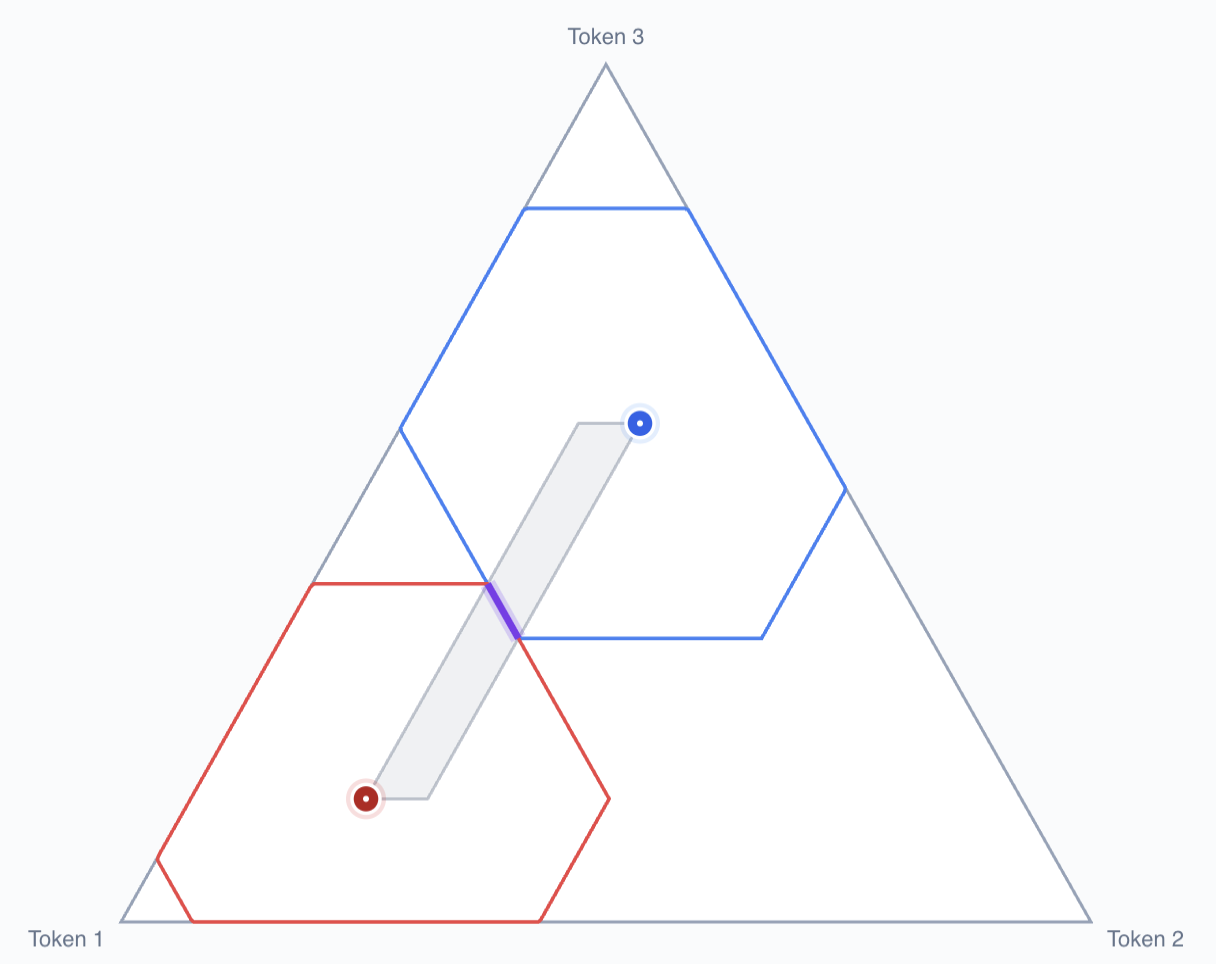} & \includegraphics[trim=0bp 0bp 0bp 0bp,clip,width=0.3\columnwidth]{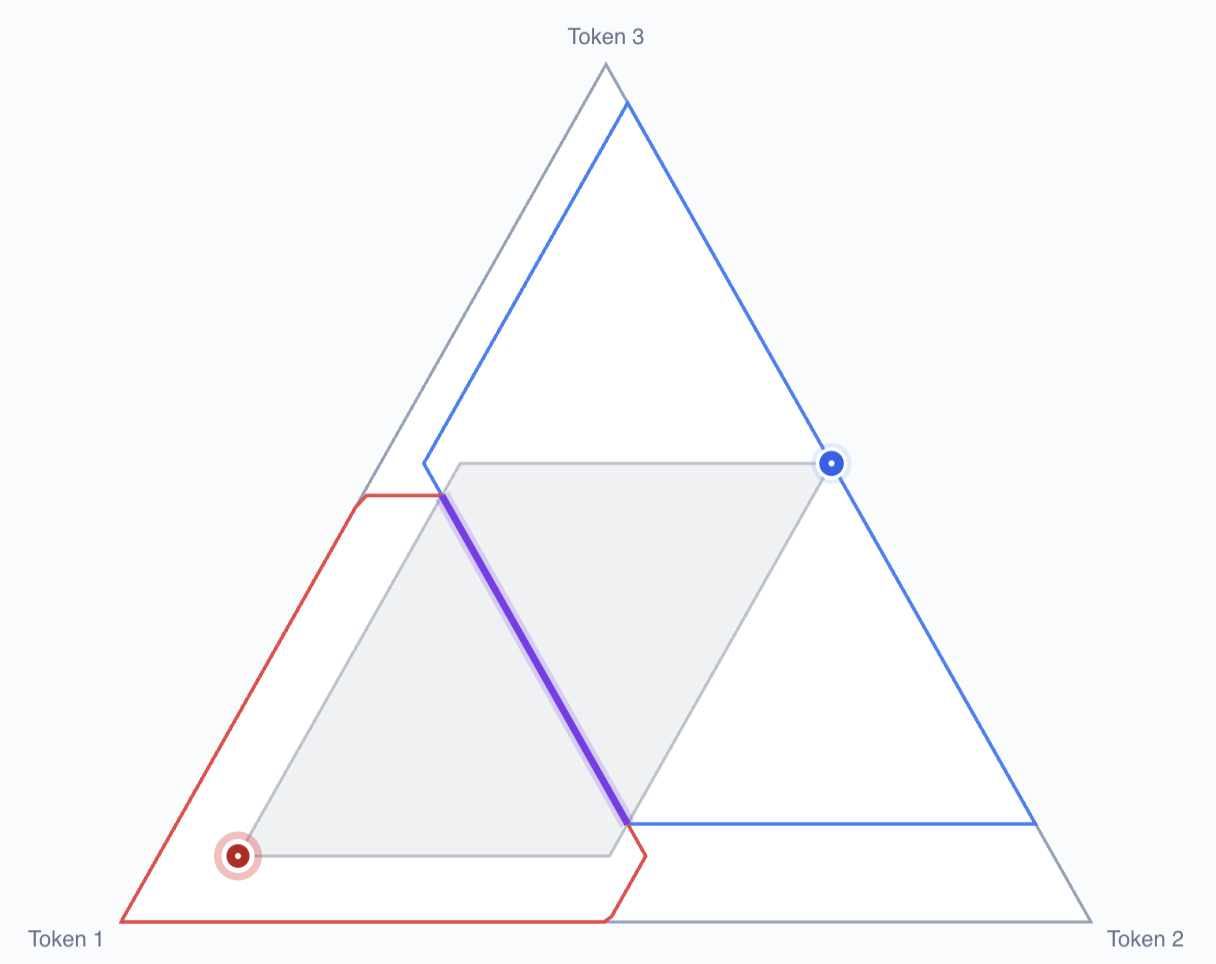} & \includegraphics[trim=0bp 0bp 0bp 0bp,clip,width=0.3\columnwidth]{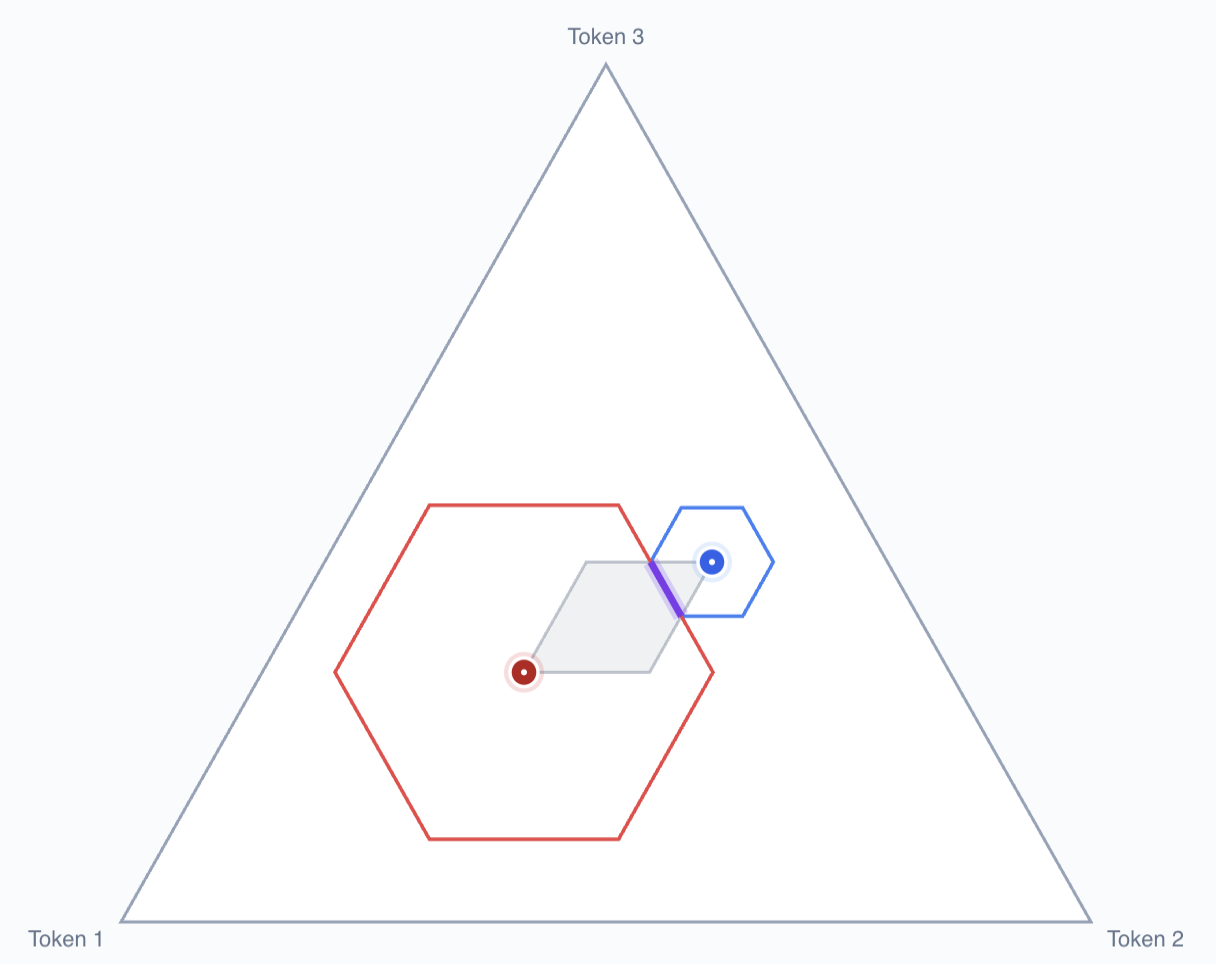} \\ \Xhline{2pt}
    \end{tabular}}
  \caption{Example of optimal solutions for the TV-MD problem with $n=3$ tokens in $\Delta_3$, as intersection (thick segment) between the two $L_1$ balls centered in $\ve{p}$ and $\ve{q}$ (in color {\color{red} red} and {\color{blue} blue}) whose radii define $D$ and the optimal objective value for the total variation divergences in \eqref{def-pb-f-md-1-general}. The shaded area is set $\left[\min\{\ve{p},  \ve{q}\}, \max\{\ve{p},  \ve{q}\}\right] \cap \Delta_3$ (see text).}
    \label{fig:tv-balls}
  \end{figure*} 

\subsection{Warmup: the special case of the total variation divergence}\label{subsec-tv}

The case  $f = f_{\mathrm{TV}}$ in $f$-MD is especially interesting: its set of optimal solutions has a beautiful geometric characterization and its proof is a few liner. We denote it as TV-MD.
\begin{theorem}\label{thm-TV-MD}
Under Assumption \ref{assum-mda}, we have
  \begin{eqnarray}
    \textsc{md}^1_{ f_{\mathrm{TV}}}(\ve{p}, \ve{q}; D) & = & \Delta_n \cap \{D_{\mathrm{TV}} (. \|\ve{q}) = D\} \cap \left[\min\{\ve{p},  \ve{q}\}, \max\{\ve{p},  \ve{q}\}\right] .\label{def-opt-tv-tv}
\end{eqnarray}
  \end{theorem}
\begin{proof}
  The TV divergence satisfies the triangle inequality, hence
  \begin{eqnarray}
D_{\mathrm{TV}} (\ve{q} \|\ve{p}) & \leq & D_{\mathrm{TV}} (\ve{\pi} \|\ve{p}) + D_{\mathrm{TV}} (\ve{\pi} \|\ve{q}) \label{def-ineq}.
  \end{eqnarray}
  The LHS is fixed and the objective to minimize is $D_{\mathrm{TV}} (\ve{\pi} \|\ve{p})$. Any $\ve{\pi} \in \Delta_n$ can be formulated as $\pi_i = \alpha_i p_i + (1-\alpha_i) q_i, \alpha_i \in \mathbb{R}, \forall i \in [n]$ since $\ve{p}\neq \ve{q}$, which yields after simplification for the RHS of \eqref{def-ineq}:
  \begin{eqnarray}
D_{\mathrm{TV}} (\ve{\pi} \|\ve{p}) + D_{\mathrm{TV}} (\ve{\pi} \|\ve{q}) & = & \sum_{i \in [n]} (|1-\alpha_i| + |\alpha_i|) \cdot |p_i - q_i|. \label{sum-div}
  \end{eqnarray}
  Choose $\alpha_i = \alpha \in [0,1], \forall i\in [n]$: the RHS equals $D_{\mathrm{TV}} (\ve{q} \|\ve{p})$ and so \eqref{def-ineq} becomes an equality: this $\ve{\pi}$ is optimal for TV-MD if it maximizes $D_{\mathrm{TV}} (\ve{\pi} \|\ve{q})$ under the constraint, i.e. it makes it active as $D_{\mathrm{TV}} (\ve{\pi} \|\ve{q}) = D$, which happens for the choice $\alpha = D / D_{\mathrm{TV}} (\ve{q} \|\ve{p})$. To be optimal thus generally requires to keep (i) the equality in \eqref{def-ineq} and (ii) the constraint active. Since $f(z) \defeq |1-z| + |z|$ satisfies $f([0,1]) = \{1\}$ and is $>1$ otherwise, optimality requires $\alpha_i \in [0,1], \forall i \in [n]$ in \eqref{sum-div} and thus $\pi_i \in [\min\{p_i, q_i\}, \max\{p_i, q_i\}], \forall i \in [n]$. The set of optimal solutions is thus $\{\ve{\pi} \in \left[\min\{\ve{p},  \ve{q}\}, \max\{\ve{p},  \ve{q}\}\right] \cap \Delta_n: D_{\mathrm{TV}} (\ve{\pi} \|\ve{q}) = D\}$, as claimed.
\end{proof}
It is worth mentioning that \citet{yin2024theoretical} analyzed the TV relaxation of SD. While they provided a thorough analysis of the optimal \textit{losses}, they did not provide the analytic form of the optimum, which has neat properties. Indeed, the TV divergence constraints define $L_1$ balls. The optimal solution of \eqref{def-opt-tv-tv} is the intersection between the simplex and two tangent balls, one whose radius depends on the optimal objective and one whose radius is parameter $D$. Figure \ref{fig:tv-balls} exemplifies three such cases (See also Figure \ref{fig:breakpoints-and-f-balls} for other $f$-MD cases). The optimal set being this "big" and "nice" naturally opens the question as to whether such optimal solutions that speed up inference might in fact be grounded in a model producing $\ve{\pi}$ that, since it is a function of the target and drafter models, could \textit{compete with or beat} the target model in terms of \textit{quality}. We shall indeed give a formal positive answer in Section \ref{sec-boost}, but before, we address and solve $f$-MD for a general $f$ \eqref{eq-const-f}.

\subsection{$f$-mentored decoding: general solution}\label{subsec-f-gen}

\begin{figure*}
  \centering
  \resizebox{0.7\columnwidth}{!}{\includegraphics[trim=30bp 20bp 30bp 20bp,clip,width=0.3\columnwidth]{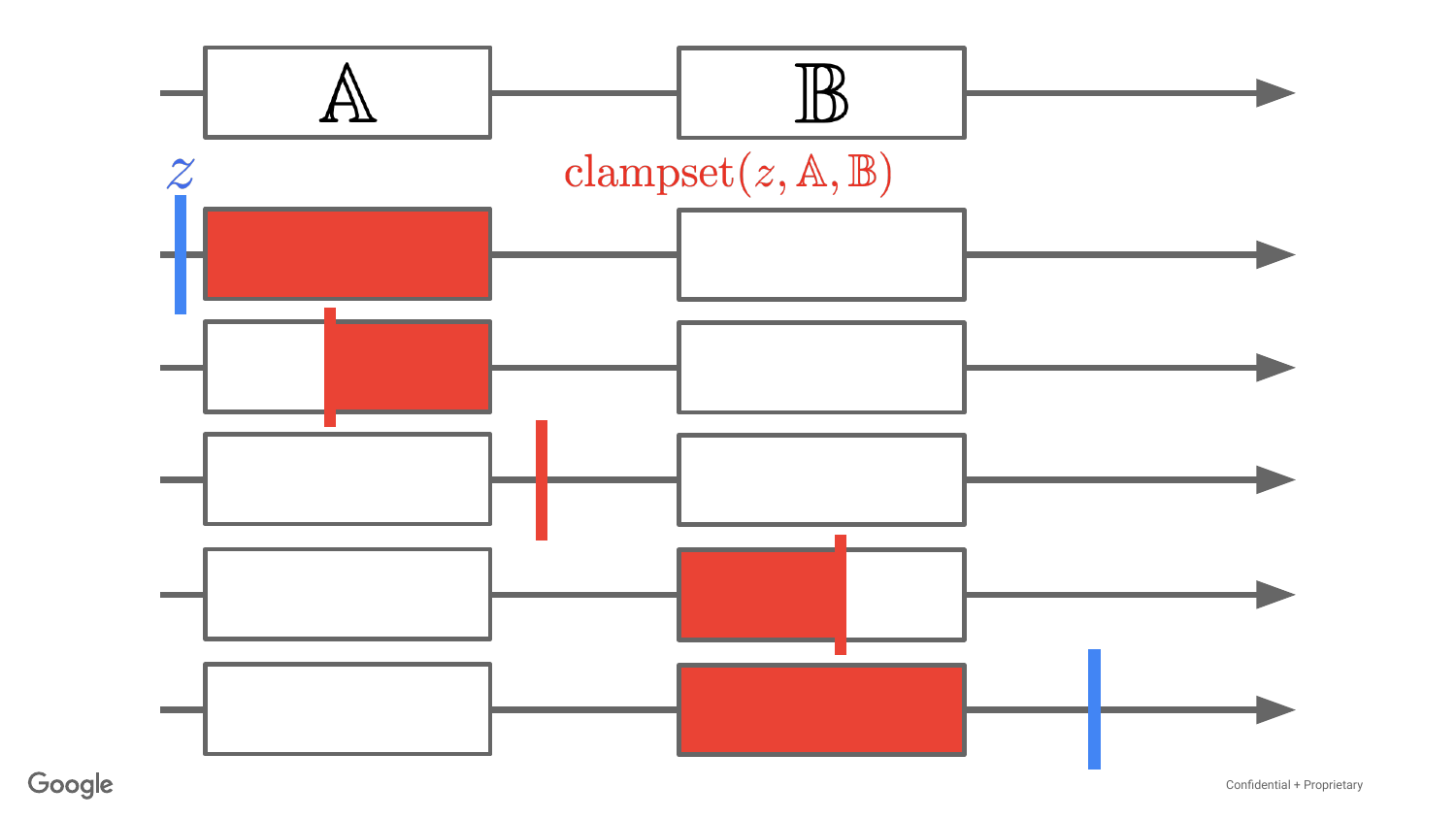}}
  \caption{Illustration of the definition of $\clamps(z, \mathbb{A}, \mathbb{B})$ \eqref{eq-def-clampset} (in {\color{red} red}) when $\mathbb{A} < \mathbb{B}$ are intervals (figured by rectangles) and $z$ (thick vertical bar) moves along the real line (see text).}
    \label{fig:def-clampset}
  \end{figure*}
  
To step up to the general case, we need additional definitions. Binary relations defined over sets of reals are true iff they hold for any applicable elements: for example, $\mathbb{A} \leq \mathbb{B}$ it true iff $a \leq b, \forall a \in \mathbb{A}, b\in \mathbb{B}$. For any $z \in \mathbb{R}, \mathbb{A} \subseteq \mathbb{R}$, we let:
\begin{eqnarray*}
  \mins(z, \mathbb{A}) & \defeq & \{s \in \mathbb{A} : s \leq z\}  ,\\
  \maxs(z, \mathbb{A}) & \defeq & \{s \in \mathbb{A} : s \geq z\}  ,
\end{eqnarray*}
and for any $\mathbb{B} \subset \mathbb{R}$ such that $\mathbb{A} < \mathbb{B}$,
\begin{eqnarray}
  \clamps(z, \mathbb{A}, \mathbb{B}) \defeq  \mins(0, \mathbb{B} - z) + \maxs(0, \mathbb{A} - z) + \iver{\inf \mathbb{A} \leq z \leq \sup \mathbb{B}} \cdot \{z\}, \label{eq-def-clampset}
  \end{eqnarray}
where $\iver{.}$ denotes Iverson's bracket \citep{kTN}. Some properties are notable.
\begin{lemma}\label{lemCMM}
For any $z \in \mathbb{R}, \mathbb{A} < \mathbb{B} \subset \mathbb{R}$,
  \begin{eqnarray}
    \{\min\{z, z'\}: z' \in \clamps(z, \mathbb{A}, \mathbb{B})\} & = & \{z\} + \mins(0, \mathbb{B} - z), \label{pMMC3}\\
    \{\max\{0, z'-z\}: z' \in \clamps(z, \mathbb{A}, \mathbb{B})\} & = & \maxs(0, \mathbb{A} - z). \label{pMMC4}
  \end{eqnarray}
  \end{lemma}
  \begin{proof}
In \eqref{pMMC3} the LHS is $z$ unless $\{s \in \mathbb{B} : s \leq z\} \neq \emptyset$, in which case it is this set, i.e. $\mins(z, \mathbb{B})$. In \eqref{pMMC4} the LHS is $0$ unless $\{s \in \mathbb{A} : s-z \geq 0\}\neq \emptyset$, in which case it is this set, i.e. $\maxs(0, \mathbb{A} - z)$.
\end{proof}

\paragraph{Analysis for $f$ convex differentiable} we start with the case where $f$ is differentiable in \eqref{eq-const-f}, and later relax differentiability. 
\begin{lemma}\label{lem-slater}
Assumption \ref{assum-mda} implies Slater's constraint qualification satisfied for \eqref{def-pb-f-md-1-general} and \eqref{def-pb-f-md-2-general}.
\end{lemma}
Proof in Appendix, Section \ref{sec-proof-lem-slater}. So KKT conditions are necessary and sufficient for optimality in the study of $f$-MD. We now make a connection between $\textsc{md}^{2}_f(\ve{p}, \ve{q}; D)$ and another problem, whose objective is the following one:
\begin{eqnarray}
\textsc{is}_f(\ve{p}, \ve{q}; D) \defeq \left\{\ve{\pi} \in \Delta_n: \left\{
\begin{array}{l}
\exists \alpha < \beta \in \mathrm{Im}(-f'): \ve{\pi} \in\clamps(\ve{p}, L_\beta(-f') \cdot \ve{q}, L_\alpha(-f') \cdot \ve{q})\\
D_{f} (\ve{\pi} \|\ve{q}) = D
\end{array}
\right.\right\}, \label{isAB}
  \end{eqnarray}
  where $L_y(g) \defeq \{z : g(z) = y\}$ denotes the $y$-level set of function $g$. As we shall explain later when relaxing the differentiability assumption on $f$, the set described in \eqref{isAB} is the same as in \eqref{def-opt-tv-tv} when properly relaxing the derivative to the subdifferential. In such a context, looking at Figure \ref{fig:tv-balls} for an example, let us keep in mind for now that \eqref{isAB} elicits isodivergence sets on the probability simplex which, just like \eqref{def-opt-tv-tv} does for the case of the TV divergence, denote the set of optimal solutions that we seek for $f$-MD.

  Without further ado, we state and prove the main Theorem that elicits the connections between \eqref{def-pb-f-md-2-general} and \eqref{isAB}.
  \begin{theorem}\label{th-opt1}
Under Assumption \ref{assum-mda}, there exists a bijection between $\textsc{md}^2_f(\ve{p}, \ve{q}; D)$ and $\textsc{is}_f(\ve{p}, \ve{q}; D)$. Specifically,
\begin{itemize}
\item [($\textsc{md}^2_f\rightarrow\textsc{is}_f$)] for any couple $(\ve{r}, \ve{s}) \in \textsc{md}^2_f(\ve{p}, \ve{q}; D)$, we have $\ve{\pi} \defeq \ve{p} \odot \ve{r} + (1-\ve{p}^\top \ve{r}) \ve{s} \in \textsc{is}_f(\ve{p}, \ve{q}; D)$ for the choices
\begin{eqnarray}
\alpha & \defeq & \expect_{i \sim \ve{u}} \left[(-f') \left(\frac{\pi_i}{q_i}\right)\right], \quad\mbox{ with }\ve{u} \defeq \frac{1}{1-\ve{p}^\top \ve{r}} \cdot (\ve{p}-\ve{p} \odot \ve{r}) \in \Delta_n.\label{mdisAlpha}\\
\beta & \defeq & \expect_{i \sim \ve{s}} \left[(-f') \left(\frac{\pi_i}{q_i}\right)\right]\label{mdisBeta}
\end{eqnarray}
\item [($\textsc{is}_f\rightarrow\textsc{md}^2_f$)] for any $\ve{\pi} \in\textsc{is}_f(\ve{p}, \ve{q}; D)$, the couple $(\ve{r}, \ve{s})$ defined in \eqref{eq-r-s-from-pi} is in $\textsc{md}^2_f(\ve{p}, \ve{q}; D)$.
\end{itemize}
\end{theorem}
Proof in Appendix, Section \ref{sec-proof-th-opt1}. Using Lemma \ref{lem-equiv-sol}, we get as immediate corollary another characterization of $\textsc{md}^1_f(\ve{p}, \ve{q}; D)$:
\begin{eqnarray*}
\textsc{md}^1_f(\ve{p}, \ve{q}; D) & = & \textsc{is}_f(\ve{p}, \ve{q}; D),
\end{eqnarray*}
which is not unreminiscent of the case of the total variation in \eqref{def-opt-tv-tv} (more on this later). Finally, we have the following Lemma stating some important properties of $L_\alpha$ and $L_\beta$ in ($\textsc{md}^2_f\rightarrow\textsc{is}_f$), whose proof is given in the proof of Theorem \ref{th-opt1}.
\begin{lemma}\label{lemLAlphaBeta}
$\alpha, \beta$ in \eqref{mdisAlpha}, \eqref{mdisBeta} satisfy:
\begin{eqnarray}
L_\beta(-f') < L_\alpha(-f') & ; &  \min L_\beta(-f') \leq 1 \leq \max L_\alpha(-f'),\label{lAB-1}\\
\min_i p_i/q_i < \max L_\beta(-f')& ; &  \min L_\alpha(-f') < \max_i p_i/q_i.\label{lAB-2}
\end{eqnarray}
\end{lemma}
Note that \eqref{lAB-2} follows from the fact that $\alpha, \beta$ are expectations in \eqref{mdisAlpha}, \eqref{mdisBeta}. An additional important result is the following one, which states that $\ve{r}$ has full support.
\begin{lemma}\label{lemPOST}
Under Assumption \ref{assum-mda}, any couple $(\ve{r}, \ve{s}) \in \textsc{md}^2_f(\ve{p}, \ve{q}; D)$ satisfies $\ve{r} > \ve{0}$.
  \end{lemma}
  Proof in Appendix, Section \ref{sec-proof-lemPOST}.

  \paragraph{Mentored decoding: analysis for general $f$ \eqref{eq-const-f}} A simple trick allows to alleviate the differentiability condition on $f$ and prove the result for any convex $f$ \eqref{eq-const-f}, and it proceeds from the simple example of how Theorem \ref{th-opt1} also covers the case of TV, whose generator is $f(z) \defeq |z-1|$, non differentiable only in $z=1$. We first smooth the generator in an open $\delta$-neighborhood of 1, eventually with a $y$-translation of the graph to keep $f(1) =0$. We want to prevent $L_\alpha(-f'), L_\beta(-f')$ to be picked from this neighborhood, so we are going to tune $\delta > 0$. If $L_\alpha(-f')$ is in, as $\delta \searrow 0$, the objective converges to that of speculative decoding, and if $L_\beta(-f')$ is in, as $\delta \searrow 0$, the $f$-divergence value goes to 0. So we can pick $\delta>0$ small enough for $L_\alpha(-f')$ to be out of the neighborhood (objective small enough) with $L_\beta(-f')$ out of the neighborhood (acceptable divergence).

We thus end up with only one possible solution, $L_\alpha(-f') = L_{-1}(-f') = [1+\delta, +\infty)$ and $L_\beta(-f') = L_{1}(-f') = (-\infty, 1-\delta]$ and thus according to \eqref{isAB} all solutions of (\textsc{is}) satisfy
 \begin{eqnarray}
  \ve{\pi} & \in & \clamps(\ve{p}, (-\infty, 1-\delta] \cdot \ve{q}, \cdot [1+\delta, +\infty) \cdot \ve{q}). \label{setClamp1}
  \end{eqnarray}
Note that $\forall i \in [n]$, we have
\begin{eqnarray*}
  \lefteqn{\clamps(p_i, (-\infty, 1-\delta] \cdot q_i, \cdot [1+\delta, +\infty) \cdot q_i)}\\
  &= & \left\{ 
\begin{array}{ccl}
\left[p_i, (1-\delta) q_i\right] & \mbox{ if } & p_i < (1-\delta) q_i\\
\left[(1+\delta) q_i, p_i\right] & \mbox{ if } & p_i > (1+\delta) q_i\\
p_i & \mbox{ if } & p_i \in [(1-\delta) q_i, (1+\delta) q_i]
\end{array}
\right. .
\end{eqnarray*}
Because $p_i \neq q_i, \forall i$ (Assumption \ref{assum-mda}), we can further choose $\delta > 0$ small enough so that we always have 
\begin{eqnarray}
p_i & \not\in & [(1-\delta) q_i, (1+\delta) q_i], \forall i\label{propPI}.
\end{eqnarray}
The set \eqref{setClamp1} simplifies as:
\begin{eqnarray*}
\lefteqn{\clamps(\ve{p}, (-\infty, 1-\delta] \cdot \ve{q}, \cdot [1+\delta, +\infty) \cdot \ve{q})}\\
& = & \left[\min\{\ve{p}, (1+\delta) \cdot \ve{q}\}, \max\{\ve{p}, (1-\delta) \cdot \ve{q}\}\right],
  \end{eqnarray*}
and none of the intervals is empty thanks to \eqref{propPI}. We thus get
\begin{eqnarray}
\textsc{is}_{f}(\ve{p}, \ve{q}; D) & = & \left\{\ve{\pi} \in \left[\min\{\ve{p}, (1+\delta) \cdot \ve{q}\}, \max\{\ve{p}, (1-\delta) \cdot \ve{q}\}\right] \cap \Delta_n \wedge D_{f} (\ve{\pi} \|\ve{q}) = D \right\}. \label{eq-opt-is-tv}
\end{eqnarray}
This set converges to $\textsc{is}_{TV}(\ve{p}, \ve{q}; D)$ as $\delta \rightarrow 0$ and $f$ converges to the generator of TV in any $\ell_p$ norm in the interval $[\min_i p_i/q_i, \max_i p_i/q_i]$, and we check that the solution found in \eqref{eq-opt-is-tv} matches \eqref{def-opt-tv-tv}.

Now, any convex function defined on an open convex set is differentiable anywhere except maybe on a set of measure zero \citep[Theorem 25.5]{rCA}, so for any point of non differentiability of a general $f$, our analysis above also holds for a sufficiently small $\delta > 0$, for which, after passing to the limit with $\delta \rightarrow 0$, we get the proof that $\textsc{is}_f(\ve{p}, \ve{q}; D)$ in \eqref{isAB} generalizes to
  \begin{eqnarray}
    \lefteqn{\textsc{is}_f(\ve{p}, \ve{q}; D)}\nonumber\\
  \hspace{-0.5cm} & \hspace{-0.7cm} \defeq & \hspace{-0.7cm} \left\{\ve{\pi} \in \Delta_n: \left\{
                                             \begin{array}{l}
                                               \hspace{-0.3cm}
  \begin{array}{l}
    \exists g, h \in \partial f, \\
    \exists \alpha \in \mathrm{Im}(-g), \exists \beta \in \mathrm{Im}(-h)
    \end{array}
\hspace{-0.3cm} : \left\{
  \begin{array}{l}
    \alpha < \beta \\
    \ve{\pi} \in\clamps(\ve{p}, L_\beta(-h) \cdot \ve{q}, L_\alpha(-g) \cdot \ve{q})
 \end{array}
\right.   \\
\hspace{-0.2cm} D_{f} (\ve{\pi} \|\ve{q}) = D
\end{array}
\right. \hspace{-0.5cm} \right\}, \label{isAB-nonconvex}
  \end{eqnarray}
  where $\partial f$ is the subdifferential of $f$; \eqref{mdisAlpha}, \eqref{mdisBeta} become, for $g, h  \in \partial f$,
  \begin{eqnarray}
\alpha \defeq \expect_{i \sim \ve{u}} \left[(-g) \left(\frac{\pi_i}{q_i}\right)\right] & , & \beta \defeq \expect_{i \sim \ve{s}} \left[(-h) \left(\frac{\pi_i}{q_i}\right)\right],\label{mdisAlphaBeta}
  \end{eqnarray}
  and $\ve{u}$ does not change. Since the level sets of the subdifferential of a strictly convex function are singletons, we immediately get the following Corollary as a consequence of \eqref{isAB-nonconvex} and the definition of $\clamps$.

\begin{corollary}
\label{cor-strictly-cvx}
    Under Assumption \ref{assum-mda}, suppose $f$ in \eqref{eq-const-f} is strictly convex. Then $\textsc{is}_f(\ve{p}, \ve{q}; D)$ is a singleton consisting of the unique $\ve{\pi} \in \Delta_n$ satisfying (i) $D_{f} (\ve{\pi} \|\ve{q}) = D$ and (ii) 
    \begin{eqnarray}
    \ve{\pi} & = & \max\{(1-b) \cdot \ve{q}, \min\{\ve{p}, (1+a) \cdot \ve{q}\}\} \label{eq-pi-ab}
    \end{eqnarray}
    for some unique $a > 0, b \in (0,1]$.
\end{corollary}  
As a consequence of Theorem \ref{th-opt1} and Lemma \ref{lem-equiv-sol}, $\ve{\pi}$ in \eqref{eq-pi-ab} also satisfies $\ve{\pi} \in \textsc{md}^1_f(\ve{p}, \ve{q}; D)$. Corollary \ref{cor-strictly-cvx} shows that the optimal solution of mentored decoding has a remarkable simple form in most cases. We now show that there is in fact much more to this simple form.

\paragraph{Simple approximations to $f$-MD} In our path to join the properties of mentored decoding and boosting, we need an intermediate result of independent interest. For any $\mathbb{A} \subseteq \mathbb{R}, z \in \mathbb{R}$, we let $z \cdot \mathbb{A} \defeq \{zz' : z' \in \mathbb{A}\}$. For any $a, b$, let
  \begin{eqnarray}
  \mathbb{A} & \defeq & \{i : p_i > (1+a) \cdot q_i\}, \label{defAbis}\\
  \mathbb{I} & \defeq & \{i : p_i \in q_i \cdot [1-b, 1+a]\}, \label{defIbis}\\
  \mathbb{B} & \defeq & \{i : p_i < (1-b) \cdot q_i\}. \label{defBbis}
  \end{eqnarray}
  The dependence of $\mathbb{A}, \mathbb{B}$ and $\mathbb{I}$ on $a,b$ is implicit for the sake of readability. We now define an important set of couples of reals
  \begin{definition}\label{def-cab}
For any $\ve{p}, \ve{q}$ output to the drafter and target, respectively, let $\mathcal{C}(\ve{p}, \ve{q})$ be the set of couples $(a,b)$ satisfying:
\begin{eqnarray}
     a & \in & \left[0,\max_i \frac{p_i}{q_i} - 1\right), \label{kktconst1}\\
    b & \in & \left[0, 1 - \min_i \frac{p_i}{q_i}\right), \label{kktconst2}\\
    -a q(\mathbb{A}) + b q(\mathbb{B}) & = & p(\mathbb{I}) - q(\mathbb{I}),\label{kktconst3}
\end{eqnarray}
where we have let $p(\mathbb{M}) \defeq 0 + \sum_{i\in \mathbb{M}} p_i$ for any $\mathbb{M} \subseteq [n]$ (and similarly, $q(\mathbb{M}) \defeq 0 + \sum_{i\in \mathbb{M}} q_i$). 
    \end{definition}
Set $\mathcal{C}(\ve{p}, \ve{q})$ has important properties, that we now state. 
\begin{theorem}\label{thmAPPROX1}
  $\forall (a,b) \in \mathcal{C}(\ve{p}, \ve{q})$, the choice 
\begin{eqnarray}
  \ve{r} & \defeq & \min\{\ve{1}, (1+a)\cdot \ve{q} \oslash \ve{p}\},\label{eq-defR}\\
  \ve{s} & \defeq &  (1-\ve{p}^\top \ve{r})^{-1} \cdot\max\{\ve{0}, (1-b)\cdot \ve{q} - \ve{p}\} \label{defS}
\end{eqnarray}
has the properties that the corresponding mentored distribution $\ve{\pi} = \ve{p} \mbox{ clamped to } [1-b, 1+a]\cdot \ve{q} \in \Delta_n$ \eqref{eq-pi-from-r-s} and:
\begin{enumerate}
\item [(I)] The corresponding acceptance probability of mentored decoding, $\pacc(MD)$, satisfies
    \begin{eqnarray}
\pacc(MD) & = & \pacc(SD) + a q(\mathbb{A}) + (p(\mathbb{I}_{>1})-q(\mathbb{I}_{>1})),\label{eq-pacc-sd-md}
    \end{eqnarray}
    where $\pacc(SD)$ is the acceptance probabilities of speculative decoding and 
    \begin{eqnarray*}
    \mathbb{I}_{>1} & \defeq & \left\{i : p_i\in q_i \cdot (1, 1+a]\right\} \quad \mbox{($\subseteq \mathbb{I}$, with the convention $p(\emptyset) = q(\emptyset) \defeq 0$)}.
    \end{eqnarray*}
  \item [(II)] for any $f$ as per \eqref{eq-const-f}, $\exists u\in [1-b, 1+a]$ such that the choice $(\ve{r}, \ve{s})$ in \eqref{eq-defR}, \eqref{defS} satisfies $(\ve{r}, \ve{s}) \in \textsc{md}^2_f(\ve{p}, \ve{q}; D)$ \eqref{def-pb-f-md-2-general} for
    \begin{eqnarray}
D & = & q(\mathbb{A}) \cdot f(1+a)  + q(\mathbb{B}) \cdot f(1-b) + q(\mathbb{I}) \cdot f(u).\label{Dapprox}
\end{eqnarray}
\end{enumerate}
\end{theorem}
Proof in Appendix, Section \ref{sec-proof-thmAPPROX1}. We check that Theorem \ref{thmAPPROX1} is optimal in the sense that $a, b \rightarrow 0$, we have the convergence $D \rightarrow f(1) = 0$ and $\pacc(MD)  \rightarrow \pacc(SD) $, since $(\ve{r}, \ve{s})$ converges towards the solution of speculative decoding. By definition, $p(\mathbb{I}_{>1}) > q(\mathbb{I}_{>1})$ if $\mathbb{I}_{>1} \neq \emptyset$ and obviously $a q(\mathbb{A}) \geq 0$, so both added terms in \eqref{eq-pacc-sd-md} contribute to having $\pacc(MD) > \pacc(SD)$.\\

Remember that $u\in [1-b, 1+a]$ and $f(1) = 0$ \eqref{eq-const-f} so the unknown term in \eqref{Dapprox} may be quite small depending on the choice of $f$. 
  We have already seen that TV is special in $f$-divergences for mentored decoding: its set of solutions is a simple geometric problem, which, for any other $f$, becomes substantially more involved. It turns out that the TV divergence holds another singular property.

\paragraph{One mentor to rule them all and the role of TV-MD} Before tackling boosting, we show two important invariants. First, under some lightweight conditions on $f$ -- satisfied in particular by all strictly convex generators --, all optimal solutions of $f$-MD are also in the set of optimal solutions for the total variation divergence. Second the set of optimal mentored distributions as $D$ ranges as per Assumption \ref{assum-mda} \textit{are the same} for any strictly convex $f$. 
\begin{theorem}\label{thmUniversal}
Under assumption \ref{assum-mda}, for any $f$ as per \eqref{eq-const-f} and any $(\ve{r}, \ve{s}) \in \textsc{md}^2_f(\ve{p}, \ve{q}; D)$ such that $\alpha, \beta$ in \eqref{mdisAlphaBeta} satisfy $L_\beta(-h) \leq 1 \leq L_\alpha(-g)$, there exists $D'>0$ such that
\begin{eqnarray*}
(\ve{r}, \ve{s}) & \in & \textsc{md}^2_{f_{\mathrm{TV}}}(\ve{p}, \ve{q}; D').
\end{eqnarray*}
Hence, any such optimal solution to $f$-MD is also optimal for the total variation divergence. Furthermore, for any strictly convex generators $f, g$ \eqref{eq-const-f},
\begin{eqnarray}
\{\textsc{md}^2_f(\ve{p}, \ve{q}; D) : D \mbox{ as per Assumption \ref{assum-mda}}\} = \{\textsc{md}^2_g(\ve{p}, \ve{q}; D) : D \mbox{ as per Assumption \ref{assum-mda}}\}. \label{eq-opt-f-g-md}
\end{eqnarray}
\end{theorem}
Proof in Appendix, Section \ref{sec-proof-thmUniversal}. We stress the importance of these properties, both from the standpoint of finding optimal mentored distributions (see also Section \ref{sec-algo-prop-md}) and also for the particular case of the total variation, whose remarkable properties already included modeling the optimal rejection metric for SD \citep[Theorem 2]{yin2024theoretical}.

\section{Mentored decoding meets boosting}\label{sec-boost}

In this Section, we connect mentored decoding as analyzed in Section \ref{sec-men-dec} to one of ML's most famous training framework, boosting \citep{sfBF}. Our main boosting algorithm is different from the classical blueprint, so we shall have to introduce and analyze it first. But before, we define the general boosting framework. We have access to a training sample $\mathcal{S} \defeq \{w_{i}, (\ve{x}_i, \ve{y}_i) , \ve{x}_i \in \mathcal{X}, \ve{y}_i \in \mathcal{Y}\}_{i \in [m]}$ of $m$ \textit{examples}. Here, $\mathcal{X}$ is the set of all possible inputs of a LLM, including prompts, etc.. We adopt the lightweight approach of \citet{zzrsMC} for $\mathcal{Y}$. $\mathcal{Y} \defeq \{\ve{y} \in \mathbb{R}^n : \ve{1}^\top \ve{y} = 0\}$ and $\ve{y}_i$ has two possible coordinates, $1/n_i$ and $-1/(n-n_i)$; $y_{ij} = 1/n_i$ iff token $j$ is a potential next token for $\ve{x}_i$ and $n_i$ is the number of such potential next tokens. We denote $\mathcal{Y}_i \defeq \{j : y_{ij} = 1/n_i\}$ and $\overline{\mathcal{Y}}_i \defeq [n] \backslash \mathcal{Y}_i $. We assume without loss of generality that $0<n_i<n, \forall i \in [m]$ so none of these sets is empty. Finally, $\ve{0} < \ve{w} \in \Delta_m$ is the initial weight vector of the training sample, usually uniform.

\paragraph{Boosting in our LLM context} Even when our embedding of boosting in mentored decoding shall be made with two models, one drafter and one target, we first develop a general theory for any number of such models. Also, distinguishing drafters and targets makes no real sense for the general boosting theory we first develop, so let us assume first we have a sequence of $T> 1$ LLMs whose last layer (real) prediction is denoted $\ve{h}_t : \mathcal{X} \rightarrow \mathbb{R}^n, t \in [T]$. Note that we assume that these models are already available, which makes sense in the current state of LLMs, but we might as well \textit{train} sequentially $T$ models as is usually the case in boosting. The results we present here are oblivious to how the models are made available.

\paragraph{Predictions} Should we use separately each of these models, the corresponding probability vectors $\ve{p}_t$ to predict the next token would be proportional to $\exp \ve{h}_t$. In our case however and for technical reasons, we are going to renormalize $\ve{h}_t$ by a scalar positive constant computed from the training sample, thus playing no role in ranking probabilities. Let
\begin{eqnarray}
\ve{p}_t(\ve{x}) & \defeq & \frac{1}{Z_t} \cdot \exp\left(\frac{1}{h_{t,\infty}}\cdot \ve{h}_t(\ve{x})\right) \in \Delta_n, \quad \mbox{ with } h_{t,\infty}\defeq \max_{j\in [m]} \|\ve{h}_{t}(\ve{x}_j)\|_\infty,
\end{eqnarray}
where $Z_t$ is used for normalization. Importantly, $h_{t,\infty}$ is the max $L_\infty$ norm of $\ve{h}_{t}$ \textit{on training}: it is thus trivially computable and finite. In boosting's jargon, each such predictor is called a \textit{weak} predictor because boosting provides a way to craft an ensemble from each of them with rapidly improving quality even when each weak predictor is just slightly better than random guessing. Boosting works by combining all last layers -- or equivalently all these $T$ probability vectors -- to get a boosted output $\ve{p}_T(\ve{x})$. The quality of $\ve{p}_T(\ve{x})$ is evaluated by comparing, for each training example $i\in [m]$, output probabilities for its potential next tokens in $\mathcal{Y}_i$ to the other ones in $\overline{\mathcal{Y}}_i$. Specifically, we want the coordinates of $\ve{p}_T(\ve{x})$ in $\mathcal{Y}_i$ to be large enough compared to those in $\overline{\mathcal{Y}}_i $, where comparisons use the geometric average of the corresponding sets. The geometric average has the essential property to be zero-attracting: for such successful examples, it will prevent in general \textit{any} coordinate of $\ve{p}_T(\ve{x})$ in $\mathcal{Y}_i$ to be too close to zero.

\subsection{The boosting scheme for general $T$}

Our boosting scheme relies on a substantial generalization of \citep{DBLP:journals/ai/NockN07} to the multiclass case and geared to the analysis of probabilities and not real valued predictions. Define the sequence of weights $\ve{w}_t \in \Delta_m, t\in [T]$ such that $\ve{w}_1 \defeq \ve{w}>\ve{0}$ is the weight vector in $\mathcal{S}$ and otherwise obeys the recurrence
      \begin{eqnarray}
w_{(t+1)i} & \defeq & w_{ti} \cdot \frac{1 -\frac{\mu_t}{2 h_{t, \infty}}\cdot \ve{y}_{i}^\top \ve{h}_t(\bm{x}_i)}{1-\mu_t^2}, t \in [T], i \in [m]\label{eq-def-weights}
      \end{eqnarray}
      (note that formula \eqref{eq-def-weights} is self-normalized in $\Delta_m$: there is no normalization coefficient as e.g. in AdaBoost), where coefficient $\mu_t$ is an \textbf{edge} defined as
      \begin{eqnarray}
        \mu_t & \defeq & \frac{1}{2 h_{t,\infty}} \cdot \sum_{i \in [m]} w_{ti} \cdot \ve{y}_{i}^\top \ve{h}_t(\bm{x}_i), t \in [T]. \label{eq-def-mut}
      \end{eqnarray}
      H{\"o}lder's inequality and the definition of $\mathcal{Y}$ imply $|\ve{y}_{i}^\top \ve{h}_t (\bm{x}_i)| \leq \|\ve{y}_{i}\|_1 \cdot h_{t, \infty}$ and $\|\ve{y}_{i}\|_1 = n_i * (1/n_i) + (n-n_i) * (1/(n-n_i)) = 2$, so $|\mu_t| \leq 1$. In fact, let us assume without loss of generality that $|\mu_t| < 1$ otherwise either $\ve{h}_t(\bm{x}_i)$ or $-\ve{h}_t(\bm{x}_i)$ has the same signs as $\bm{y}_i$ for all $i\in [m]$ and so we are guaranteed $p_{tj}(\ve{x}_i) > p_{tk}(\ve{x}_i)$ for any $i \in [m], j \in \mathcal{Y}_i, k \in \overline{\mathcal{Y}}_i$, which would defeat the purpose of boosting $\ve{h}_t$. Secondly, if $\mu_t \leq 0$ then by just flipping $\ve{h}_t \rightarrow -\ve{h}_t$, we get the new $\mu_t \geq 0$. To summarize, we observe
      \begin{eqnarray*}
        \mu_t & \in & [0,1), \forall t \in [T].
      \end{eqnarray*}
The fact that our weight update does without normalization coefficient is a crucial differentiator with the AdaBoost lineage of boosting algorithms \citep{10.1214/aos/1024691352,sfBF}: it saves the algorithmic computation of the normalizing coefficient, and more importantly, the simple closed form of the weights shall be important for the analysis of boosting in the context of mentored decoding.
      We now construct $\tilde{\ve{\pi}}_T(\ve{x})$, the boosted output. We voluntarily name it with the same symbol as the mentored distribution of mentored decoding.
      \begin{definition}
        For $\mu_t$ defined in \eqref{eq-def-mut}, let
         \begin{eqnarray}
        c_t & \defeq & \frac{1}{4}\cdot \ln\left( \frac{1+\mu_t}{1-\mu_t}\right), t \in [T]. \label{eq-def-ct}
         \end{eqnarray}
         The boosted output model $\tilde{\ve{\pi}}_T: \mathcal{X} \rightarrow \Delta_m$ is defined as:
      \begin{eqnarray}
     \tilde{\ve{\pi}}_T (\ve{x}) & \defeq & \frac{1}{Z_T} \cdot \prod_{t=1}^{T} \left(\ve{p}_t(\ve{x})\right)^{\frac{c_t}{\sum_{u \in [T]} c_u}} \in \Delta_n,\label{eq-def-boosted-p}
      \end{eqnarray}
      where $Z_T$ is the normalization coefficient.
    \end{definition}    
    Note that in the context of next token prediction, the full computation of \eqref{eq-def-boosted-p} is optional. In particular, we can always spare the computation of $Z_T$. We now analyze the boosting abilities of $\tilde{\ve{\pi}}_T $.

\subsection{Boosting the individual predictions in $\tilde{\bm{\pi}}_T$: main theorem}

For any set of non negative reals $\mathcal{A}$, $\overline{\mathcal{A}}^G$ denotes the geometric average with uniform weights of the elements of $\mathcal{A}$: for example, $\overline{\{1, 2\}}^G \defeq 1^{1/2} \cdot 2^{1/2} = \sqrt{2}$.
    \begin{theorem}\label{th-boost-P}
      For any $T>1$, suppose without loss of generality that the sequence $\mu_1, \mu_2, ..., \mu_T$ is non-negative and with expectation $\expect[\mu]> 0$. Then $\tilde{\ve{\pi}}_T$ in \eqref{eq-def-boosted-p} satisfies
 \begin{eqnarray}
     \pr_{i \sim \ve{w}}\left[ \overline{\{\tilde{\pi}_{T,j}(\ve{x}_i), j \in \mathcal{Y}_i\}}^G \leq \rho \cdot \overline{\{\tilde{\pi}_{T,j}(\ve{x}_i), j \in \overline{\mathcal{Y}}_i\}}^G \right] \leq \exp \left(-\frac{1}{6} \cdot \sum_{t=1}^T \mu_t^2\right), \forall \rho \leq \exp\left(\frac{2Q}{3}\right), \label{eq-bound-boosted-2}
 \end{eqnarray}
with $Q \defeq \expect[\mu] + \frac{\var[\mu]}{\expect[\mu]}$ and $\var[\mu]$ is the variance of the sequence $\mu_1, \mu_2, ..., \mu_T$.
    \end{theorem}
    Proof in Appendix, Section \ref{sec-proof-th-boost-P}.
    \begin{remark}
      The RHS of \eqref{eq-bound-boosted-2} also applies to the boosting scheme of the seminal paper of \citet{10.1214/aos/1024691352}, which does not give an explicit rate for the empirical risk, apart from mentioning that it is exponentially decreasing.
\end{remark}
Hence, we are guaranteed that a rapidly growing proportion of training sample will have a geometric average of the probabilities for the true next tokens larger than the geometric average of the other "bad" tokens by a "margin" factor $\rho > 1$. Note the quantitative advantage of the geometric average being zero-attracting for those "good" examples: if the geometric average of the bad tokens is $>0$, then no coordinate in the good tokens can be zero. We now summarize a more qualitative analysis based on Theorem \ref{th-boost-P}. 
\begin{definition}\label{def-ba}
  The boosting advantage of the sequence $\{\ve{h}_{t}\}_{t\in [T]}$ is the quantity
\begin{eqnarray}
A\left(\{\ve{h}_{t}\}_{t\in [T]}\right) & \defeq & \sum_{t=1}^T \mu_t^2, \label{def-boostad}
\end{eqnarray}
where $\mu_t$ is defined in \eqref{eq-def-mut}. 
\end{definition}
Introducing boosting's so-called Weak Learning Assumption \citep{10.1214/aos/1024691352,DBLP:journals/ai/NockN07}:
    \begin{align}
      \exists \upgamma > 0 : |\mu_t| \geq \upgamma, \forall t \in [T], \label{eq-def-wla}\tag{\textbf{WLA}}
    \end{align}
    we get an $\Omega(T)$ boosting advantage:
    \begin{eqnarray}
A\left(\{\ve{h}_{t}\}_{t\in [T]}\right) & \geq & \upgamma^2 \cdot T, \label{b-adv-lin}
\end{eqnarray}
    and so under \eqref{eq-def-wla}, for any $\varepsilon > 0$, we have that a proportion $\geq 1 - \varepsilon$ of the training examples observe $\overline{\{\tilde{\pi}_{T,j}(\ve{x}_i), j \in \mathcal{Y}_i\}}^G >  \rho \cdot \overline{\{\tilde{\pi}_{T,j}(\ve{x}_i), j \in \overline{\mathcal{Y}}_i\}}^G$ as soon as 
    \begin{eqnarray}
T & = & \left\lceil\frac{6}{\upgamma^2} \cdot \log \frac{1}{\varepsilon}\right\rceil,\label{eq-bsup-T}
    \end{eqnarray}
    and the largest possible $\rho$ is $\geq \exp(2\upgamma/3)$. Note also that if each $\ve{h}_{t}$ were to be chosen uniformly at random in a set of, say, unit-$L_2$ norm predictors, then the expectation over randomness would give $\expect[\ve{y}_{i}^\top \ve{h}_t(\bm{x}_i)] = 0$ for each $i \in [m]$, which justifies the name weak predictors for our sequence of $\ve{h}_{t}$ as the \eqref{eq-def-wla} only requires them to slightly beat such a random performance. Finally, in the context of LLMs, note also that \eqref{eq-bound-boosted-2} provides a simple way to cherry pick a subset of available pretrained models, by greedily picking the one maximizing $|\mu_t|$\footnote{The greedy selection may not be optimal over all sequences of inclusion, see Section \ref{sec-disc}.}.

    We now have the tools to connect boosting and mentored decoding. We achieve this in two Subsections, first tackling the case of the total variation, and then the general case. The way we fold boosting in is different in both cases.

\subsection{Mentored decoding and boosting: the case of total variation}\label{subsec-boost-TV}

\begin{figure*}
  \centering
  \resizebox{0.8\columnwidth}{!}{\begin{tabular}{cc}\Xhline{2pt}
 \includegraphics[trim=0bp 0bp 0bp 0bp,clip,width=0.3\columnwidth]{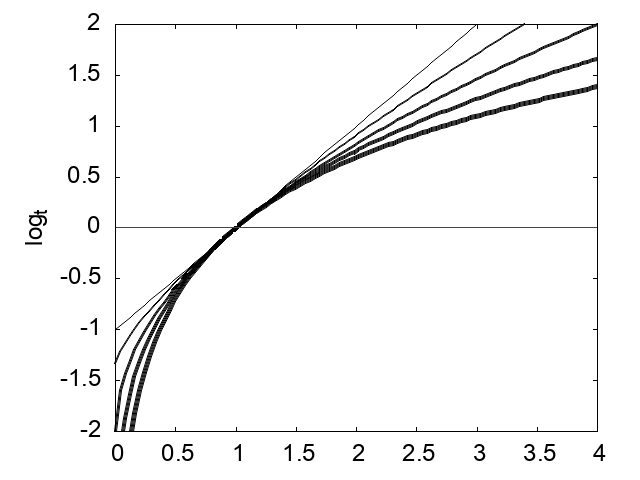}&
                                                        \includegraphics[trim=0bp 0bp 0bp 0bp,clip,width=0.28\columnwidth]{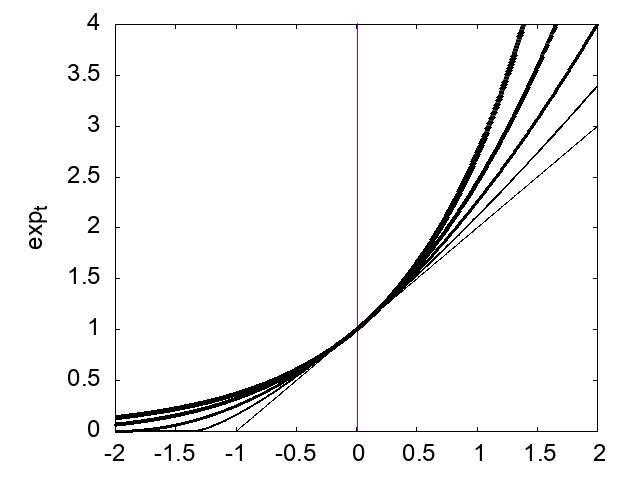}  \\ \Xhline{2pt}
    \end{tabular}}
  \caption{Plots of $\log_t$ (left) and $\exp_t$ (right), for $t = 0, 1/4, 1/2, 3/4, 1$ in black curves where thickness increases with $t$ (see text).}
    \label{fig:t-exp-log}
  \end{figure*} 

Mentored decoding builds a mentored distribution $\ve{\pi}$ that depend on the output $\ve{p}, \ve{q}$ of the drafter and target. From the boosting standpoint, which analyzes the composite / ensemble \textit{model} producing $\ve{\pi}$, we thus end up analyzing the boosting ability of potentially \textit{as many ensemble models} as there can be for \textit{any} outputs of the drafter and target. The connection between mentored decoding and boosting is made by a combination of the models' outputs specific to each output $\ve{p}, \ve{q}$, in such a way that it always yields guarantees on the exponential rate in \eqref{eq-bound-boosted-2} while being optimal from the mentored decoding problem \eqref{def-pb-f-md-2-general} \eqref{def-pb-f-md-1-general}, and the key parameters of these two problems -- the edges \eqref{eq-def-mut} for boosting, the divergence constraint $D$ for mentored decoding -- depend on a real parameter function of $\ve{p}$ and $\ve{q}$, whose existence is guaranteed by Assumption \ref{assum-mda}. From now on, in the context of mentored decoding, the boosting setting corresponds to the specific case of $T=2$ models. This is obviously a very small number of models, but boosting has this property that the marginal improvement of the first few models due to boosting is usually dramatically larger than for the next ones: combining drafter and target models may be sufficient for the boosted model to be better than each of them. Of course, our boosting setting may also apply to mentored decoding settings involving more than two models.

\paragraph{Computation of the mentored distribution} The key non-trivial constraint for any mentored distribution $\ve{\pi}$ to be optimal for $f_{\mathrm{TV}}$-MD is to belong to the hyperrectangle defined by $\ve{p}$ and $\ve{q}$ \eqref{def-opt-tv-tv}. We analyze the construction of the boosted model in \eqref{eq-def-boosted-p} and show how its output from $\ve{p}$ and $\ve{q}$ can be compliant with this constraint. For the analysis, we introduce the tempered versions of log and exp \cite[Chapter 7]{nGT}:
  \begin{eqnarray}
\log_t (z) \defeq \frac{1}{1-t}\cdot\left(z^{1-t}-1\right) & , & \exp_t (z) \defeq \left[1+(1-t) z\right]^{1/(1-t)}_+ \quad ([z]_+ \defeq \max\{0,z\})\label{defExpLogT},
  \end{eqnarray}
  where the case $t=1$ is the extension by continuity to the $\log$ and $\exp$ functions, respectively (see Figure \ref{fig:t-exp-log} for examples). Our focus is essentially on $t\in (0,1)$, for which the concavity / convexity of functions is the same as for $t=1$, see also \citet{DBLP:conf/aaai/AmidNNW24,DBLP:conf/aistats/AmidNW23,DBLP:conf/nips/NockAW23,nGT} for further relevant properties. The following Lemma is central to our analysis. We let $\iver{.}$ be Iverson's bracket \citep{kTN}, i.e. the Boolean truth value of the predicate inside.
\begin{lemma}\label{lem-tempered-mentor}
  For any $\ve{p}, \ve{q} \in \Delta_n$ satisfying $\ve{p}, \ve{q} > \ve{0}$, any $0\leq \alpha \leq 1$, denote
  \begin{eqnarray}
    i^* & = & \arg\max_i \left(\frac{q_i}{p_i}\right)^{\iver{p_i > q_i}-\alpha}\label{eq-def-istar}
  \end{eqnarray}
  (without loss of generality, this is a singleton). Suppose the following holds:
  \begin{eqnarray}
    \exp_\alpha \expect_{i \sim \ve{p}} \log_\alpha \frac{q_i}{p_i} & \geq & \max_i \min \left\{\frac{q_i}{p_i}, \frac{p_i}{q_i}\right\} \quad \mbox{ if $p_{i_*} > q_{i_*}$}, \label{eq-cond-tmp-1}\\
    \exp_{1-\alpha} \expect_{i \sim \ve{q}} \log_{1-\alpha} \frac{p_i}{q_i} & \geq & \max_i \min \left\{\frac{q_i}{p_i}, \frac{p_i}{q_i}\right\} \quad \mbox{ if $p_{i_*} < q_{i_*}$}. \label{eq-cond-tmp-2}
  \end{eqnarray}
  Then if we let
  \begin{eqnarray}
    \ve{v}_\alpha & \propto & \ve{p}^\alpha \odot \ve{q}^{1-\alpha} \in \Delta_n \label{def-mu-alpha}
  \end{eqnarray}
  the distribution obtained from the coordinate-wise geometric average of $\ve{p}$ and $\ve{q}$, the following holds:
  \begin{eqnarray}
    \ve{v}_\alpha \in [\min\{\ve{p}, \ve{q}\}, \max\{\ve{p}, \ve{q}\}].\label{eq-const-tv-md}
  \end{eqnarray}
\end{lemma}
Proof in Appendix, Section \ref{sec-proof-lem-tempered-mentor}. Note that the Lemma is useful only if $\ve{p} \neq \ve{q} $: otherwise, it can hold only when $\ve{p} = \ve{q}$. When $\ve{p} \neq \ve{q} $ and $\ve{p}, \ve{q} > \ve{0}$, it is easy to show that $\log_t$ continuously converges to $z \mapsto z - 1$ as $t\rightarrow 0$ so that the LHS of \eqref{eq-cond-tmp-1} (resp. \eqref{eq-cond-tmp-2}) continuously converges to 1 as $\alpha \rightarrow 0$ (resp. $\alpha \rightarrow 1$), since we observe $\exp_t(0) = 1, \forall t \in [0,1]$. Since the RHS are $<1$, \eqref{eq-cond-tmp-1} (resp. \eqref{eq-cond-tmp-2}) necessarily holds for any $0<\alpha<\alpha_*$ (resp $1-\alpha_*<\alpha<1$) for a small enough $\alpha_*<1$. So under our Assumption \ref{assum-mda}, mentored decoding for the output can be accompanied by a "qualitative" form of boosting for the models. We now complete it with a quantitative one, first describing the ensemble model.

\paragraph{The combination of drafter and target} We now define two key parameters to analyze the imbrication of boosting and mentored decoding.
\begin{definition}\label{def-epsilonpq}
For any $\ve{p}$, $\ve{q}$ complying with Assumption \ref{assum-mda}, let $\rhopq, \epsilonpq$ denote any reals such that $\epsilonpq \in (0,1], \rhopq \geq 0, \epsilonpq (1+\rhopq) \leq 1$ and:
\begin{align}
    \begin{array}{ccc}
      \min_i \max \left\{\frac{q_i}{p_i}, \frac{p_i}{q_i}\right\} \geq 1 + \rhopq,\quad & \max_i \frac{q_i}{p_i} \leq \frac{1}{\epsilonpq}, \quad & \min_i \frac{q_i}{p_i} \geq \epsilonpq.
\end{array} \label{eq-def-eps}\tag{\textbf{ED}}
\end{align}
\end{definition}
While the two rightmost conditions bound the most dissimilar coordinates in $\ve{q}$ and $\ve{p}$, the leftmost is a condition on the most similar one, all in term of density ratio. For example, if $\ve{p} \defeq (3/10, 1/10, 3/5)$ and $\ve{q} \defeq (2/5, 1/5, 2/5)$, the coordinate realizing the leftmost to the rightmost condition in \eqref{eq-def-eps} (with equality) will be the first, second and third respectively. Intuitively, the "freedom" in the joint choice of $\ve{p}$ and $\ve{q} $ augments as $\epsilonpq$ and $\rhopq$ decrease: as $\rhopq, \epsilonpq \rightarrow 0$, the most similar coordinates can be as close to 1 as desired, the most dissimilar coordinates in terms of the ratio $q_./p_.$ can span as much as $\mathbb{R}_+$ as desired. In our context however, we can expect the opposite: drafter and target outputs should achieve some level of agreement in their outputs because they were trained to achieve some quality level in their predictions. The more they would agree on $\ve{p}$ and $\ve{q}$ coordinate-wise, the larger we can pick $\epsilonpq$. We however need to eventually decrease $\rhopq$ for \eqref{eq-def-eps} to remain true, keeping in mind we must keep $\rhopq>0$ because of Assumption \ref{assum-mda}.

From the boosting standpoint, we expect the target model to be better than the drafter from the edge standpoint, so the boosted ensemble first includes the target and then the drafter (interestingly enough, this greedy strategy can prove suboptimal, see Section \ref{sec-disc}). While the first boosting coefficient strictly follows \eqref{eq-def-ct}, the boosted coefficient of the drafter may be (nonlinearly) \textit{scaled down} to ensure that the resulting mentored distribution $\ve{\pi}$ is optimal for the $f_{\mathrm{TV}}$-MD problem. This scaling depends on $\epsilonpq$ and $\rhopq$. First, $\alpha$ in \eqref{def-mu-alpha} is simply
\begin{eqnarray}
\alpha & \defeq & \frac{c_2}{c_2+c_\target}, \label{eq-def-alpha-boost}
\end{eqnarray}
where both $c$s are computed using \eqref{eq-def-ct}. The edge $\mu_\target$ is as in \eqref{eq-def-mut} and thus its computation does not depend on $\ve{p}, \ve{q}$. However $\mu_2$ used for $c_2$ is $\mu_\drafter$ eventually clamped:
\begin{eqnarray}
\mu_2 & \defeq & \min \left\{\mu_\drafter, \frac{\left(1+\mu_\target\right)^{\rhopq \epsilonpq}-\left(1-\mu_\target\right)^{\rhopq \epsilonpq}}{\left(1+\mu_\target\right)^{\rhopq \epsilonpq}+\left(1-\mu_\target\right)^{\rhopq \epsilonpq}}\right\},\label{eq-def-mu2}
\end{eqnarray}
where $\mu_\drafter$ follows \eqref{eq-def-mut}.
      
\paragraph{Main theorem} Armed with these definitions, we now prove the Theorem that brings $f_{\mathrm{TV}}$-MD optimality and boosting.

\begin{theorem}\label{thm-boost-and-mentor-TV}
  The computation of $\alpha$ as in \eqref{eq-def-alpha-boost} simultanously yields:
  \begin{itemize}
  \item $\ve{v}_\alpha$ in \eqref{def-mu-alpha} is optimal for the $f_{\mathrm{TV}}$-MD problem for the choice $D = D_{f_\mathrm{TV}} (\ve{v}_\alpha \|\ve{q})$;
  \item the corresponding boosted model with coefficients $c_\target, c_2$ has boosted advantage satisfying $A\left(\{\ve{h}_1 \defeq \ve{h}_\target, \ve{h}_2 \defeq \ve{h}_\drafter\}\right) \geq  (1 +\rhopq^2 \epsilonpq^2)\cdot \mu^2_\target$.
  \end{itemize}
  Finally, we observe
  \begin{eqnarray}
    D_{f_\mathrm{TV}} (\ve{v}_\alpha \|\ve{q}) & \leq & \rhopq\epsilonpq\label{eq-bound-tv-alpha}.
  \end{eqnarray}
  \end{theorem}
  Proof in Appendix, Section \ref{sec-proof-thm-boost-and-mentor-TV}. Importantly, the boosting advantage can be free from any dependence in $\epsilonpq$ and $\rhopq$ if there is no clamping of $\mu_2$ -- in such a case, we get $A = \mu^2_\target + \mu^2_\drafter$. In particular, quantity $\epsilonpq$ limits the quality of boosting and in the limit as $\epsilonpq \rightarrow 0$, the boosted guarantees of combining two models vanish and we can only guarantee a quality identical to the target model's. What we should expect in such a case, given that there is then a form of convergence of $\ve{v}_\alpha$ to the target output $\ve{q}$ in this case from \eqref{def-mu-alpha}, is that the fate of boosting clearly becomes a blessing for the TV divergence, namely that the authorized bound $D$ also converges to zero as a function of $\epsilonpq$. This is what \eqref{eq-bound-tv-alpha} guarantees.

\subsection{Mentored decoding and boosting: general case}\label{subsec-boost-f}

We now make use of the approximation to $f$-MD in Theorem \ref{thmAPPROX1} for boosting. Our path to get here is much different from Theorem \ref{thm-boost-and-mentor-TV}: the case of $f_{\mathrm{TV}}$-MD yields a huge set of optimal solutions which we showed can contain convenient boosting solutions as well. For general $f$ however, the set of optimal solution is in general as small as a singleton (if $f$ strictly convex). Instead of hammering boosting solutions in such a small set, we are going to show that the \textit{approximate} solutions to $f$-MD of Theorem \ref{thmAPPROX1} have \textit{de facto} nice boosting properties. In other words, the coordinates of the mentored distribution for the potential next tokens cannot be "too small" with respect to other coordinates, that are associated to tokens that cannot be potential next tokens. The amount by which both sets of coordinates compare to each other depends on parameters evaluated on a sample from the domain.

\begin{theorem}\label{thm-boost-plus-md-general-f}
  For any drafter and target models, and any sample $\mathcal{S}' \defeq \{w'_{i}, (\ve{x}_i, \ve{y}_i) , \ve{x}_i \in \mathcal{X}, \ve{y}_i \in \mathcal{Y}\}_{i \in [m']}$, denote respectively $\ve{p}_i, \ve{q}_i$ the outputs of drafter and target on input $\ve{x}_i$. Suppose the edges of the target and drafter on $\mathcal{S}'$ satisfy $\mu_\target, \mu_\drafter > 0$ \eqref{eq-def-mut}. For any $i \in [m']$ and any $(a_i,b_i) \in \mathcal{C}(\ve{p}_i, \ve{q}_i)$, let
\begin{eqnarray}
\ve{\pi}_i & \defeq & \ve{p}_i \mbox{ clamped to } [1-b_i, 1+a_i]\cdot \ve{q}_i, \forall i\in [m'] \label{eq-def-pii}
\end{eqnarray}
be the mentored distribution defined from drafter and target via the respective $\ve{r}_i, \ve{s}_i$ in \eqref{eq-defR}, \eqref{defS}. Let
\begin{eqnarray}
  k_i & \defeq & \sqrt{\frac{1+a_i}{1-b_i}} \quad (\geq 1).\label{eq-def-ki}
\end{eqnarray}
Then this mentored distribution satisfies
 \begin{eqnarray}
   \pr_{i \sim \ve{w}'}\left[ \overline{\{\pi_{ij}  : j \in \mathcal{Y}_i\}}^G \leq \rho \cdot \min\left\{k_i \cdot \epsilon^{\frac{c_2}{c_1+c_2}}_{\ve{p}_i\ve{q}_i}, \frac{1}{k_i}\right\}^2 \cdot \overline{\{\pi_{ij}  : j \in \overline{\mathcal{Y}}_i\}}^G \right]  & \leq & \exp \left(-\frac{\mu^2_1 + \mu^2_2}{6}\right), \label{eq-bound-boosted-md}
 \end{eqnarray}
for any $\rho \leq \exp\left(\frac{2\cdot (\mu^2_1 + \mu^2_2)}{3\cdot(\mu_1 + \mu_2)}\right)$. Here, $\mu_1\defeq \mu_\target, \mu_2\defeq \mu_\drafter$, $c_1, c_2$ are defined in \eqref{eq-def-ct} and $\epsilon_{\ve{p}_i \ve{q}_i}$ is as in \eqref{eq-def-eps}.
\end{theorem}
Proof in Appendix, Section \ref{sec-proof-thm-boost-plus-md-general-f}. We have two important remarks regarding Theorem \ref{thm-boost-plus-md-general-f}. First, any $f$-MD problem has optimal mentored distributions with the general form \eqref{eq-def-pii} (Theorem \ref{thmAPPROX1}), for any generator $f$ \eqref{eq-const-f}, so Theorem \ref{thm-boost-plus-md-general-f} applies to all instances of $f$-MD. Second, the Theorem holds for \textit{any} sample $\mathcal{S}' \defeq \{w'_{i}, (\ve{x}_i, \ve{y}_i) , \ve{x}_i \in \mathcal{X}, \ve{y}_i \in \mathcal{Y}\}_{i \in [m']}$ for which $ \mu_\target, \mu_\drafter > 0$, which is arguably a very weak assumption -- in fact weaker than the weak learning assumption \eqref{eq-def-wla}. Given a domain for which a sample $\mathcal{S}'$ is available, the Theorem can be used to get an indication of the general quality of mentored distributions obtained from Theorem \ref{thmAPPROX1}. The Theorem also carries qualitative value if we consider that we can leave the boosting parameters implicit, since we do not need the boosting model. In such a case, if the target model is substantially better than the drafter, $c_2/(c_1+c_2)$ is very small and we can reduce the $\min$ to the $k_i$ dependent part, making the factor of the geometric average $\Omega(\rho/k_i^2)$, i.e. independent of the boosting parameters. This makes $k_i$, which depends on the clamping in the mentored distribution, directly influence the quality of the coordinates for potential next token vs others. We also shall see, on a toy simulation, that $k$ can stay very close to 1 even for a substantial increase of the acceptance probability compared to $\pacc(SD)$ (Table \ref{tab:sim-3}).

\section{Algorithms and related properties}\label{sec-algo-prop-md}

  \begin{algorithm}[t]
\caption{\cabbreakpoints $(\textsc{r}, \textsc{c}, a, b, i, j, q_A, q_B, a_{\max}, b_{\max})$}\label{algo-cabbreakpoints}
\begin{algorithmic}
  \STATE  \textbf{Input:} sorted array $\textsc{r} \defeq \{(p_i / q_i, q_i)\}_{i=1}^n$ in increasing order of the ratio $p_./q_.$; current set of breakpoints $\textsc{c}$, parameters $0\leq a \leq a_{\max}$ all in \eqref{kktconst1}, parameters $0\leq b \leq b_{\max}$ all in \eqref{kktconst2}, indexes $i,j \in [n]$, probabilities $q_A, q_B \in [0, 1]$; // Ratio values in $\textsc{r}$ are called "ticks"
  \STATE  Step 1 : \textbf{if} $a \leq a_{\max}$ \textbf{and} $b \leq b_{\max}$ \textbf{then} $\textsc{c} \leftarrow \textsc{c} \cup \{(a, b)\}$; // add current breakpoint
  \STATE  Step 2 : \textbf{if} $a > a_{\max}$ \textbf{or} $b > b_{\max}$ \textbf{or} $j = 0$ \textbf{or} $i = n - 1$ \textbf{then return} $\textsc{c}$;
  \STATE  Step 3 : // compute thresholds
  \begin{eqnarray*}
    \tau_a \leftarrow \left(r_{i0} - 1 - a\right)\cdot q_A & ; & \tau_b \leftarrow \left(1-b-r_{j0}\right)\cdot q_B;
  \end{eqnarray*}
    \STATE  Step 4 : \textbf{if} $\tau_b < \tau_a$ \textbf{then}  // $b$ reaches a tick
    \STATE \hspace{1.5cm} 4.1 : $a  \leftarrow a + \frac{\tau_b}{q_A}$; 
    \STATE \hspace{1.5cm} 4.2 : $b \leftarrow 1 - r_{j0}$; 
    \STATE \hspace{1.5cm} 4.3 : $q_B \leftarrow q_B - r_{j1}$;
    \STATE \hspace{1.5cm} 4.4 : $j \leftarrow j - 1$;
    \STATE  Step 5 : \textbf{else if} $\tau_b > \tau_a$ \textbf{then} // $a$ reaches a tick
    \STATE \hspace{1.5cm} 5.1 : $a \leftarrow r_{i0} - 1$;
    \STATE \hspace{1.5cm} 5.2 : $b  \leftarrow b +\frac{\tau_a}{q_B}$;
    \STATE \hspace{1.5cm} 5.3 : $q_A \leftarrow q_A - r_{i1}$;
    \STATE \hspace{1.5cm} 5.4 : $i \leftarrow i + 1$;
    \STATE  Step 6 : \textbf{else} // $a$ and $b$ reach a tick
    \STATE \hspace{1.5cm} 6.1 : $a \leftarrow r_{i0} - 1$;
    \STATE \hspace{1.5cm} 6.2 : $b \leftarrow 1 - r_{j0}$; 
    \STATE \hspace{1.5cm} 6.3 : $q_A \leftarrow q_A - r_{i1}$;
    \STATE \hspace{1.5cm} 6.4 : $q_B \leftarrow q_B - r_{j1}$;
    \STATE \hspace{1.5cm} 6.5 : $i \leftarrow i + 1$;
    \STATE \hspace{1.5cm} 6.6 : $j \leftarrow j - 1$;
  \STATE  Step 9 : \cabbreakpoints $(\textsc{r}, \textsc{c}, a, b, i, j, q_A, q_B, a_{\max}, b_{\max})$;
\end{algorithmic}
\end{algorithm}

We now study the algorithmic side of the theory developed so far. Note that the algorithmic efficiency of boosting to compute $\mu_\drafter$ and $\mu_\target$ following \eqref{eq-def-mut} is orthogonal to the mentored decoding part and does not depart from boosting's blueprint complexity, save of course the normalization of boosting's distribution that we do not need to perform, unlike AdaBoost. Only $\epsilonpq, \rhopq$ in \eqref{eq-def-eps} need to be computed in addition. For any given $\ve{p}, \ve{q}$, the complexity is $O(n)$. This is no more than the computation of the resampling probability (Subsection \ref{subsec-one-two}), yet it also gets in the computation of acceptance probabilities and thus brings an additional computation cost when accepting tokens. This, of course, can be reduced, e.g. by quantization of the vectors. We now investigate the mentored decoding side, which has several non-trivial and very useful properties from an algorithmic standpoint.

    \begin{algorithm}[t]
\caption{\allcabbreakpoints $(\textsc{r}, a_{\max}, b_{\max})$}\label{algo-allcabbreakpoints}
\begin{algorithmic}
  \STATE  \textbf{Input:} sorted array $\textsc{r} \defeq \{(p_i / q_i, q_i)\}_{i=1}^n$ in increasing order of the ratio $p_./q_.$, $a_{\max} > 0$ as in \eqref{kktconst1}, $b_{\max} > 0$ as in \eqref{kktconst2};
  \STATE  \textbf{Output:} sequence of breakpoints $\textsc{c}(\ve{p}, \ve{q})$;
  \STATE  Step 1 : compute initial values
  \STATE \hspace{1.5cm} 1.1 : $ i  \leftarrow  \min \{k : r_{k0} > 1\}$; // $r_{k0}\defeq $ coordinate 0 of couple $\#k$ in $\textsc{r}$
  \STATE \hspace{1.5cm} 1.2 : $ j  \leftarrow  \max \{k : r_{k0} < 1\}$;
  \STATE \hspace{1.5cm} 1.3 : $ q_A  \leftarrow  \sum_{k \geq i} r_{k1}$; // $r_{k1}\defeq $ coordinate 1 of couple $\#k$ in $\textsc{r}$
  \STATE \hspace{1.5cm} 1.4 : $ q_B  \leftarrow  \sum_{k \leq j} r_{k1}$;
  \STATE \hspace{1.5cm} 1.5 : $ a  \leftarrow  0$;
  \STATE \hspace{1.5cm} 1.6 : $ b  \leftarrow  0$;
  \STATE \hspace{1.5cm} 1.7 : $ \textsc{c}  \leftarrow  \emptyset$; 
  \STATE  Step 2 : \cabbreakpoints $(\textsc{r}, \textsc{c}, a, b, i, j, q_A, q_B, a_{\max}, b_{\max})$; 
  \STATE  Step 3 : \textbf{return} $\textsc{c}$; // $\textsc{c}(\ve{p}, \ve{q})$
\end{algorithmic}
\end{algorithm}

\begin{figure*}
  \centering
   \includegraphics[trim=50bp 40bp 50bp 50bp,clip,width=0.8\columnwidth]{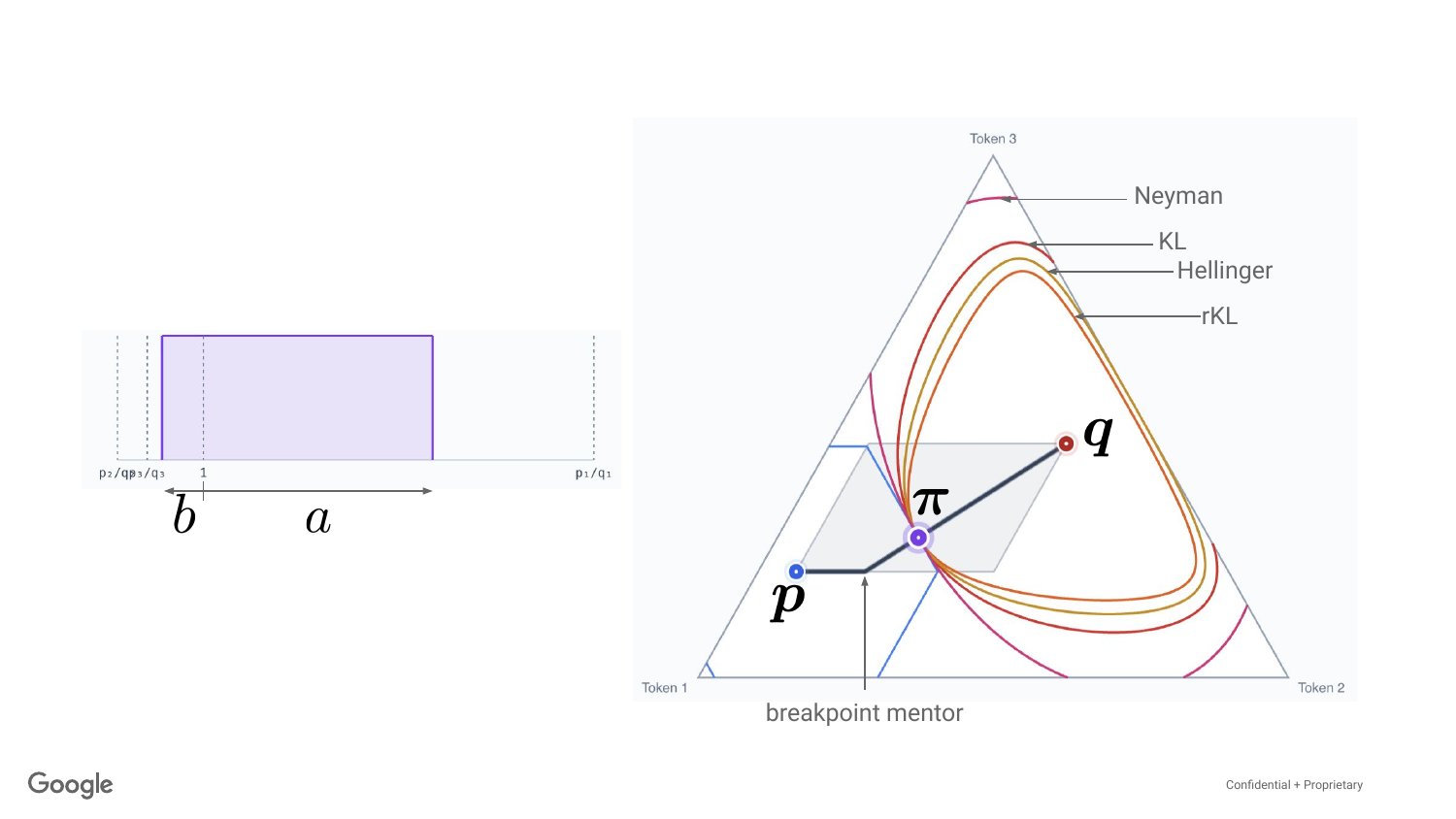} 
  \caption{In the $\Delta_3$ simplex (same setting as Figure \ref{fig:tv-balls}), we illustrate Theorem \ref{thmUniversal} and Lemma \ref{lem-b-from-a-and-c-breakpoints}. The thick dark line is the set of mentored distributions built from querying \cabalgo~in Algorithm \ref{algo-cabalgo}. The bend represents a mentored distribution corresponding to a breakpoint in $\textsc{c}(\ve{p}, \ve{q})$. We display the couple $(a, b)$ corresponding to the mentored distribution $\ve{\pi}$. Finally, for a set of four $f$-divergences (rKL = reverse KL), we display that $\ve{\pi}$ is also the optimal solution in all $\textsc{md}_f(\ve{p}, \ve{q}; .)$ because it is the intersection between the TV-ball defining the optimal objective and the $f$-divergence ball defining the constraint activated in $f$-MD.}
    \label{fig:breakpoints-and-f-balls}
  \end{figure*} 

\paragraph{On computing and querying $\mathcal{C}(\ve{p}, \ve{q})$ in Definition \ref{def-cab}} Theorem \ref{thmAPPROX1} shows that $\mathcal{C}(\ve{p}, \ve{q})$ is key to solving $f$-MD. Given outputs $\ve{p}, \ve{q}$ of the drafter and target, we show how to build a $O(n)$-sized data structure that we call \textit{breakpoints}, $\textsc{c}(\ve{p}, \ve{q})$, in $O(\mathrm{sort}(n))$ time. Such breakpoints are the cornerstone of our approach to get the desired elements of $\mathcal{C}(\ve{p}, \ve{q})$. Quite remarkably, the data structure \textit{does not depend on $f$}, and can thus be used for any applicable $f$ \eqref{eq-const-f} afterwards. The breakpoints are couples $(a,b) \in \mathcal{C}(\ve{p}, \ve{q})$. Apart from the speculative decoding solution for which $a=b=0$, all other couples of $\textsc{c}(\ve{p}, \ve{q})$ satisfy the invariant that at least one of $a$ and $b$ depends on a ratio $p_./q_.$. Each of such ratios being uniquely present at the exclusion of at least one extreme ratio, the cardinal $N$ of $\textsc{c}(\ve{p}, \ve{q})$ satisfies $N \leq n$.

    \begin{algorithm}[t]
\caption{\cabalgo $(\textsc{c}, \tilde{a})$}\label{algo-cabalgo}
\begin{algorithmic}
  \STATE  \textbf{Input:} breakpoints list $\textsc{c} \defeq \textsc{c} (\ve{p}, \ve{q}) \defeq \{(a_i, b_i) : i \in [N]\}$ in strict increasing order of both coordinates, query $0 \leq \tilde{a} \leq a_N$;
  \STATE  \textbf{Output:} $\tilde{b}$ such that $(\tilde{a}, \tilde{b}) \in \mathcal{C}(\ve{p}, \ve{q})$;
  \STATE  Step 1 : find $i$ such that $a_i \leq \tilde{a} \leq a_{i+1}$;
  \STATE  Step 2: // compute $\tilde{b}$
  \begin{eqnarray}
    \tilde{b} & \leftarrow & b + \frac{(\tilde{a}-a) \cdot q(\mathbb{A}_i)}{q(\mathbb{B}_i)} \label{eq-def-tildeb}; \quad \mbox{// $\mathbb{A}_i, \mathbb{B}_i$ are as in \eqref{defAbis}, \eqref{defBbis} for breakpoint $(a_i, b_i)$}
  \end{eqnarray}
  \STATE  Step 3 : \textbf{return} $\tilde{b}$;
\end{algorithmic}
\end{algorithm}

\begin{theorem}\label{thm-cab-from-c-breakpoints}
   The set of breakpoints $\textsc{c}(\ve{p}, \ve{q})$ returned by \allcabbreakpoints~in Algorithm \ref{algo-allcabbreakpoints} satisfy $\textsc{c}(\ve{p}, \ve{q}) \subseteq \mathcal{C}(\ve{p}, \ve{q})$.
 \end{theorem}
 Proof in Appendix, Section \ref{sec-proof-thm-cab-from-c-breakpoints}. It is clear from \allcabbreakpoints~that all breakpoints returned are ordered in strictly increasing values of both coordinates $a$ and $b$. So let us denote $\textsc{c} (\ve{p}, \ve{q}) \defeq \{(a_i, b_i) : i \in [N]\}$, indexing its elements to reflect the order.
 
 \begin{lemma}\label{lem-b-from-a-and-c-breakpoints}
For any $0 \leq \tilde{a} \leq a_N$, $\tilde{b}$ returned by \cabalgo~in Algorithm \ref{algo-cabalgo} satisfies $(\tilde{a}, \tilde{b}) \in \mathcal{C}(\ve{p}, \ve{q})$.
 \end{lemma}
 Proof in Appendix, Section \ref{sec-proof-lem-b-from-a-and-c-breakpoints}. We follow with a series of fundamental properties, most of which follow directly from Lemma \ref{lem-b-from-a-and-c-breakpoints}. 
\begin{lemma}\label{lem-cont-ab}
For any $a$ satisfying \eqref{kktconst1}, there exists $b$ such that $(a,b) \in \mathcal{C}(\ve{p}, \ve{q})$. Reciprocally, for any $b$ satisfying \eqref{kktconst2}, there exists $a$ such that $(a,b) \in \mathcal{C}(\ve{p}, \ve{q})$. Finally, the set of points $(\tilde{a}, \tilde{b})$ from Lemma \ref{lem-b-from-a-and-c-breakpoints} is a continuous strictly increasing function from $\mathbb{R}_+$ to ${\mathbb{R}}_+$.
  \end{lemma}
  Proof in Appendix, Section \ref{sec-proof-lem-cont-ab}. Whenever $f$ is strictly convex, the optimal mentored distribution is unique and thus Lemmata \ref{lem-b-from-a-and-c-breakpoints} and \ref{lem-cont-ab} guarantee the exhaustiveness of our algorithms. If $f$ is not strictly convex, such as for the total variation divergence (see Figure \ref{fig:breakpoints-and-f-balls}), our algorithms elicit one of many solutions.

  \paragraph{Finding optimal mentored distributions} There are two ways to use $\mathcal{C}(\ve{p}, \ve{q})$ for optimal mentored distributions. The first tackles solutions of $f$-MD in \eqref{def-pb-f-md-1-general}, in two steps: first, we find the successive indexes $i$ and $i+1$ in $\textsc{c} (\ve{p}, \ve{q})$ whose mentored distributions $\ve{\pi}_i$, $\ve{\pi}_{i+1}$ satisfy $D \in [D_f(\ve{\pi}_i \| \ve{q}), D_f(\ve{\pi}_{i+1} \| \ve{q})]$. Then, if necessary, we query \cabalgo~for a dichotomic search of the optimum sought to desired precision. In all cases, the whole complexity is $O(n \log n)$. This simple algorithm gets to the optimum at arbitrary desired precision. However, there is a \textit{much} cheaper way to get approximate solutions with guarantees, and it relies on the fast that if in addition to storing couples $(a, b)$, \allcabbreakpoints~also keeps track of the corresponding $q_A, q_B$ computed in the algorithm, then any stored quadruple $(a, b, q_A, q_B)$ allows to compute $D$ in \eqref{Dapprox} in $O(1)$ for any desired $u\in [1-b, 1+a]$. Hence if instead of computing actual values one relies on the corresponding upperbound $\hat{D}$ of $D$ that follows (for a conservative approach to approximation), the whole procedure described above drops in complexity from $O(n \log n)$ to $O(\log(n))$ to get to the target $\hat{D}$ and the corresponding parameters $a, b$. This, of course, is subject to the usefulness of \eqref{Dapprox} for such a goal. The remark after Theorem \ref{thmAPPROX1} applies: since $u\in [1-b, 1+a]$ and $f(1)=0$, since curve $\{(\tilde{a}, \tilde{b})\}$ obtained from Lemma \ref{lem-cont-ab} is continuous, strictly increasing and contains $(0,0)$, there is always an interval $[1-b, 1+a]$ for which such an approach is useful. What we also establish below, from a toy simulation standpoint and several $f$-divergences, is that usefulness can extend to a substantial range of acceptance probabilities (Table \ref{tab:sim-1-2}).
  
  The second way to use $\mathcal{C}(\ve{p}, \ve{q})$ consists in tackling the dual problem of $f$-MD, which is also interesting, especially if the drafter is good enough that we can constrain on the acceptance probability instead of the divergence to target. In this problem, subject to a minimal acceptance probability $P\geq \pacc(SD)$, one is required to find a mentored distribution with minimal $f$-divergence to the target. This problem admits much cheaper routines than for $f$-MD, namely $O(\log n)$ solution to compute the optimal couple $(a,b)$ (and thus $O(n)$ to compute each of the optimal $\ve{\pi}, \ve{r}, \ve{s}$). It consists in sandwiching the sought $P$ between those of two successive indexes $i$ and $i+1$ in $\textsc{c} (\ve{p}, \ve{q})$ -- say $P_i$ and $P_{i+1}$ --, and then doing a simple intrapolation on the respective parameters $(a_i, b_i)$ and $(a_{i+1}, b_{i+1})$ based on solving for $\beta$ the convex combination $P = \beta \cdot P_i + (1-\beta) \cdot P_{i+1}$. The fact that we can bypass any $f$-divergence computation because the solution is invariant to the choice of $f$ is a direct application of Theorem \ref{thmUniversal}. Notice that achieving $O(\log n)$ is without algorithmic frills, but allowing a few (lookup tables, hashtables, etc.) allows to bring it down to $O(1)$.
  
  \paragraph{Guaranteed cheap solutions with better $\pacc$ and small divergence} We state a fundamental property on the function giving the value of the divergence thershold $D$ as a function of the optimal acceptance probability $\pacc(MD)$ in $\textsc{md}^2_f(\ve{p}, \ve{q}; D)$ \eqref{def-pb-f-md-2-general}:
\begin{eqnarray}
  D_f(P) & \defeq & D : \exists (\ve{r}, \ve{s}) \in \textsc{md}^2_f(\ve{p}, \ve{q}; D) \mbox{ s.t. } \ve{p}^\top \ve{r} = P \label{eq-def-dfp}
\end{eqnarray}
(parameters $\ve{p}, \ve{q}$ are left implicit from context). For any function $g$ for which it exists, $g'_r$ denotes the right derivative.
\begin{theorem}\label{thm-opt-f}
  For any convex $f$ \eqref{eq-const-f}, the right derivative of $D_f(P)$ in $\pacc(SD)$ exists and satisfies
  \begin{eqnarray}
    (D_f)'_r(\pacc(SD)) & = & \max \partial f(1) - \min \partial f(1). \label{eq-der-right}
  \end{eqnarray}
  Furthermore, $D_f(P)$ is convex, strictly so iff $f$ is strictly convex.
  \end{theorem}
Proof in Appendix, Section \ref{sec-proof-thm-opt-f}. Suppose $f$ differentiable in $z=1$. Then \eqref{eq-der-right} crucially gives 
\begin{eqnarray}
    (D_f)'_r(\pacc(SD)) & = & 0, \label{eq-right-der-zero}
\end{eqnarray}
and so for \textit{any} such divergence, the neighborhood of the minimal acceptance probability will have divergence close to zero: depending on the $f$-divergence, it may be possible to get a substantial increase of the acceptance probability at a low cost divergence-wise.

\begin{table*}
  \centering
  \begin{tabular}{c|c|c}
    Name & Generator $f(z)$ & Comments \\ \Xhline{2pt}
    Kullback-Leibler (KL) & $z \log z$ & \\ \hline
    reverse Kullback-Leibler (rKL) & $-\log z$ \\ \hline
    Hellinger & $1 - \sqrt{z}$ & \\ \hline
    Neyman & $(z-1)^2$ & reverse $\chi^2$\\ \hline
    Pearson & $(z-1)^2/z$ & $\chi^2$\\ \hline
    $TV_2$ & $2 \cdot \max\left\{ f_{\mathrm{TV}}(z), 2 \cdot|z-1| -1\right\}$ & \\ \hline
    Amari($\alpha$) & $(z^\alpha - \alpha z + \alpha - 1)/(\alpha(\alpha-1))$ & $\alpha \in \mathbb{R}$ \\ \Xhline{2pt}
    \end{tabular}
    \caption{$f$-divergences used in our toy simulations; we recall $f_{\mathrm{TV}} (z) \defeq |z-1|/2$ (see text).}
    \label{tab:generators-f} 
\end{table*}

\paragraph{Toy simulation} We made a simple simulation of uniform $\ve{p}, \ve{q} \in \Delta_{n}$ for $n=100$. Table \ref{tab:sim-1} presents results obtained on these distributions. The left plot exemplifies Lemma \ref{lem-cont-ab} showing the strict monotonicity and continuity of the set of points returned by \cabalgo~in Algorithm \ref{algo-cabalgo}. The right plot displays the corresponding $f$-divergence as a function of the acceptance probability of mentored decoding, $\pacc(MD)$. From top to bottom in the legend, the generators of the $f$-divergences are: Kullback-Leibler, reverse Kullback-Leibler, Hellinger, Neyman, Pearson, a scaling of the total variation, $TV_2$, which replaces parts of the TV divergence generator by steeper segments and half lines, and finally two instances of Amari $\alpha$-divergences for $\alpha\in \{-1.5, 1.5\}$ \citep{anMO} (Table \ref{tab:generators-f} presents the associated generators). This plots clearly exemplifies the importance of Theorem \ref{thm-opt-f}, as for \textit{all} generators differentiable in $z=1$, $\pacc(SD)$ can be increased by more than $10\%$ at negligible divergence cost. For some divergences, such as the $\chi^2$, the divergence blows up at some point. This, of course, ultimately depends on $\ve{p}, \ve{q}$ and $f$. Table \ref{tab:sim-1-2} takes all the $f$-divergence curves and add the interval of possible $D$ values of \eqref{Dapprox} in Theorem \ref{thmAPPROX1}, in between the min and max values of $D$ as $u$ ranges in $[1-a, 1+b]$, for a range of acceptance probability for MD that ranges in between $\pacc(SD)$ and $\pacc(SD)$ plus 25$\%$. Remark that the bound is quite crude for some divergences (KL, rKL, Hellinger) but can be quite informative on the true $D$ for the others (Neyman, Pearson, $TV_2$, Amari's $\alpha$-divergence) even for a substantial increase of the acceptance probability past SD's. This simple experiment demonstrates the potential usefulness of Theorem \ref{thmAPPROX1} for approximate solutions to \eqref{def-pb-f-md-1-general} as described above.

Table \ref{tab:sim-2} further digs into the guarantees of Theorem \ref{thmAPPROX1}, showing on this example how the $f$-divergence bound $D$ and optimal acceptance probability of mentored decoding in $\textsc{md}^2_f(\ve{p}, \ve{q}; D)$ \eqref{def-pb-f-md-2-general} vary as a function of parameters $(\tilde{a}, \tilde{b})$ in $\mathcal{C}(\ve{p}, \ve{q})$ extracted by \cabalgo~in Algorithm \ref{algo-cabalgo}. Finally, Table \ref{tab:sim-3} computes the key coefficient $k$ \eqref{eq-def-ki} that governs our boosting bound in Theorem \ref{thm-boost-plus-md-general-f}. It shows, in this simulated case, that one can easily increase the acceptance probability by more than $10\%$ and still keep $k$ very close to 1, which is good news for the boosting bound \eqref{eq-bound-boosted-md}. Interestingly also, the dependence of $k$ in $a$ is close to being linear.

\begin{table*}
  \centering
  \begin{tabular}{cc}\Xhline{2pt}
    \includegraphics[trim=0bp 0bp 0bp 0bp,clip,width=0.45\columnwidth]{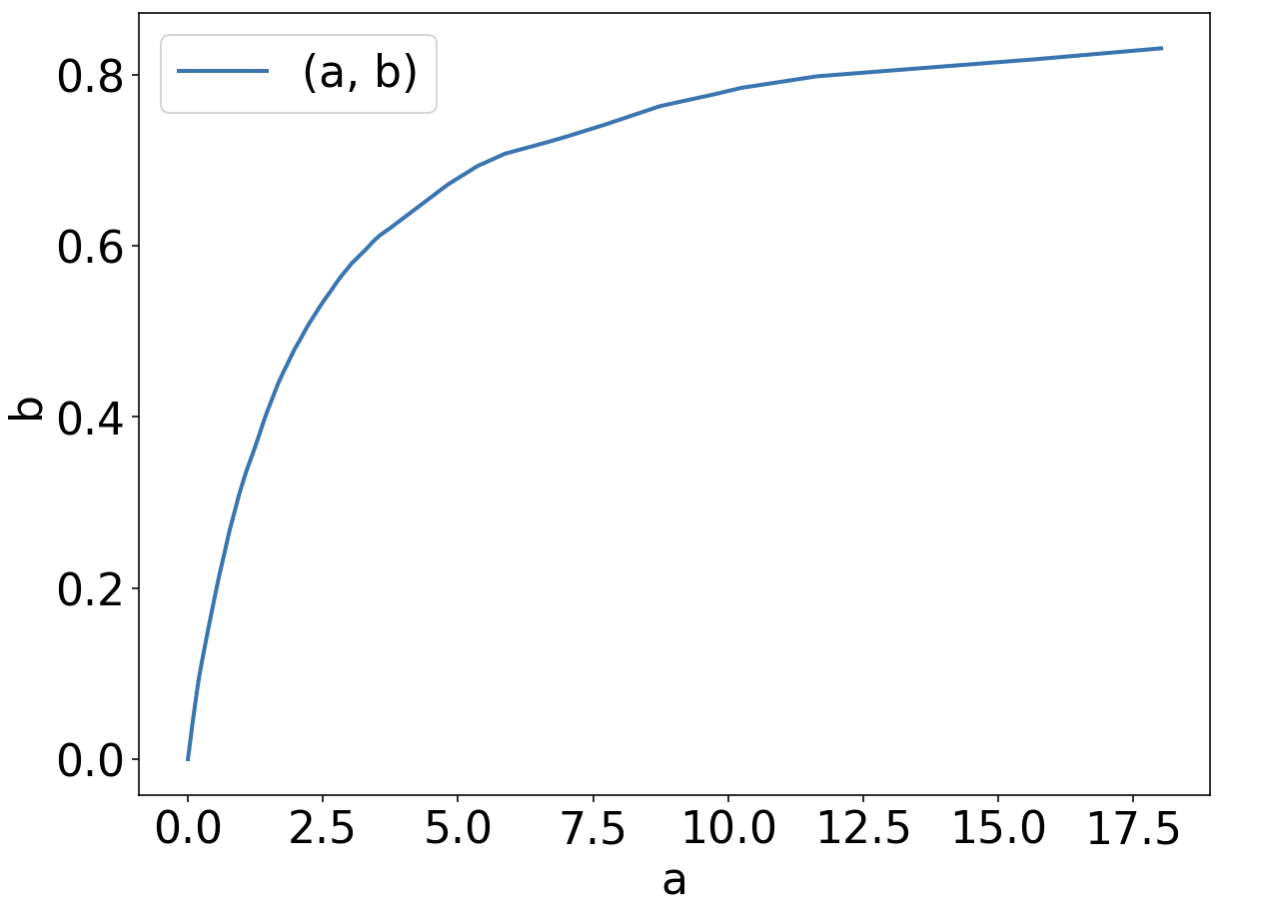}  &  \includegraphics[trim=0bp 0bp 0bp 0bp,clip,width=0.45\columnwidth]{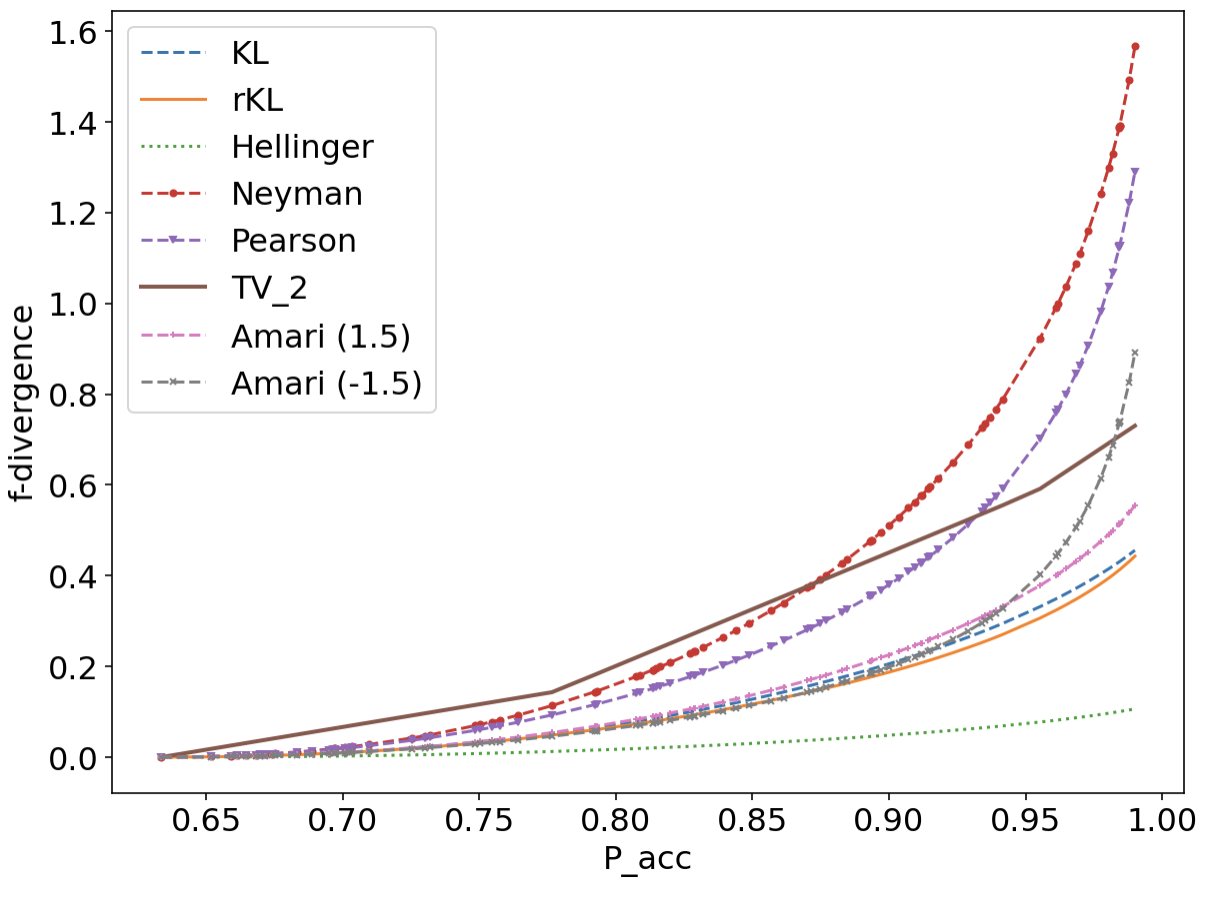} \\ \Xhline{2pt}
    \end{tabular}
    \caption{\textit{Left:} curve of points $(\tilde{a}, \tilde{b})$ in $\mathcal{C}(\ve{p}, \ve{q})$ from Lemma \ref{lem-cont-ab}, showing it is strictly increasing and continuous. \textit{Right:} curves showing the corresponding $f$-divergence bound $D$ in \eqref{def-pb-f-md-2-general} \eqref{def-pb-f-md-1-general} as a function of the optimal acceptance probability $\pacc(MD)$, for several $f$-divergences. Remark that for all except $TV_2$, the right derivative at $\pacc(SD)$ is indeed zero and the divergence stays close to 0 even past $>10\%$ increase acceptance probability $\pacc(MD)$ (see text).}
    \label{tab:sim-1}
  \end{table*}

 \setlength{\tabcolsep}{0pt}
\begin{table*}
  \centering
  \begin{tabular}{cccc}\Xhline{2pt}
    \includegraphics[trim=0bp 0bp 0bp 0bp,clip,width=0.25\columnwidth]{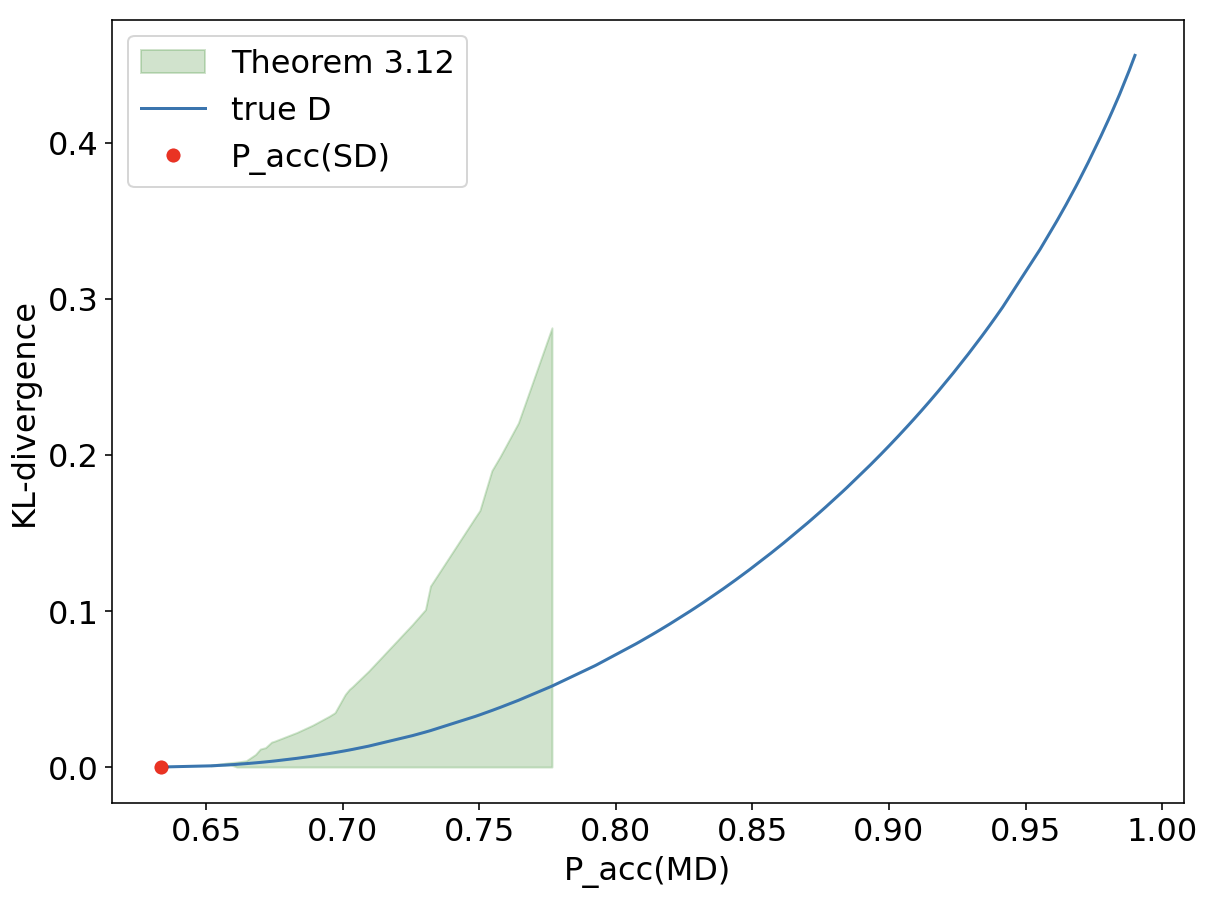}  & \includegraphics[trim=0bp 0bp 0bp 0bp,clip,width=0.25\columnwidth]{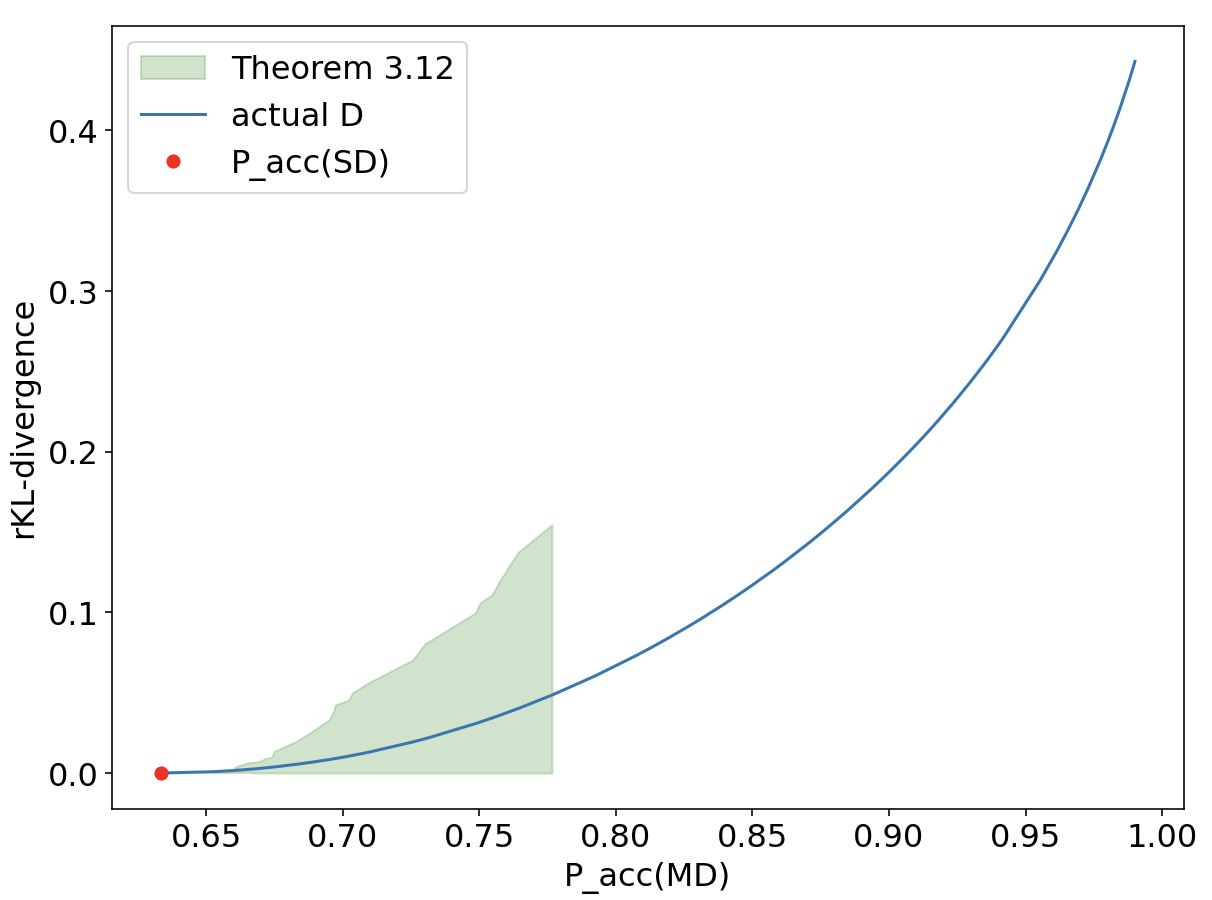}  & \includegraphics[trim=0bp 0bp 0bp 0bp,clip,width=0.25\columnwidth]{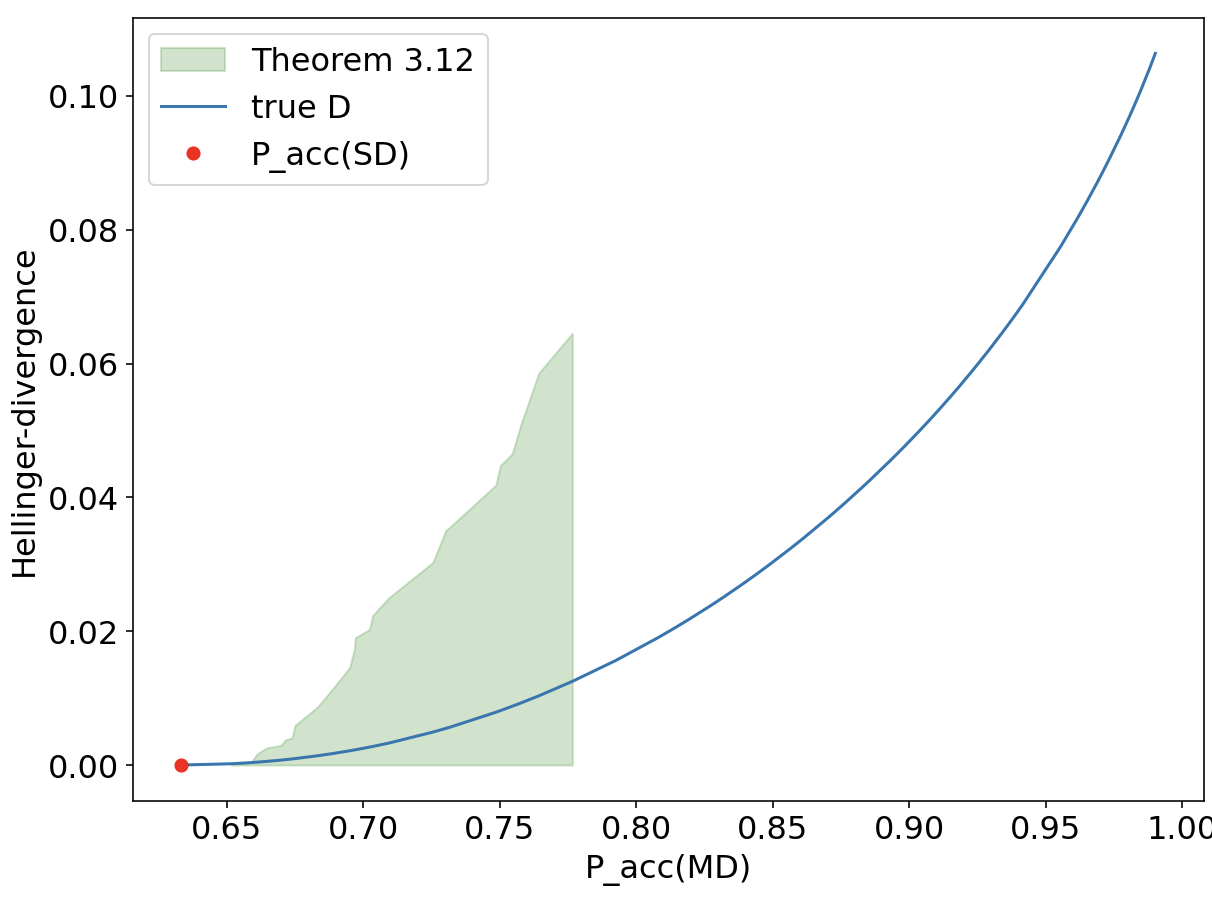}  & \includegraphics[trim=0bp 0bp 0bp 0bp,clip,width=0.25\columnwidth]{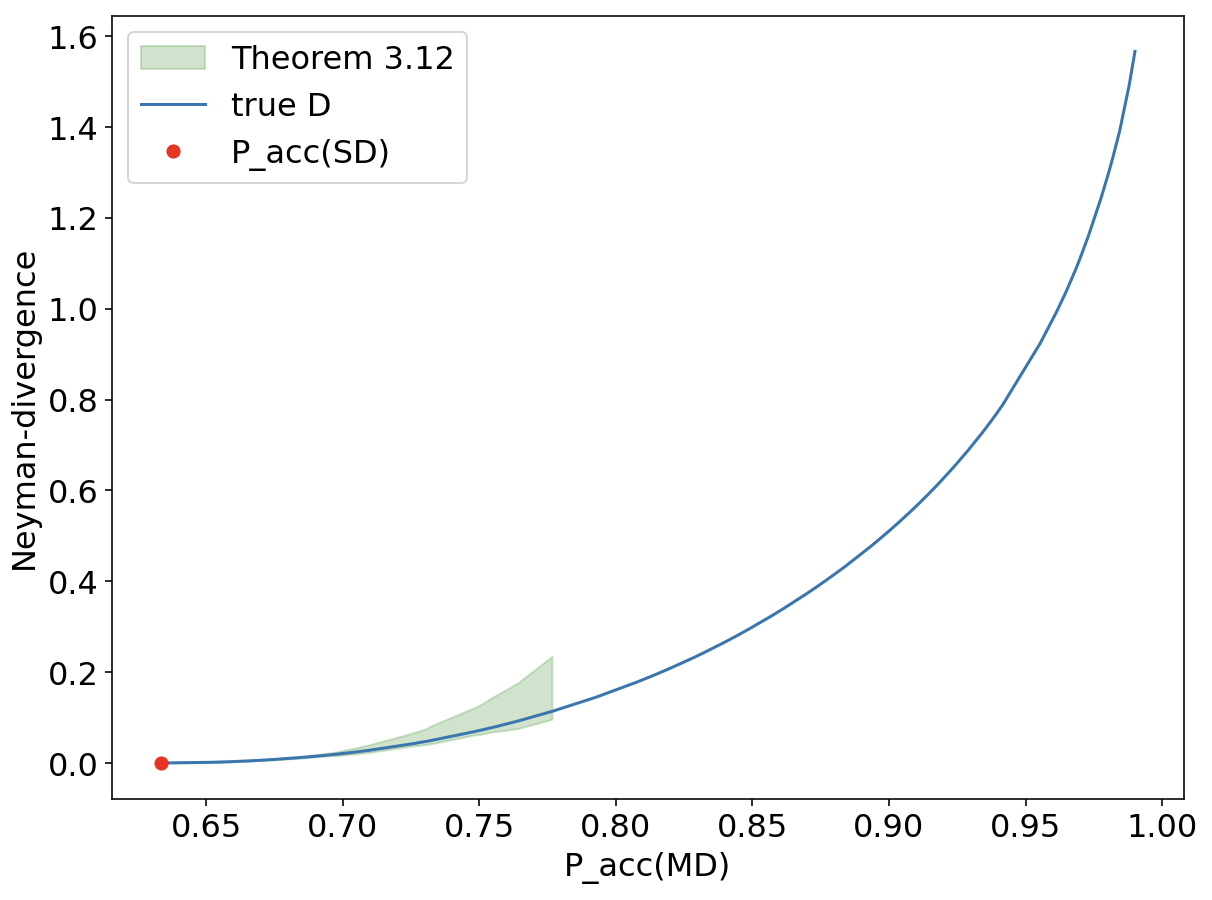}  \\ 
    Kullback-Leibler (KL) & reverse KL & Hellinger & Neyman \\ \Xhline{2pt}
    \includegraphics[trim=0bp 0bp 0bp 0bp,clip,width=0.25\columnwidth]{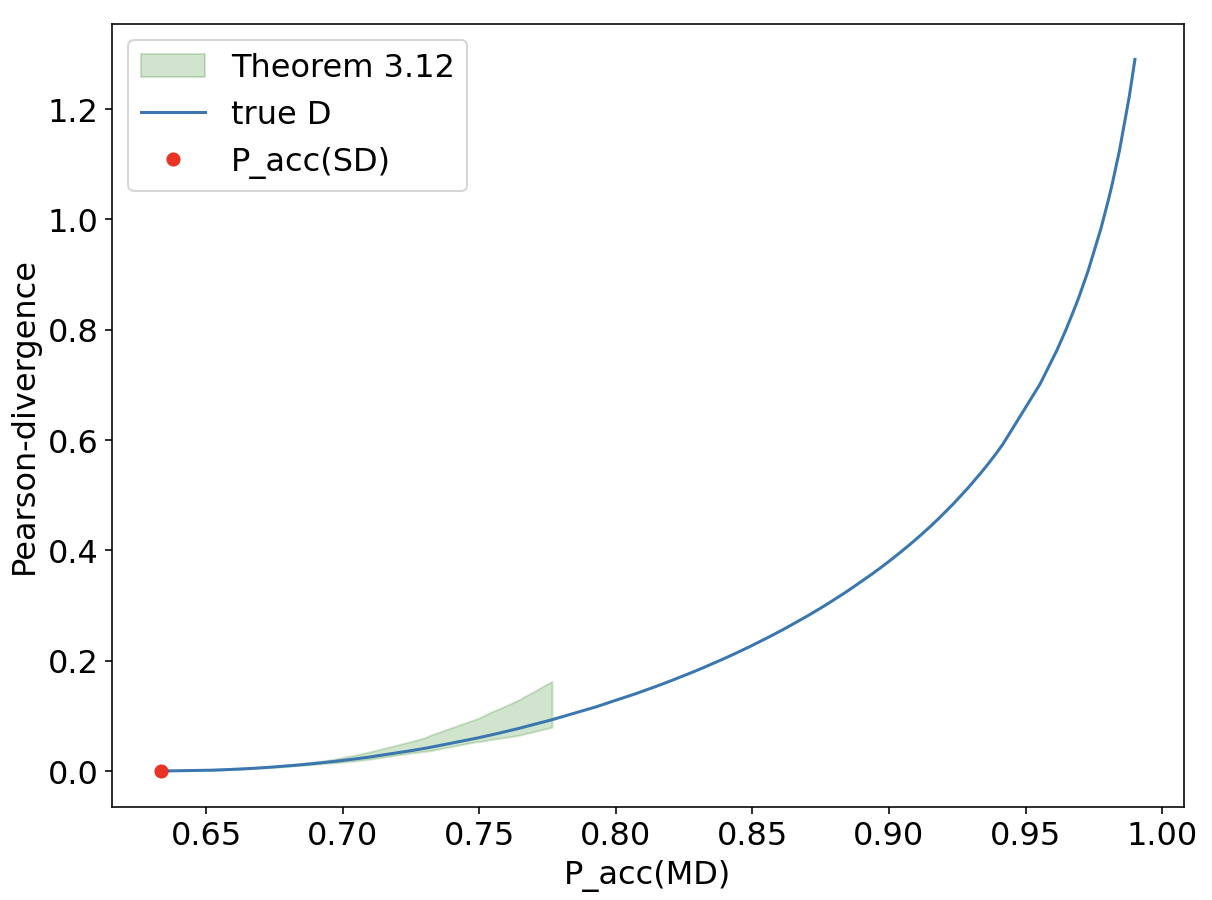}  & \includegraphics[trim=0bp 0bp 0bp 0bp,clip,width=0.25\columnwidth]{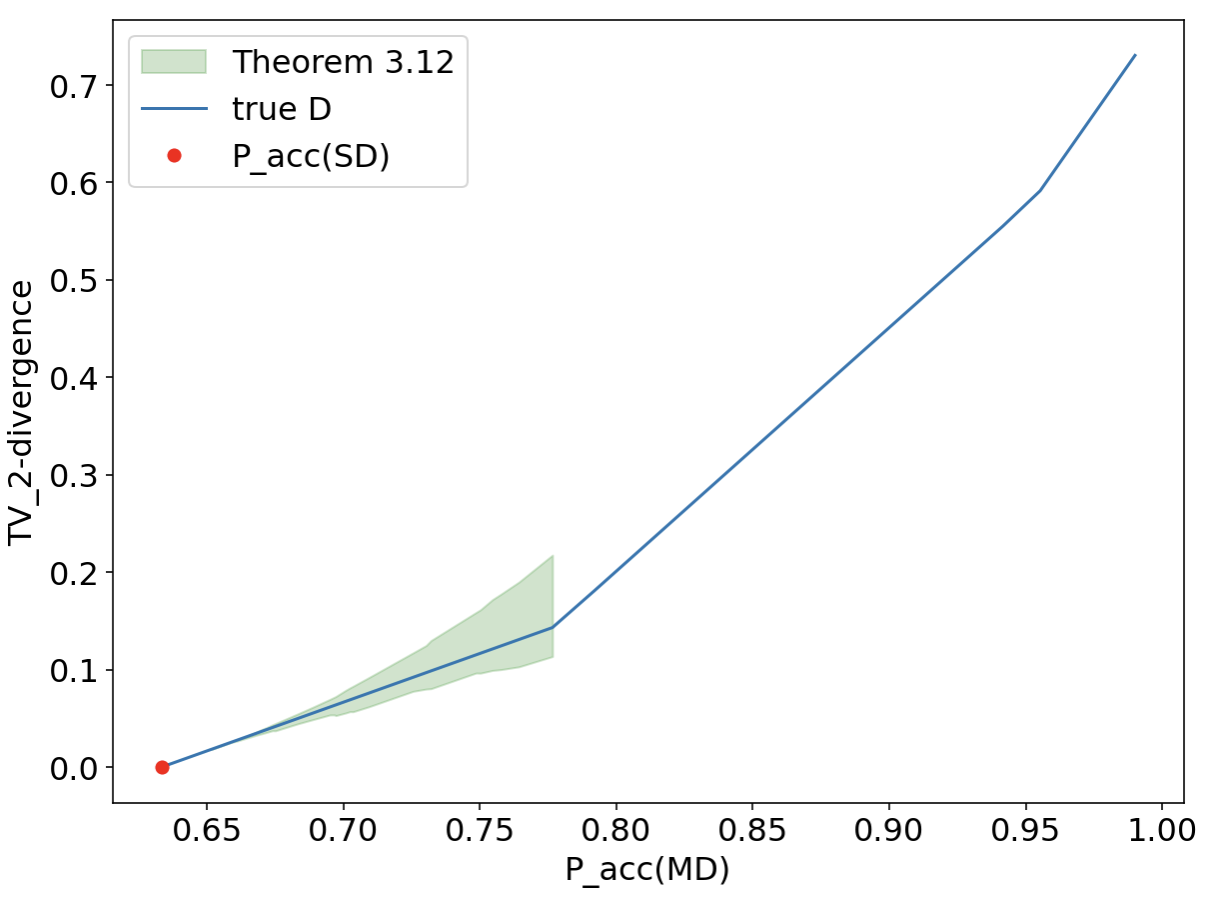}  & \includegraphics[trim=0bp 0bp 0bp 0bp,clip,width=0.25\columnwidth]{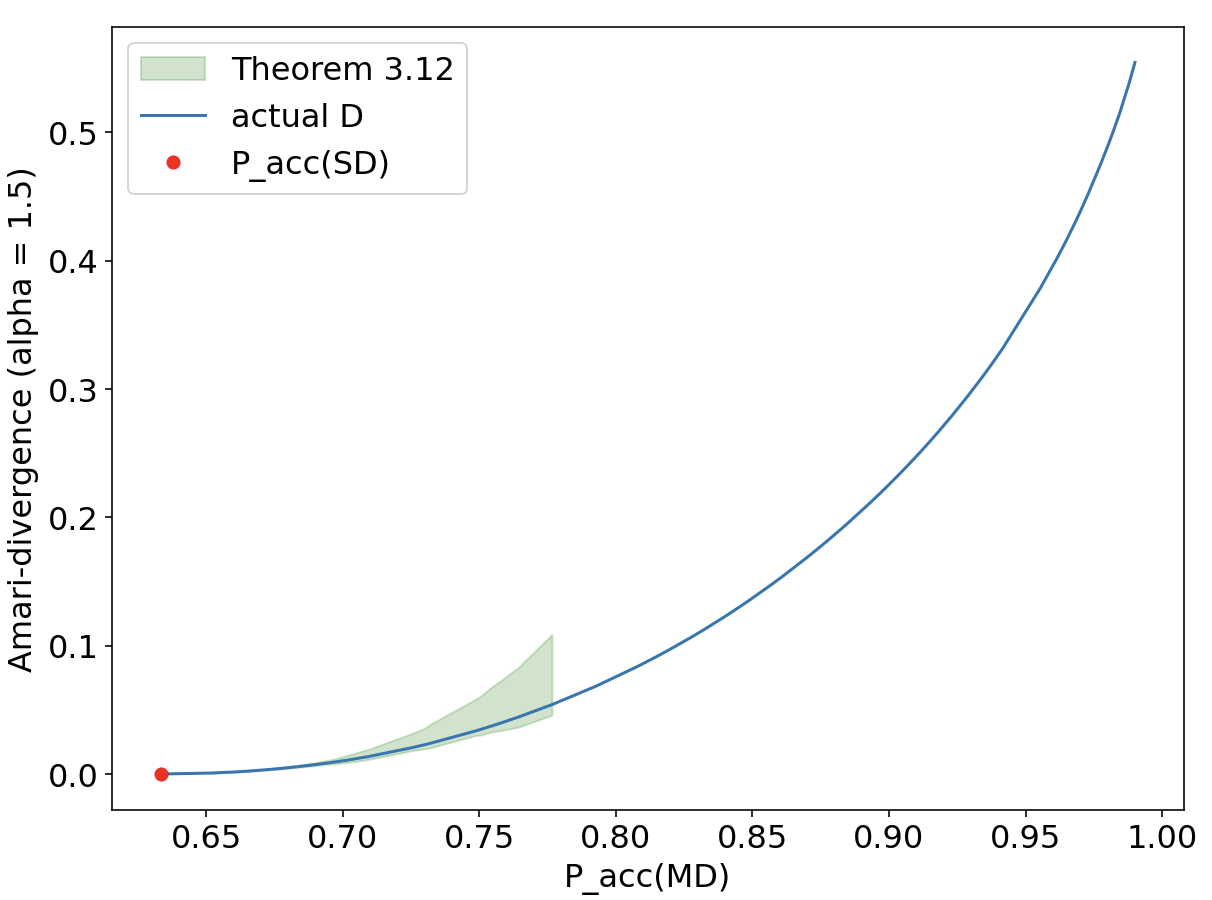}  & \includegraphics[trim=0bp 0bp 0bp 0bp,clip,width=0.25\columnwidth]{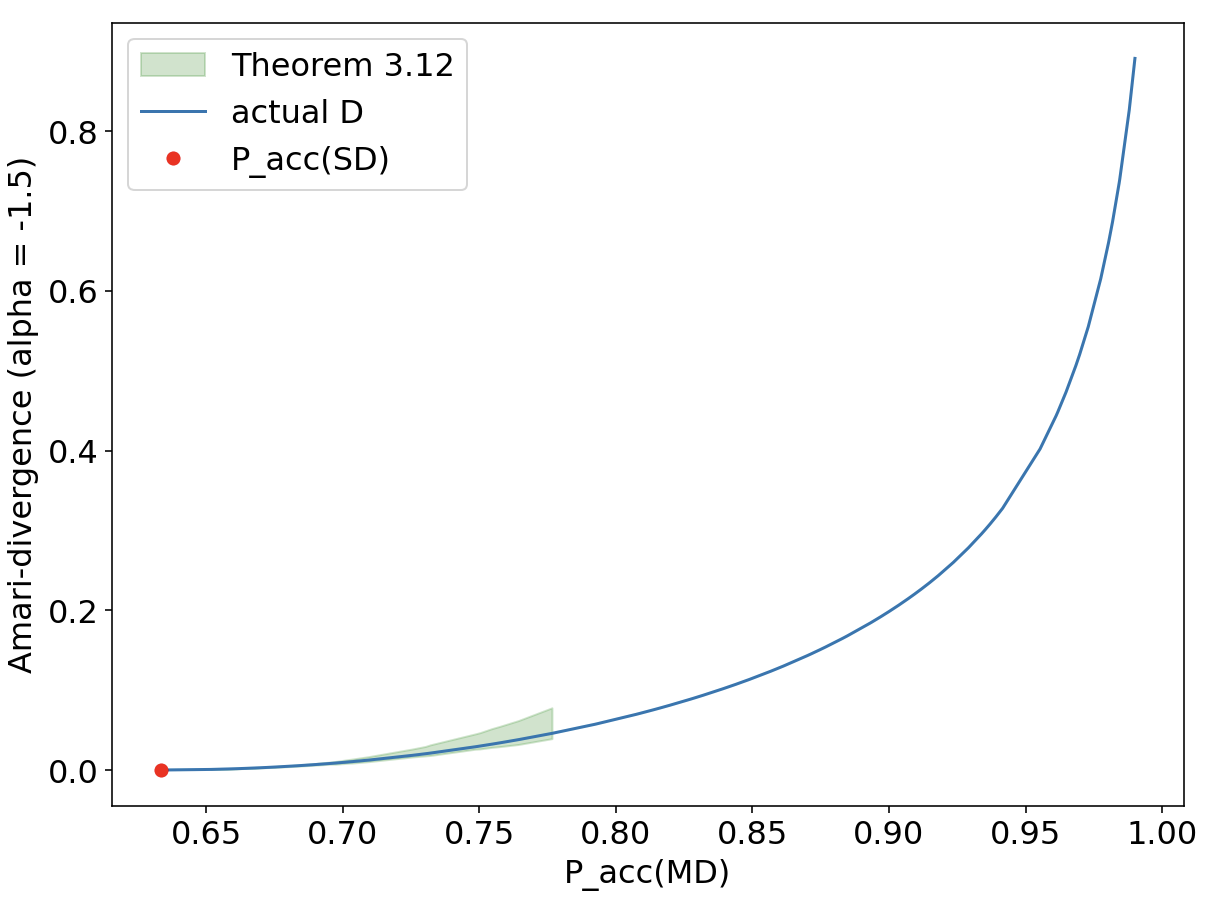}  \\ 
    Pearson & $TV_2$ & Amari ($1.5$) & Amari ($-1.5$) \\ \Xhline{2pt}
    \end{tabular}
    \caption{For each of the $f$-divergence plots in Table \ref{tab:sim-1} (Right), we plot the divergence as a function of $\pacc(MD)$, and in filled {\color{darkgreen} green} the interval of min and max values corresponding to \eqref{Dapprox} in Theorem \ref{thmAPPROX1}, over a range of $\pacc(MD)$ which covers $\pacc(SD)$ ({\color{red} red} dot) to $\pacc(SD)$ increased by 25$\%$ (see text).}
    \label{tab:sim-1-2} 
\end{table*}

 \setlength{\tabcolsep}{5pt}
\begin{table*}
  \centering
  \begin{tabular}{cc}\Xhline{2pt}
    \includegraphics[trim=0bp 0bp 0bp 0bp,clip,width=0.4\columnwidth]{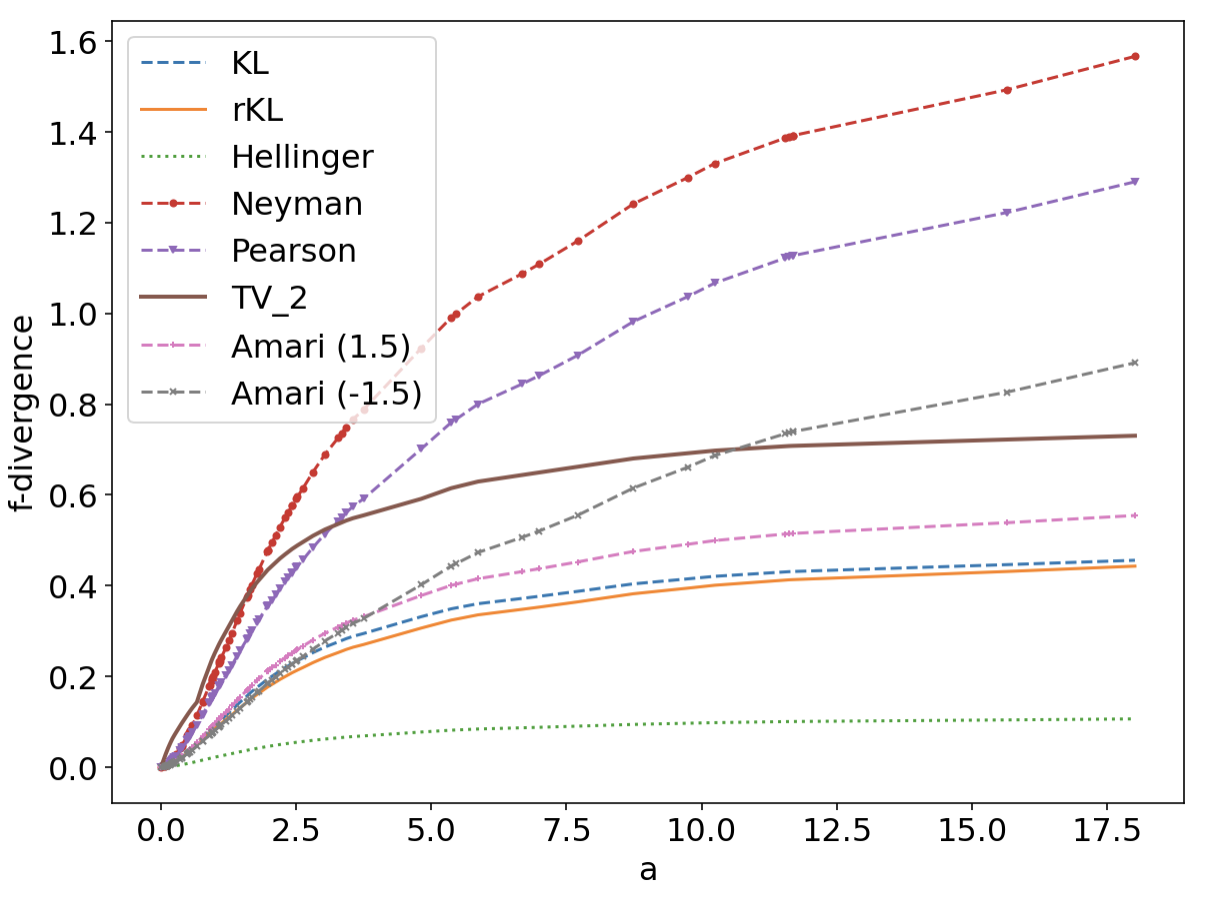}  &  \includegraphics[trim=0bp 0bp 0bp 0bp,clip,width=0.4\columnwidth]{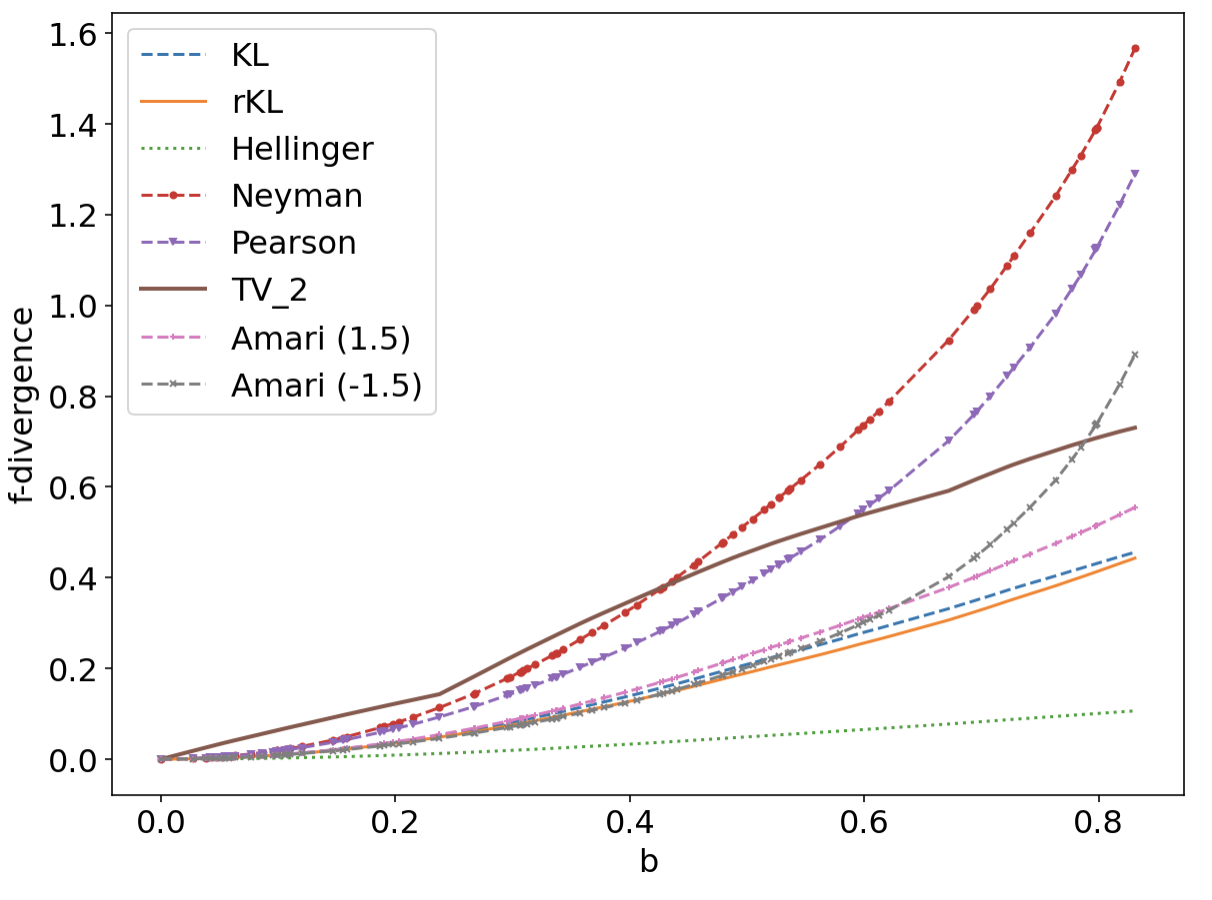} \\ \Xhline{2pt}
    \includegraphics[trim=0bp 0bp 0bp 0bp,clip,width=0.4\columnwidth]{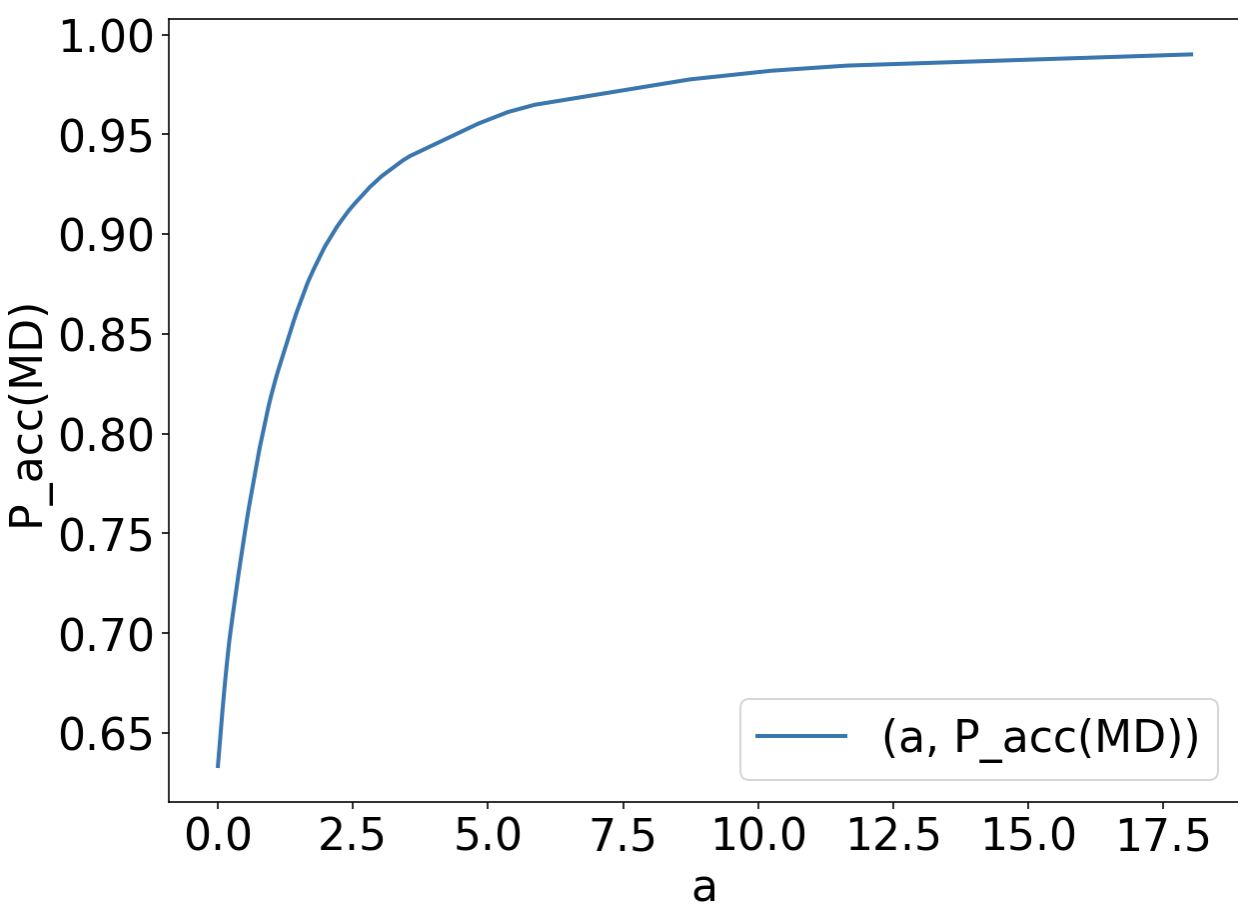}  &  \includegraphics[trim=0bp 0bp 0bp 0bp,clip,width=0.4\columnwidth]{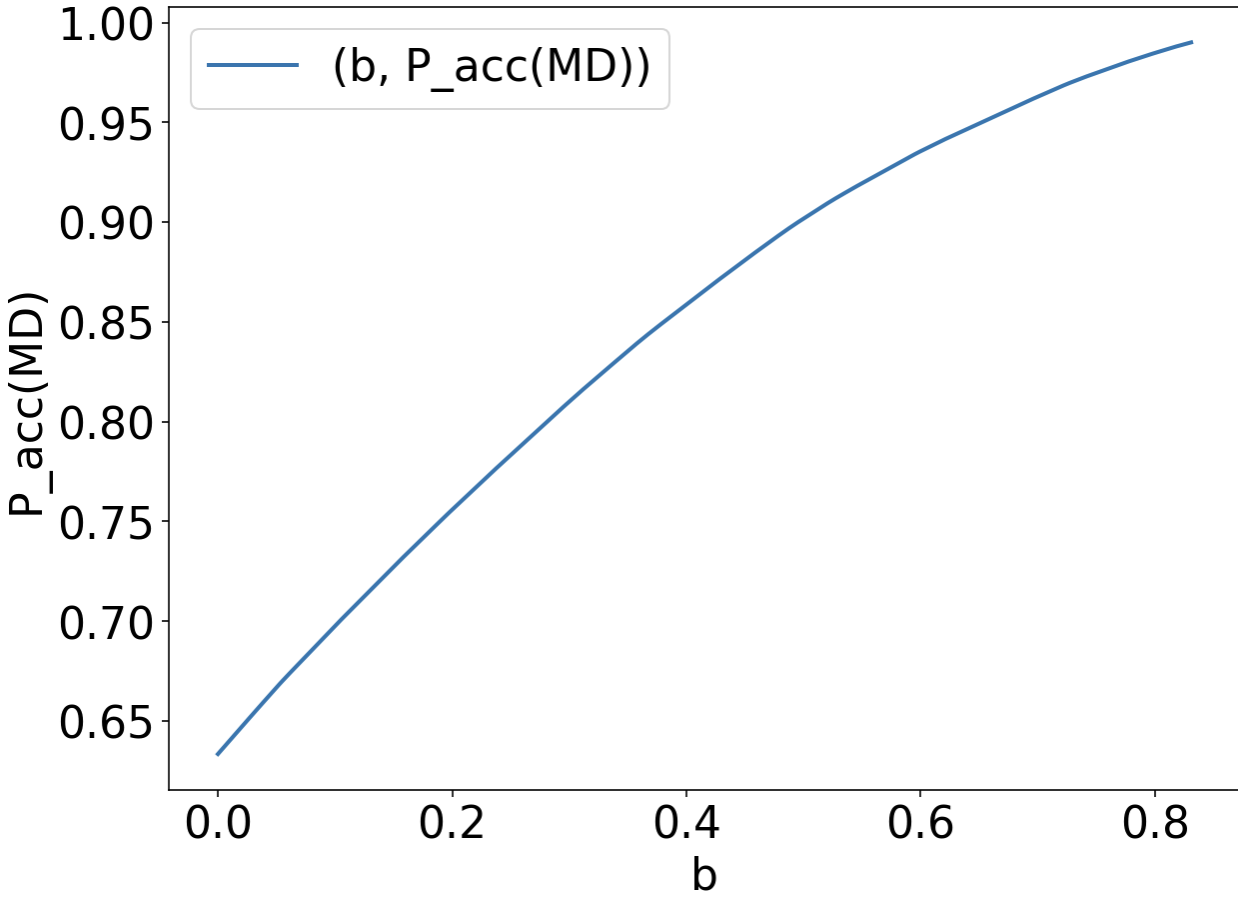} \\ \Xhline{2pt}
    \end{tabular}
    \caption{\textit{Top:} plots of the $f$-divergence bound $D$ in $\textsc{md}_f(\ve{p}, \ve{q}; D)$ as a function of $a$ (\textit{Left}) and $b$ (\textit{Right}), exemplifying Theorem \ref{thmAPPROX1}. \textit{Bottom:} plots of the optimal acceptance probability $\pacc(MD)$ as a function of $a$ (\textit{Left}) and $b$ (\textit{Right}), exemplifying Theorem \ref{thmAPPROX1} (see text).}
    \label{tab:sim-2}
\end{table*}

\begin{table*}
  \centering
  \begin{tabular}{cc}\Xhline{2pt}
    \includegraphics[trim=0bp 0bp 0bp 0bp,clip,width=0.45\columnwidth]{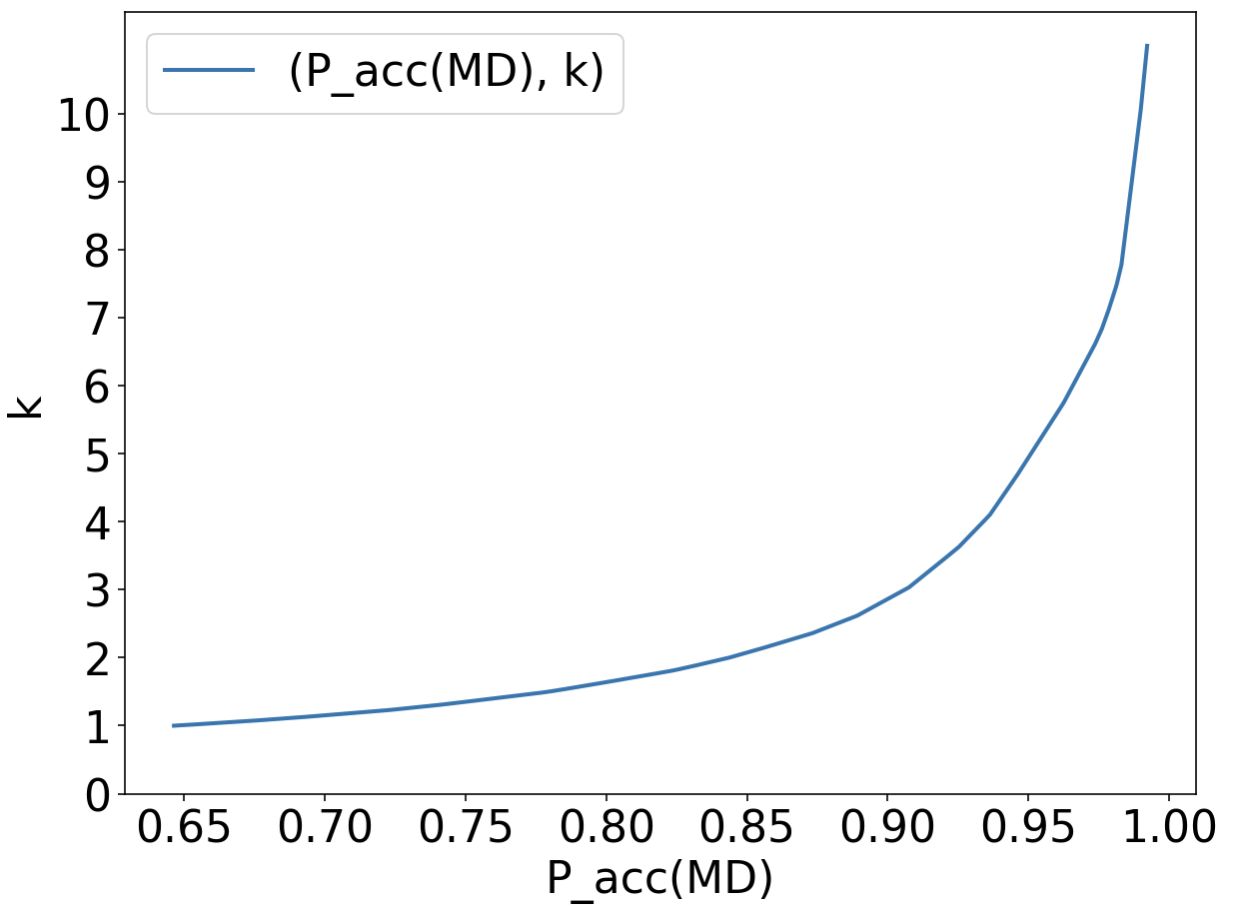}  &  \includegraphics[trim=0bp 0bp 0bp 0bp,clip,width=0.45\columnwidth]{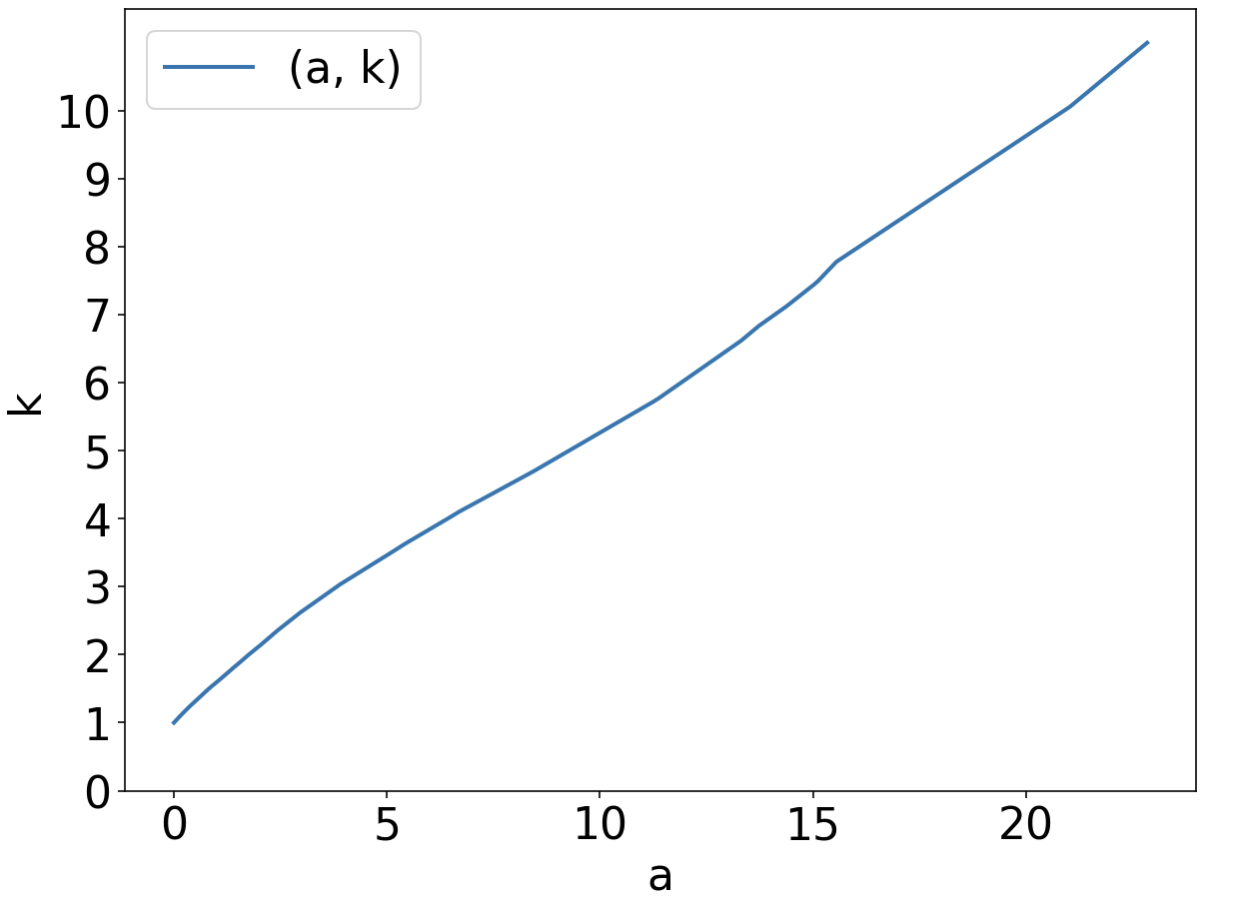} \\ \Xhline{2pt}
    \end{tabular}
    \caption{\textit{Left:} curve giving the boosting bound coefficient $k \defeq \sqrt{(1+a)/(1-b)}$ (Theorem \ref{thm-boost-plus-md-general-f}) as a function of the acceptance probability $\pacc(MD)$; remark that $k$ remains close to $1$ even for substantial increase of the acceptance probability compared to the speculative decoding solution. \textit{Right:} curve giving coefficient $k$ as a function of $a$ (see text).}
    \label{tab:sim-3}
\end{table*}

We have also performed a second simulation in which the objective was to help visualize the mentored distribution. In Table \ref{tab:sim-4}, we have computed $\ve{p}, \ve{q} \in \Delta_{n}$ for $n=1000$ both following discretized Beta distributions, and ordered in the $x$ axis the indexes in increasing $q_. / p_.$ ratios. Then, we have computed, from top-left to bottom-right, the mentored distribution (thick purple curve) by putting in evidence function $i \mapsto \min\{\pi_i, p_i\}$ whose area (purple) is the acceptance probability and the $f$-divergence is Kullback-Leibler, for values of $b = \{0, 0.1, ..., 1.0\}$ in Definition \ref{def-cab}. We can clearly identify and track the three sets of indices defining sets \eqref{defAbis}, \eqref{defIbis}, \eqref{defBbis}.

\begin{table*}
  \centering
  \begin{tabular}{cccc}\Xhline{2pt}
    \includegraphics[trim=0bp 20bp 0bp 0bp,clip,width=0.22\columnwidth]{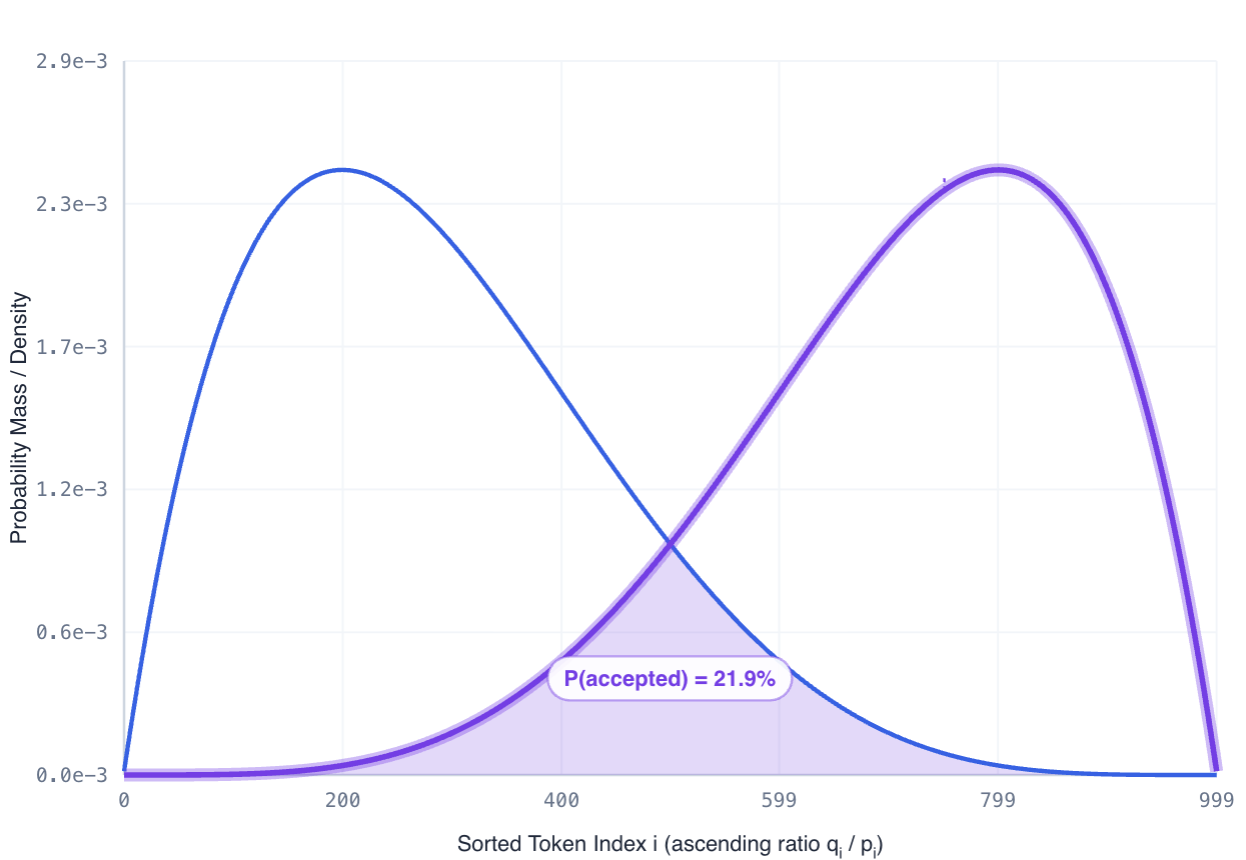}  &  \includegraphics[trim=0bp 20bp 0bp 0bp,clip,width=0.22\columnwidth]{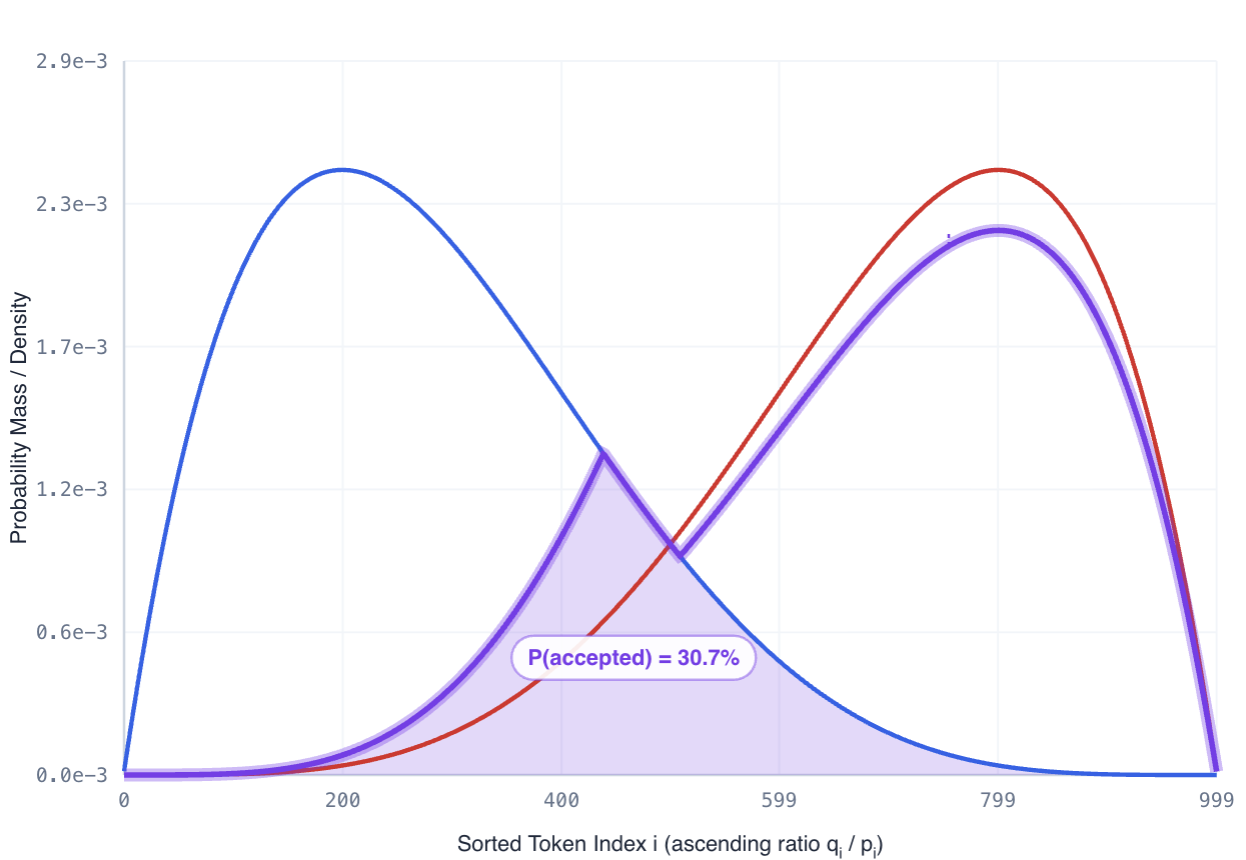}  &  \includegraphics[trim=0bp 20bp 0bp 0bp,clip,width=0.22\columnwidth]{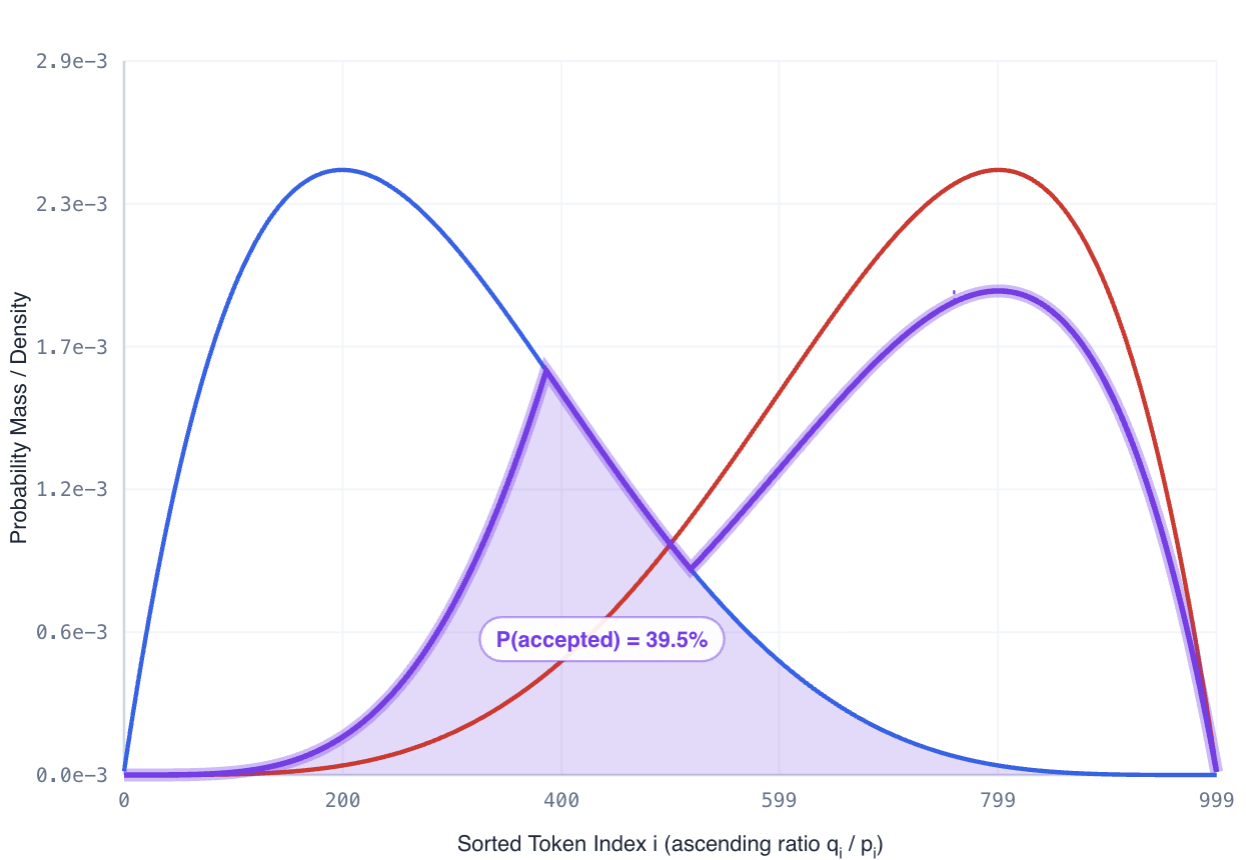}  &  \includegraphics[trim=0bp 20bp 0bp 0bp,clip,width=0.22\columnwidth]{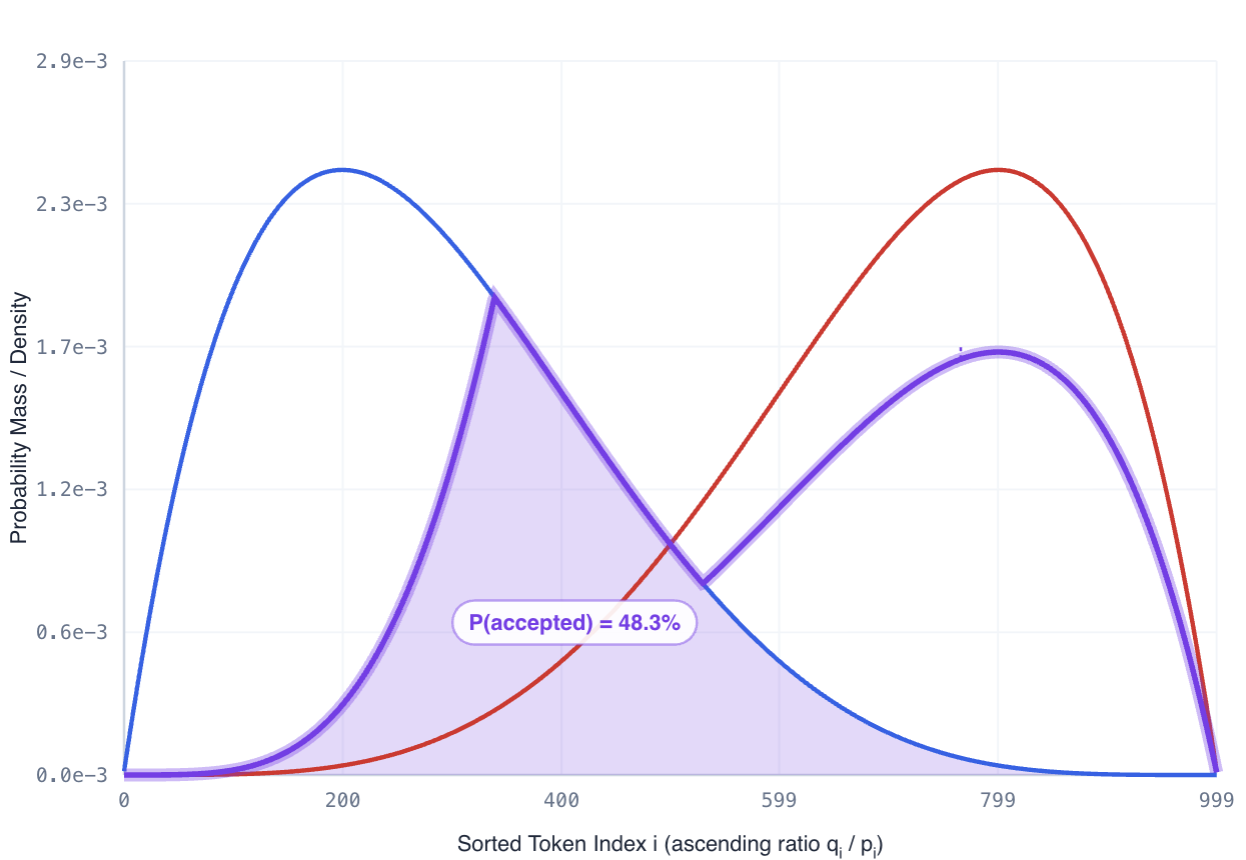}  \\ \hline
    $b=0$ ($\pacc = 0.219$) & 0.1 ($0.307$) & 0.2 ($0.395$) & 0.3 ($0.483$)\\ \Xhline{2pt}
    \includegraphics[trim=0bp 20bp 0bp 0bp,clip,width=0.22\columnwidth]{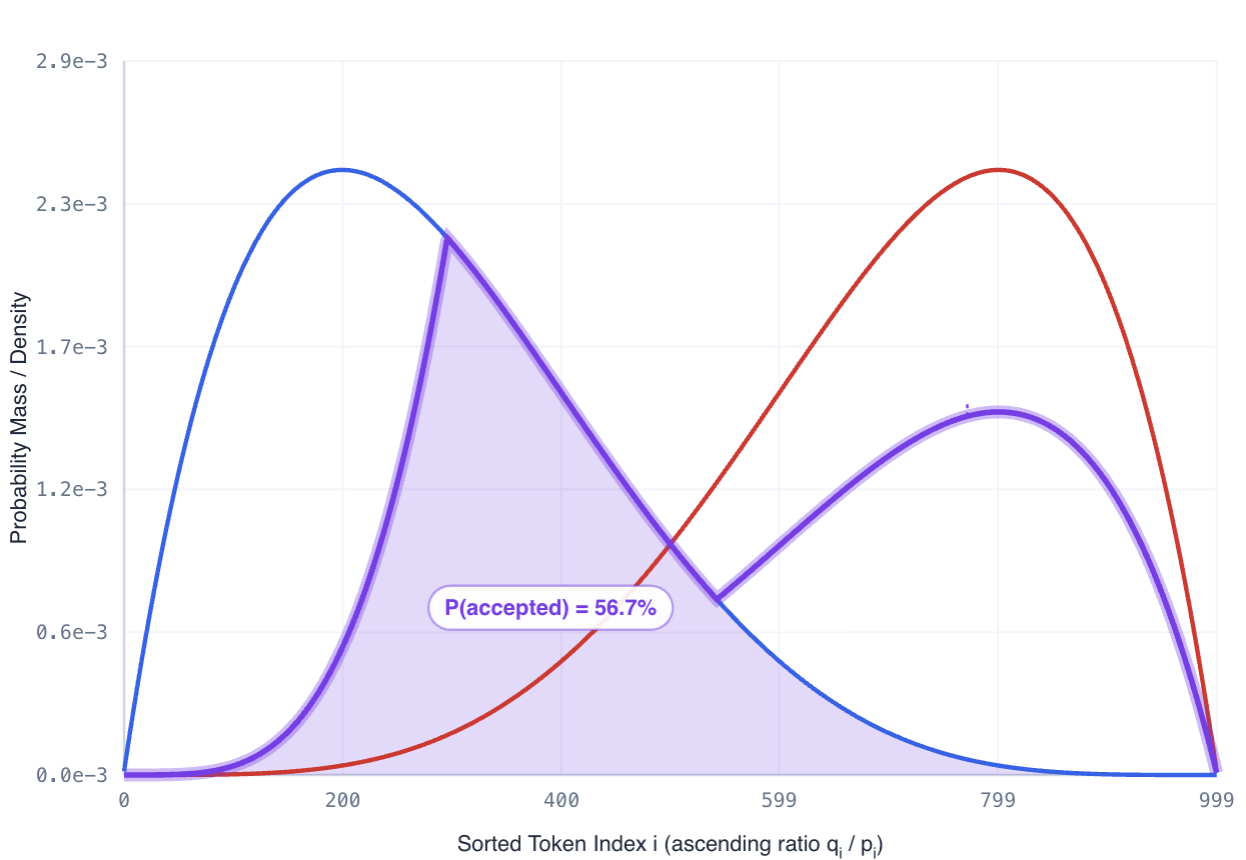}  &  \includegraphics[trim=0bp 20bp 0bp 0bp,clip,width=0.22\columnwidth]{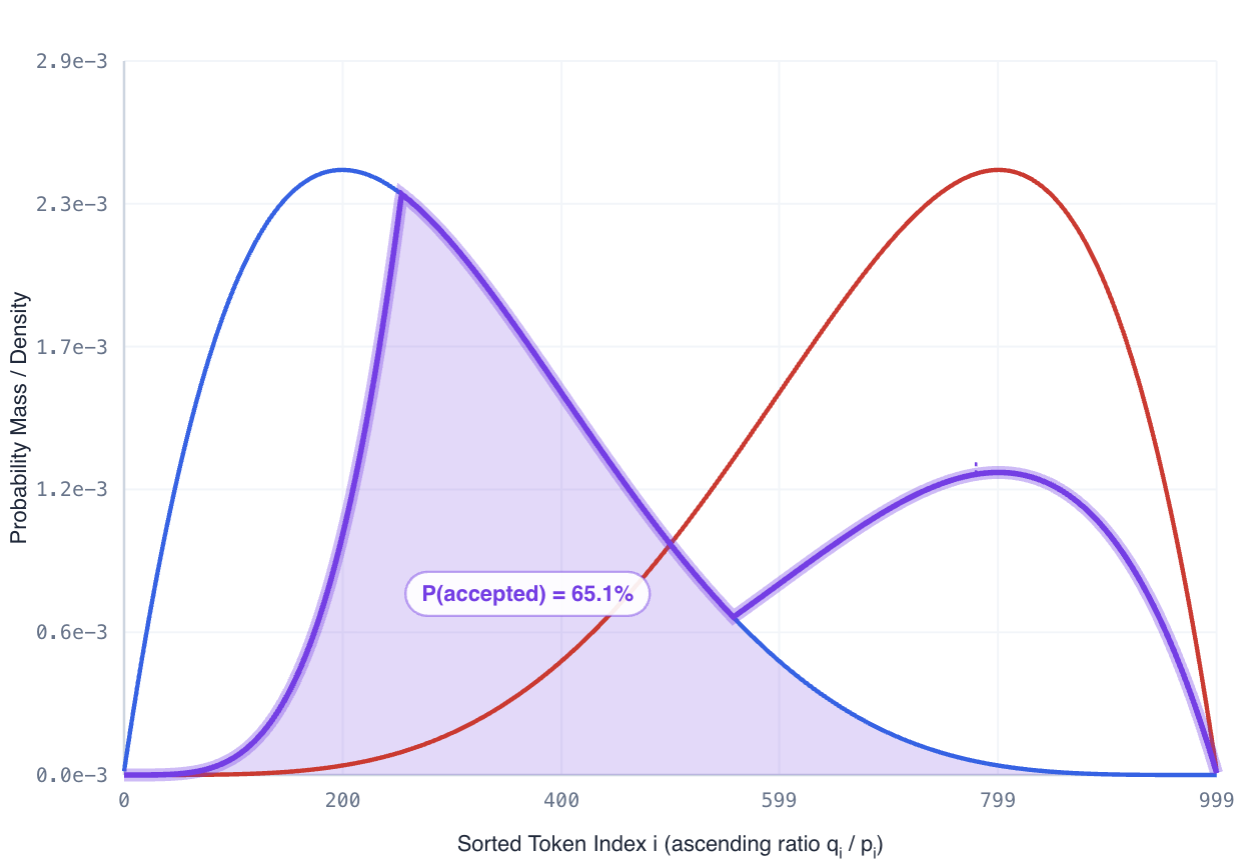}  &  \includegraphics[trim=0bp 20bp 0bp 0bp,clip,width=0.22\columnwidth]{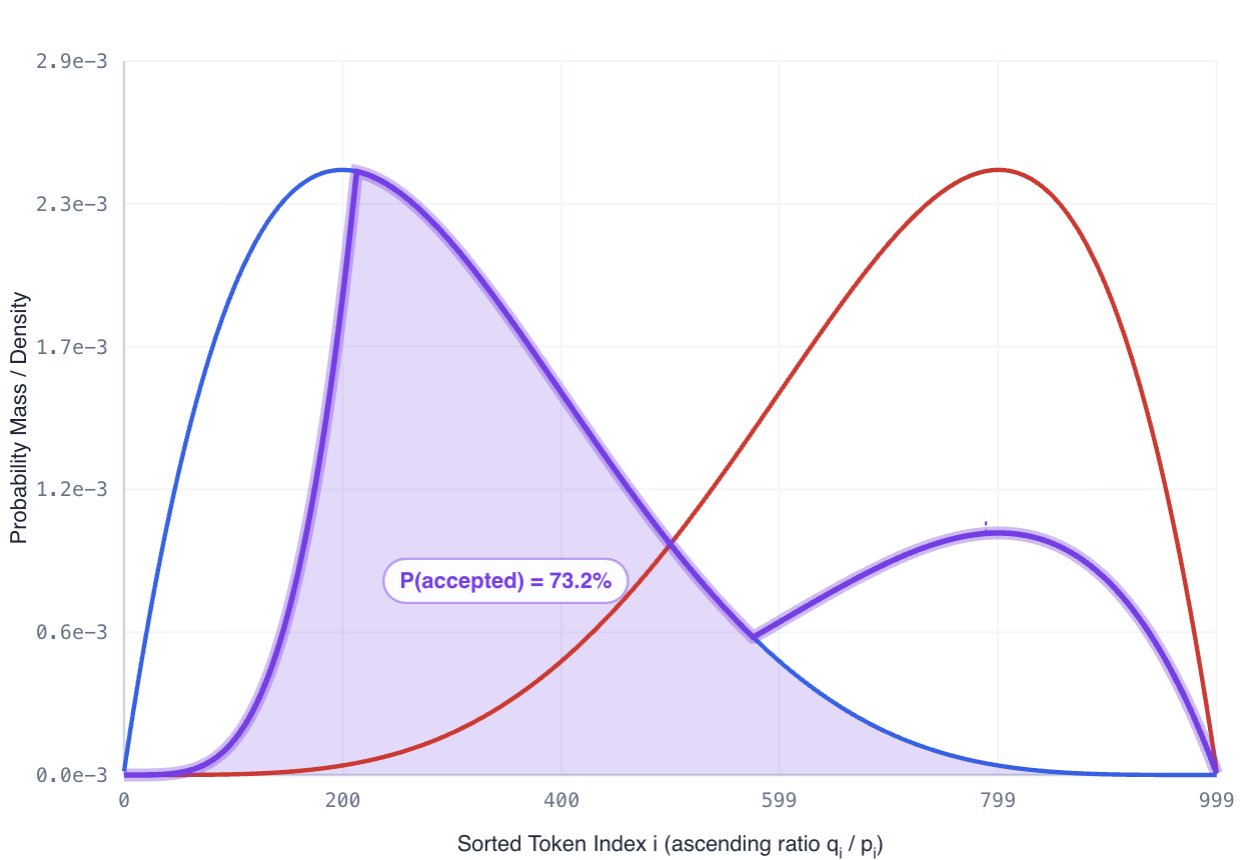}  &  \includegraphics[trim=0bp 20bp 0bp 0bp,clip,width=0.22\columnwidth]{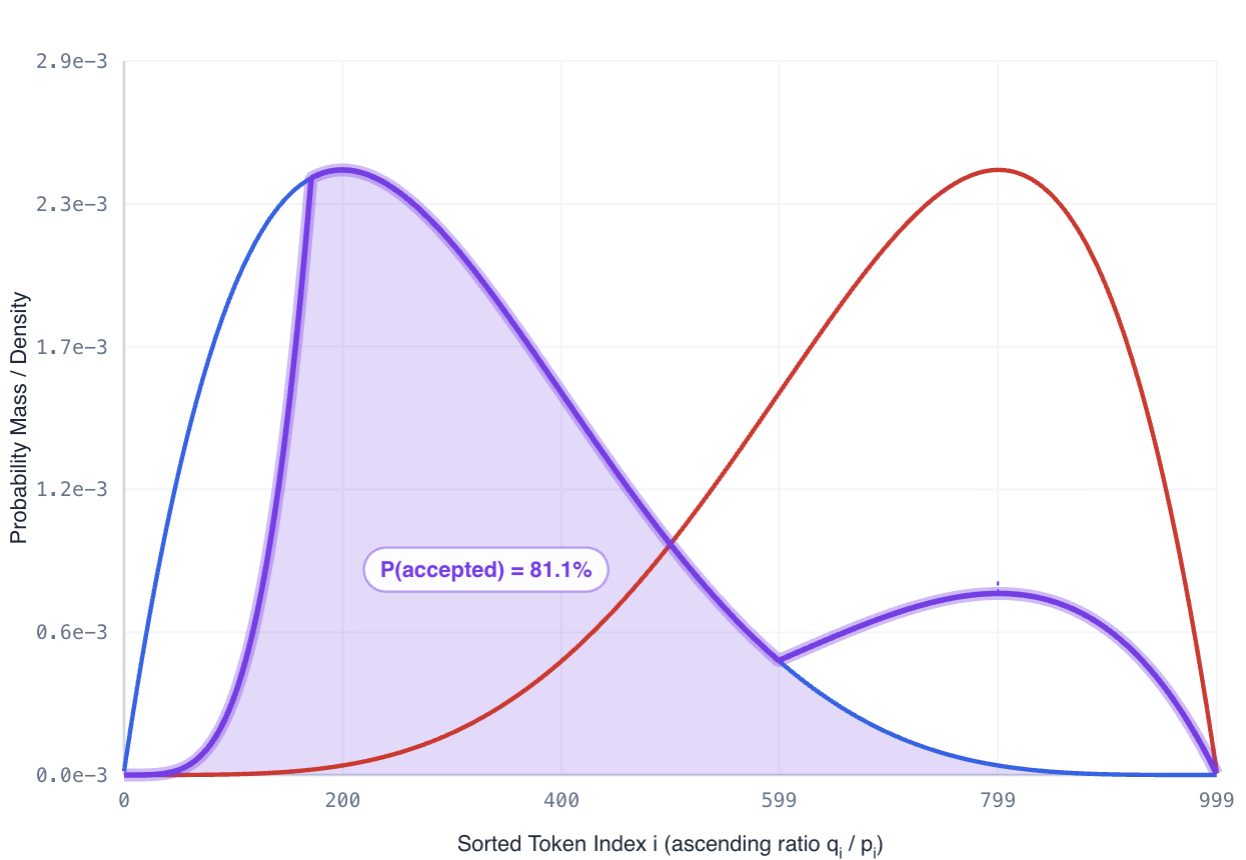}  \\ \hline
    $b=0.4$ ($\pacc = 0.567$) & 0.5 ($0.651$) & 0.6 ($0.732$) & 0.7 ($0.811$) \\ \Xhline{2pt}  
                                                                                       \includegraphics[trim=0bp 20bp 0bp 0bp,clip,width=0.22\columnwidth]{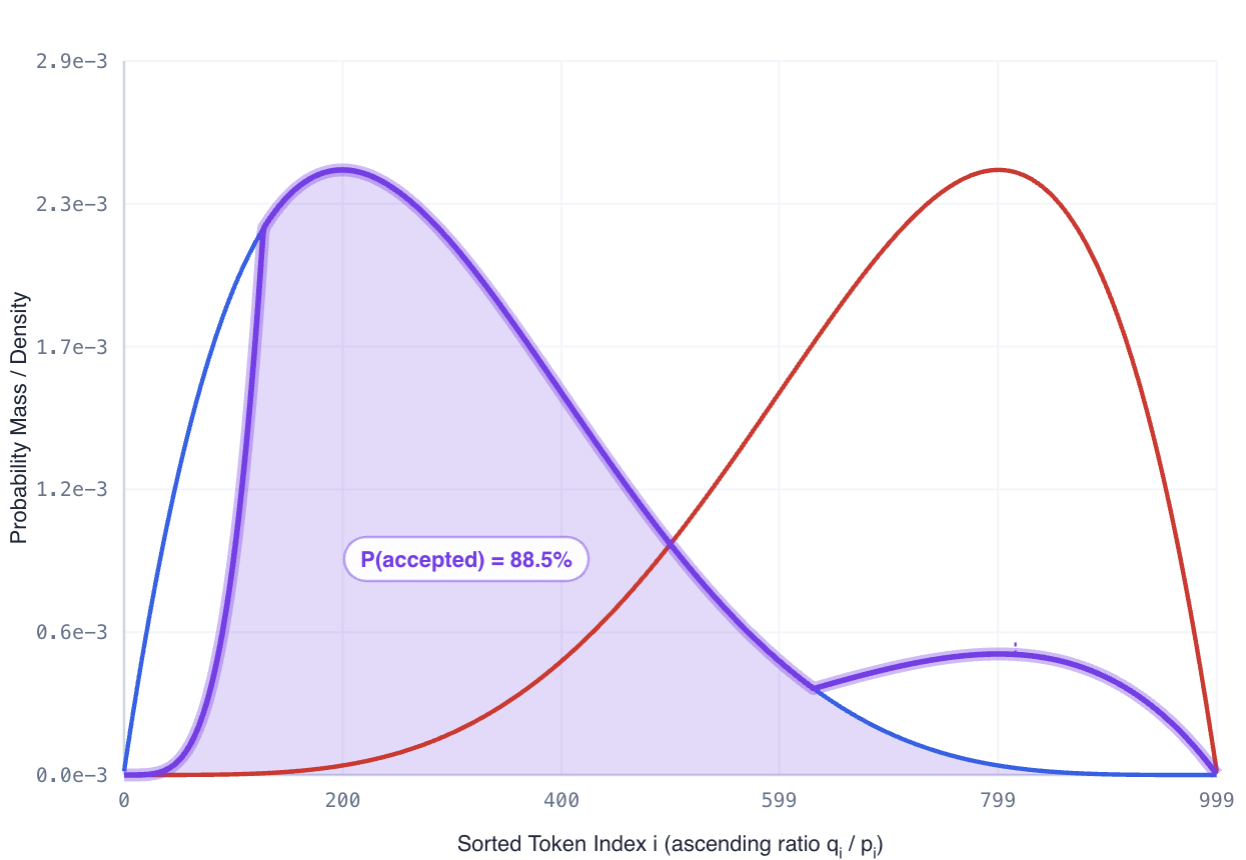}  &  \includegraphics[trim=0bp 20bp 0bp 0bp,clip,width=0.22\columnwidth]{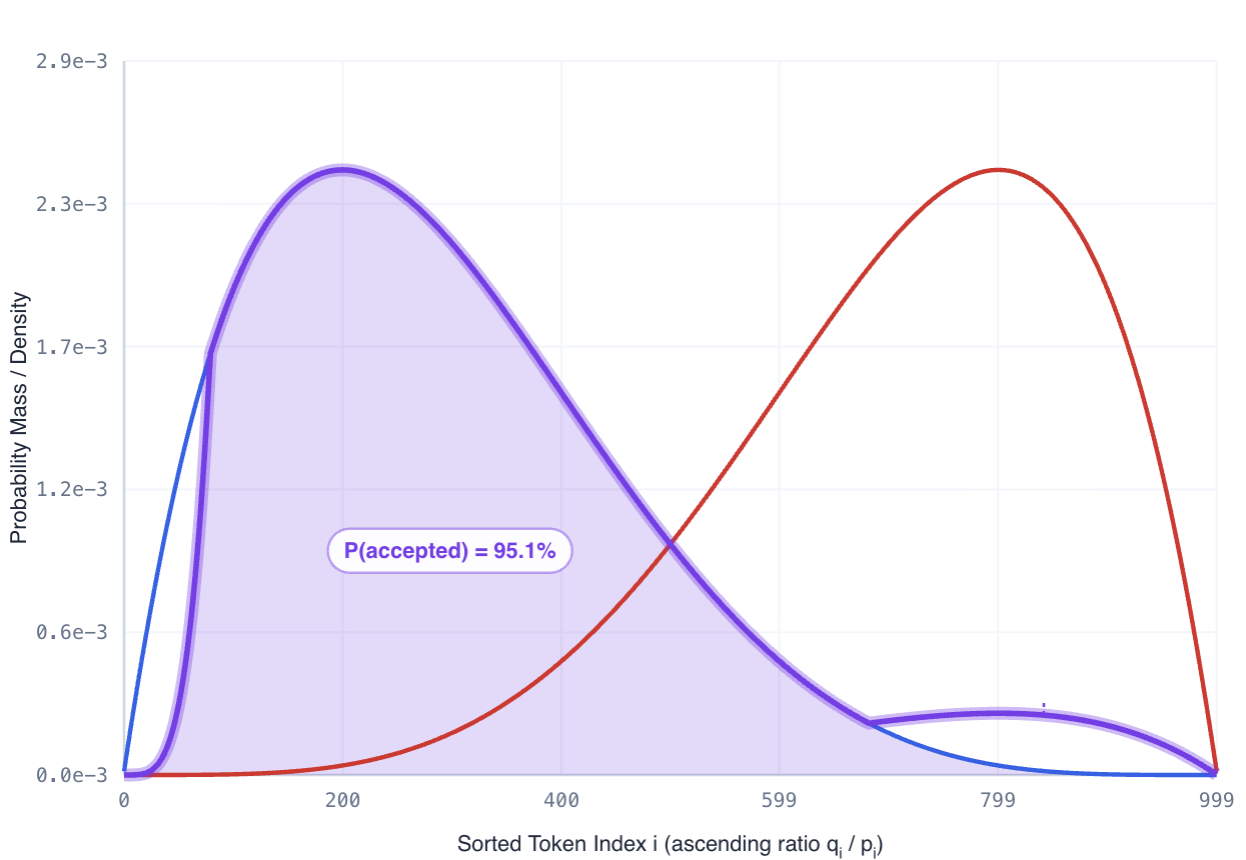}  &  \includegraphics[trim=0bp 20bp 0bp 0bp,clip,width=0.22\columnwidth]{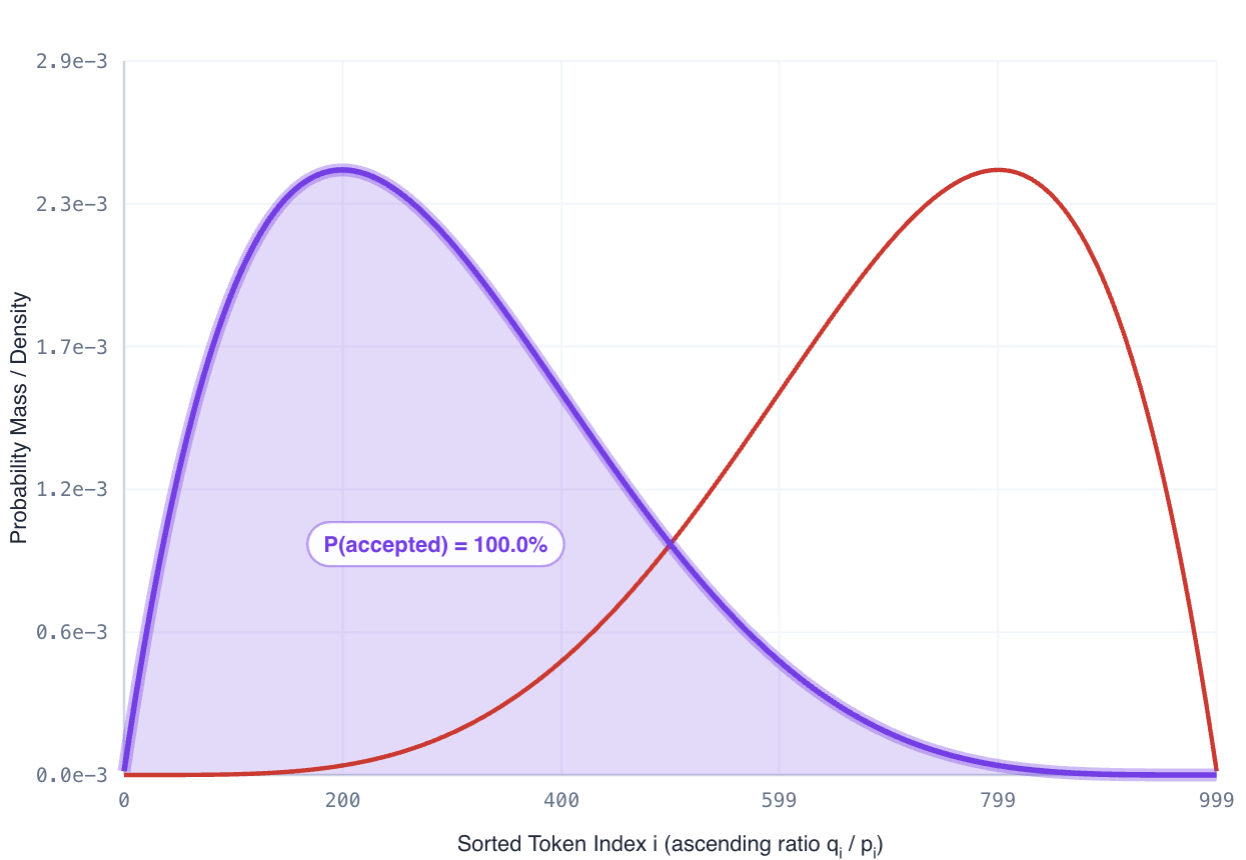} &  \\ \hline
    $b=0.8$ ($\pacc = 0.885$) & 0.9 ($0.951$) & 1 ($1$) & \\ \Xhline{2pt} 
    \end{tabular}
    \caption{Mentored distribution (thick {\color{RoyalPurple} purple} curve) and acceptance probability ({\color{Purple} purple} area) when $\ve{p}$ (thin {\color{blue} blue} curve) and $\ve{q}$ (thin {\color{red} red} curve) are discretized from Beta distributions. The values under each plot are $b$ and $\pacc(MD)$ (in parentheses, see text).}
    \label{tab:sim-4}
\end{table*}

\section{Discussion}\label{sec-disc}

In this section, we discuss the case where we want to enforce the decoding of only the top-$k$ tokens of the target, and two additional points on boosting. The relevance of the discussion on boosting extends beyond the mentored decoding case.

\paragraph{Efficient restriction to top-$k$ decoding}

Suppose we mask $n-k$ tokens on $\ve{q}$ by replacing the coordinates by $0$, e.g. to single out the top-$k$ coordinates of $\ve{q}$, with $k\geq 1$. This is a practically relevant setting that clearly breaks Assumption \ref{assum-mda}. For the sake of readability, we assume the remaining token coordinates of $\ve{q}$ are renormalized in the simplex so we do not need to overload our MD problems with new parameters: what happens for the solutions of our mentored decoding problems \eqref{def-pb-f-md-2-general}, \eqref{def-pb-f-md-1-general} ? Note that $(q_i/\pi_i) \cdot f(\pi_i / q_i) = f(u)/u$ with $u \defeq \pi_i / q_i$. If $q_i = 0$ but $\pi_i \neq 0$, the expression takes the limit $\lim_{z \rightarrow +\infty} f(z)/z$. That is why we generalize the definition of a $f$-divergence in \eqref{eq-def-f-div} to the possibility that some $q_i = 0$ by following \citet{cAC} and letting
\begin{eqnarray}
  D_f(\ve{\pi}\| \ve{q}) & \defeq & \sum_{i: q_i > 0} q_i f\left(\frac{\pi_i}{q_i}\right) + \sum_{i: q_i = 0 \wedge \pi_i > 0} \pi_i \cdot \lim_{z \rightarrow +\infty} \frac{f(z)}{z}.\label{eq-def-f-div-gen-2}
\end{eqnarray}
For clarity, we explicitly discard the problematic cases where both $\pi_i = q_i = 0$: we treat it as the limit case $\lim_{z \rightarrow 1} f(z) = f(1) = 0$ ($f$ is convex, thus continuous) per \eqref{eq-const-f}. 

\begin{lemma}\label{lem-top-k}
Suppose $\ve{q}$ has been masked to a set of $1 \leq k< n$ coordinates. Then there always exists an optimal solution $\ve{\pi}$ of \eqref{def-pb-f-md-1-general} whose support is those $k$ coordinates.
\end{lemma}
Proof in Appendix, Section \ref{sec-proof-lem-top-k}. By Lemma \ref{lem-equiv-sol}, there also exists an optimal solution $(\ve{r}, \ve{s}) \in \textsc{md}^2_f(\ve{p}, \ve{q}; D)$ whose coordinates satisfy $q_i = 0 \Rightarrow r_i = 0 \wedge s_i = 0, \forall i \in [n]$. We stress the substantial practical importance of Lemma \ref{lem-top-k}: when masking the target, instead of solving MD over the potentially huge set of $n$ coordinates / tokens (e.g. $n \approx 10^5$), we can restrict MD over the subset defining the mask (e.g. $k=16$). Also, this does not affect the connection with boosting but must be applied \textit{mutatis mutandis} for the parameters involved.

The practical consequence of Lemma~\ref{lem-top-k} is significant for modern LLM serving. 
In production pipelines, target verification is often executed with top-$k$ 
truncation. When it is the case, \textit{identifying the top-$k$ 
tokens and renormalizing their probabilities is already performed by the baseline speculative 
decoding verification stage.} Consequently, mentored decoding incurs \textbf{zero additional overhead} for top-$k$ extraction and the sorting complexity drops from $O(n \log n)$ to $O(k \log k)$. Hence, mentored decoding achieves higher acceptance rates with negligible wall-clock overhead.

\paragraph{Boosting the boosting advantages beyond \eqref{eq-def-wla}} The weak learning assumption has been instrumental in showing that boosting effectively works by amplifying the performances of models barely better than random. Here, we show that, if instead of absolute performances we focus on relative performances of the weak models, i.e. correlations between each other, then there is a similar amplification framework which, instead of providing a boosting advantage linear in the number of models $T$, gets a boosting advantage which is exponential in $T$. This result is facilitated in our case (vs AdaBoost) because weights in \eqref{eq-def-weights} are automatically normalized in the simplex. Denote $\tilde{\ve{h}}_t \defeq (1/(2 h_{t,\infty})) \cdot \ve{h}_t$, so we have $|\ve{y}_i^\top \tilde{\ve{h}}_t|\leq 1$ for $i \in [m]$. For any $t\geq 1$ and any sequence\footnote{To spare notations, we write for example $\mathcal{I} = 1, 2, 5$ without other symbol.} of integers $\mathcal{I} \geq t$ (element-wise), we let
\begin{eqnarray}
  \expect_t(\mathcal{I}) & \defeq & \sum_{i \in [m]} w_{ti} \cdot \prod_{j \in \mathcal{I}} \ve{y}_{i}^\top \tilde{\ve{h}}_j(\bm{x}_i).\label{def-E}
\end{eqnarray}
We can unravel its formula with respect to the weight index, from  \eqref{eq-def-mut} and \eqref{eq-def-weights}, into a very useful formula:
\begin{eqnarray}
  \expect_{t+1}(\mathcal{I}) & = & \sum_{i \in [m]} w_{ti} \cdot \frac{1 -\mu_{t} \cdot \ve{y}_{i}^\top \tilde{\ve{h}}_{t}(\bm{x}_i)}{1-\mu_{t}^2} \cdot \prod_{j \in \mathcal{I}} \ve{y}_{i}^\top \tilde{\ve{h}}_j(\bm{x}_i) .\nonumber\\
  & = & \frac{\expect_{t}(\mathcal{I}) - \expect_{t}(t) \cdot \expect_{t}(t, \mathcal{I})}{1 - \expect^2_{t}(t)} . \label{eq-ttm1}
\end{eqnarray}
Suppose we replace the content of \eqref{eq-def-wla} by the following. First, assume $\mu_{t} > 0$, a lightweight assumption since otherwise the chosen hypothesis does not perform better than random. We also add assumptions for $t\geq 1$ that $\tilde{\ve{h}}_{t+1}$ is not too bad with respect to $\tilde{\ve{h}}_{t}$ while being different enough from $\tilde{\ve{h}}_{t}$, which is especially relevant for large models. These are grouped in a setting called Weak Correlation Assumption.
\begin{lemma}\label{lem-ba-exp}
  Let $T>1$ be a number of boosting iterations and assume the following Weak Correlation Assumption holds: 
\begin{align}
  \exists \beta \geq 0, \delta > 0 :
  \left\{
  \begin{array}{ll}
           (1.) & \mu_{t} > 0\\
           (2.) & \expect_{t}(t+1) \geq (1-\beta) \mu_{t}\\
           (3.) & \expect_{t}(t ,t+1) \leq 1 - \beta -  (1+\delta) \mu_{t}^2
  \end{array}
                  \right., \forall t = 1, 2, ... T. \label{eq-def-wca}\tag{\textbf{WCA}}
    \end{align}
   Then the boosting advantage grows exponentially with $T$ as:
   \begin{eqnarray}
  A\left(\{\ve{h}_{t}\}_{t\in [T]}\right) & \geq & \mu_1^2 \cdot \frac{(1+\delta)^{2T} - 1}{2 \delta + \delta^2}. \label{b-adv}
\end{eqnarray}
\end{lemma}
Proof in Appendix, Section \ref{sec-proof-lem-ba-exp}. \eqref{eq-def-wca} guarantees much better rates than \eqref{eq-def-wla}, \textit{but} it can typically hold for a much more limited number of iterations. Indeed, the crux of \eqref{eq-def-wca} is to imply a geometric increase in edges, $\mu_{t+1} \geq  (1+\delta) \mu_{t}$, and we obviously observe $\mu_{t+1} \leq 1$. Yet, let us compare what the \eqref{eq-def-wla} and \eqref{eq-def-wca} can get after a maximal number $T^*$ of iterations for \eqref{eq-def-wca} to stand. To simplify, assume $\mu_{1} \geq \upgamma$ and pick $\delta = \upgamma$ of the \eqref{eq-def-wla}. To ensure $\mu_{t} \leq 1, \forall t$, we must have $T \leq T^* \defeq \log(1/\upgamma) / \log(1+\upgamma)$.

After such a number of iterations, the boosting advantage in \eqref{b-adv} satisfies $A\left(\{\ve{h}_{t}\}_{t\in [T^*]}\right) \geq A_{WCA}$ with
\begin{eqnarray*}
  A_{WCA}  & \defeq & \upgamma^2 \cdot \frac{\frac{(1+\upgamma)^2}{\upgamma^2} - 1}{2 \upgamma + \upgamma^2} = \frac{1+2\upgamma}{2\upgamma +\upgamma^2},
\end{eqnarray*}
while the boosting advantage in \eqref{b-adv-lin} yields only $A\left(\{\ve{h}_{t}\}_{t\in [T^*]}\right) \geq A_{WLA}$ with
\begin{eqnarray*}
  A_{WLA}  & \defeq & \upgamma^2 \cdot \frac{\log\left(\frac{1}{\upgamma}\right)}{\log(1+\upgamma)},
\end{eqnarray*}
and it is not hard to show that $A_{WCA} = \Omega(A_{WLA} / \upgamma)$, yielding a potential drop in the boosting rate dependence $O(1/\upgamma^2)$ in \eqref{eq-bsup-T} to the much more seldom $O(1/\upgamma)$ under \eqref{eq-def-wca}, which can be of independent interest in the context of boosting \citep[Open problems]{DBLP:journals/theoretics/AlonGHM23} but comes with the substantial caveat that the number of iterations $T$ during which the \eqref{eq-def-wca} can hold is substantially smaller than for \eqref{eq-def-wla}.

\begin{remark}
We make two important remarks regarding the \eqref{eq-def-wca} framework:
\begin{itemize}
\item from the standpoint of the weak/strong learning framework, a crucial question about the \eqref{eq-def-wca} is how "weak" it is. If we take the \eqref{eq-def-wla}, as $\upgamma \rightarrow 0$, the requirements of the \eqref{eq-def-wla} converge to the fact that $\tilde{\ve{h}}_t$ be just better than random guessing. This turns out to be the same for \eqref{eq-def-wca}: as $\beta \rightarrow 1$ and $\delta \rightarrow -1$, the requirements coalesce to the sole $\mu_{.} > 0$ -- i.e. $\tilde{\ve{h}}_t$ be just better than random guessing.
\item disregarding $\delta$, the smaller $\beta$, the weaker is (3.), but in fact, in classical boosting and in our setting where we may pick models from a pool, the picking is greedy, which always imposes $\expect_{t}(t+1) \leq \mu_{t}$ (otherwise, $\ve{h}_{t+1}$ would have been picked at iteration $t$), and leads to $\expect_{t}(t+1) < \mu_{t}$ often, so $\beta > 0$.
\end{itemize}
\end{remark}

\paragraph{Potential (sub)optimality of the greedy boosted sequence} The sequence of boosted classifiers is built iteratively, but the boosting advantage \eqref{def-boostad} is not invariant by permutation in the sequence. Usually, $\ve{h}_t$ is the "best" classifier at iteration $t$, say by maximizing $\mu_t$. When we pick it from a pool of available classifiers, which is especially relevant in our case, a natural question comes as to whether this simple strategy always delivers the best boosting advantage. A simple results shows that the greedy pick of the best classifier for $\mu_t$ can, in a particular case highlighted below, lead to a suboptimal boosting advantage even from a very local standpoint, i.e. by just permuting two successive classifiers (say $\ve{h}_t$ and $\ve{h}_{t+1}$) in the sequence.

The reasoning is straightforward and comes directly from \eqref{eq-ttm1}: since $\mu_{t+1} = \expect_{t+1}(t+1)$ and $\mu_{t} = \expect_{t}(t)$, we have:
\begin{eqnarray*}
  \mu_{t+1} & = & \frac{\expect_{t}(t+1) - \mu_{t} \cdot \expect_{t}(t ,t+1)}{1 - \mu_{t}^2 } .
\end{eqnarray*}
Denote $\ve{h}_d$ the hypothesis used in $\mu_{t+1}$ and $\ve{h}_c$ the hypothesis used in $\mu_{t}$. Under the best greedy fit scenario, we have
\begin{eqnarray}
  \mu_{t} > |\tilde{\mu}_{t}| \quad \mbox{ with } \tilde{\mu}_{t} \defeq \expect_{t}(d) \label{eq-greedy-choice}
\end{eqnarray}
(we remove the possibility of identity for simplicity). The contribution to the boosting advantage of adding in this order $\ve{h}_c, \ve{h}_d$ is $\mu^2_{t} + \mu^2_{t+1}$. There is also the (seemingly) suboptimal scenario of preferring the sequence $\ve{h}_d, \ve{h}_c$, for an alternative contribution to the boosting advantage $\tilde{\mu}^2_{t} + \tilde{\mu}^2_{t+1}$ with $\tilde{\mu}_{t+1} \defeq \expect_{t+1}(c)$. Surprisingly perhaps, we show that there is a simple condition on the \textit{covariance} of the two hypotheses such that the alternative scenario is strictly better than the best greedy fit.
\begin{lemma}\label{lem-subopt-greedy}
  With the definition stated above, under the greedy choice condition \eqref{eq-greedy-choice}, there exists $0< \rho < 0.34$ depending on $\mu_{t}, \tilde{\mu}_{t}$ such that $\tilde{\mu}^2_{t} + \tilde{\mu}^2_{t+1} > \mu^2_{t} + \mu^2_{t+1}$ iff one of the following holds:
  \begin{itemize}
  \item [(i)] $\tilde{\mu}_{t}>0$ and $\expect_{t}(c, d) - \expect_{t}(c) \expect_{t}(d) \in (0, \rho)$, or
  \item [(ii)] $\tilde{\mu}_{t}<0$ and $\expect_{t}(c, d) - \expect_{t}(c) \expect_{t}(d) \in (-\rho,0)$.
    \end{itemize}
  \end{lemma}
  The proof, in Appendix, Section \ref{sec-proof-lem-subopt-greedy}, makes $\rho$ explicit. 

\section{Conclusion}\label{sec-conc}

There have been a number of recent approaches relaxing the key constraint of speculative decoding -- that the output distribution be equal to the target's. While inference speedup was the original intent, a few recent papers also observed experimentally that the equivalent output model can sometimes beat the target when it comes to model quality. In our paper, we have shown that such a remarkable feat -- speeding up inference while getting a better model -- is indeed possible. Our two main bricks are mentored decoding as the formal setting authorizing deviations from the target, and boosting to evaluate the quality of the model produced by mentored decoding. In the course of getting to this result, we derived several new key properties of the mentored decoding setting. Among these, the particular geometric appeal of the total variation case is interesting for the variety of optimal solutions it supports, some of which are very convenient for boosting, but others might as well be relevant for other constraints. We also reached an utterly simple approximation scheme of the optimal solutions for any $f$-divergence, also with interesting ties to boosting, which shows that there exists a data structure independent from the choice of $f$, but which, once computed, can be used for any $f$-divergence to get the two parameters to compute the optimal mentored decoding solution as fast as for speculative decoding. Getting those parameters is done in logarithmic time via our breakpoint data structure. 
While constructing the breakpoints involves sorting probability ratios, this overhead 
is practically negligible: under standard top-$k$ decoding, sorting operates over only 
$k \ll n$ elements already identified by the baseline pipeline, incurring negligible compute 
on accelerators.

Another interesting avenue for future research relies on the boosting part of our paper. The boosting part of our approach is efficient with respect to the canon of AdaBoost: it is self-normalized and can be carried out with a bypass of boosting's famous weight updates. This latter property goes with computing linear correlation coefficients between models, which can be costly when the number of models increases, but at least for a few models it shows that the "architecture" of boosting does not necessarily need to be carved in the computation of the composite model producing the mentored distribution (notwithstanding the risk of numerical approximation errors with weight updates in traditional (Ada)boosting). Given the training cost of even the smallest LLM models, the LLM space -- public or private -- has plenty stored models for which boosting directly applies, but not the mentored decoding framework which originally applies to two models only. Extending mentored decoding beyond the (1 drafter, 1 target) setting is an interesting question.

\section*{Acknowledgments}

The authors thank Ariel Brand, Yishay Mansour, Nir Shabat and Ayala Shaubi-Mann for early discussions on this material.

\bibliography{bibgen}

@misc{chen2023accelerating,
      title={Accelerating Large Language Model Decoding with Speculative Sampling}, 
      author={Charlie Chen and Sebastian Borgeaud and Geoffrey Irving and Jean-Baptiste Lespiau and Laurent Sifre and John Jumper},
      year={2023},
      eprint={2302.01318},
      archivePrefix={arXiv},
      primaryClass={cs.CL}
}

@inproceedings{leviathan2023fast,
  title={Fast inference from transformers via speculative decoding},
  author={Leviathan, Yaniv and Kalman, Matan and Matias, Yossi},
  booktitle={International Conference on Machine Learning},
  pages={19274--19286},
  year={2023},
  organization={PMLR}
}

@BOOK{anMO,
       TITLE = "Methods of Information Geometry",
       AUTHOR = "{Shun-ichi} Amari and Hiroshi Nagaoka",
       YEAR = "2000",
       PUBLISHER = "Oxford University Press"}

@misc{li2026targetimitationcollaborationspeculative,
      title={Beyond the Target: From Imitation to Collaboration in Speculative Decoding}, 
      author={Jinze Li and Yixing Xu and Guanchen Li and Jinfeng Xu and Shuo Yang and Yang Zhang and Xuanwu Yin and Dong Li and Edith C. H. Ngai and Emad Barsoum},
      year={2026},
      eprint={2605.24793},
      archivePrefix={arXiv},
      primaryClass={cs.CL},
      url={https://arxiv.org/abs/2605.24793}, 
}

@inproceedings{
byun2025model,
title={3-Model Speculative Decoding},
author={Sanghyun Byun and Mohanad Odema and Jung Ick Guack and Baisub Lee and Jacob Song and Woo Seong Chung},
booktitle={NeurIPS 2025 Workshop on Structured Probabilistic Inference {\&} Generative Modeling},
year={2025},
url={https://openreview.net/forum?id=2e2RCF4Ncc}
}

@inproceedings{wang-etal-2025-alignment,
    title = "Alignment-Augmented Speculative Decoding with Alignment Sampling and Conditional Verification",
    author = "Wang, Jikai  and
      Tian, Zhenxu  and
      Li, Juntao  and
      Xia, Qingrong  and
      Duan, Xinyu  and
      Wang, Zhefeng  and
      Huai, Baoxing  and
      Zhang, Min",
    editor = "Christodoulopoulos, Christos  and
      Chakraborty, Tanmoy  and
      Rose, Carolyn  and
      Peng, Violet",
    booktitle = "Proceedings of the 2025 Conference on Empirical Methods in Natural Language Processing",
    month = nov,
    year = "2025",
    address = "Suzhou, China",
    publisher = "Association for Computational Linguistics",
    url = "https://aclanthology.org/2025.emnlp-main.343/",
    doi = "10.18653/v1/2025.emnlp-main.343",
    pages = "6751--6763",
    ISBN = "979-8-89176-332-6"
}

@inproceedings{
bachmann2025judge,
title={Judge Decoding: Faster Speculative Sampling Requires Going Beyond Model Alignment},
author={Gregor Bachmann and Sotiris Anagnostidis and Albert Pumarola and Markos Georgopoulos and Artsiom Sanakoyeu and Yuming Du and Edgar Sch{\"o}nfeld and Ali Thabet and Jonas K Kohler},
booktitle={The Thirteenth International Conference on Learning Representations},
year={2025},
url={https://openreview.net/forum?id=mtSSFiqW6y}
}

@inproceedings{zhang-etal-2025-draft,
    title = "Draft Model Knows When to Stop: Self-Verification Speculative Decoding for Long-Form Generation",
    author = "Zhang, Ziyin  and
      Xu, Jiahao  and
      Liang, Tian  and
      Chen, Xingyu  and
      He, Zhiwei  and
      Wang, Rui  and
      Tu, Zhaopeng",
    editor = "Christodoulopoulos, Christos  and
      Chakraborty, Tanmoy  and
      Rose, Carolyn  and
      Peng, Violet",
    booktitle = "Proceedings of the 2025 Conference on Empirical Methods in Natural Language Processing",
    month = nov,
    year = "2025",
    address = "Suzhou, China",
    publisher = "Association for Computational Linguistics",
    url = "https://aclanthology.org/2025.emnlp-main.844/",
    doi = "10.18653/v1/2025.emnlp-main.844",
    pages = "16685--16697",
    ISBN = "979-8-89176-332-6"
}

@inproceedings{pankratov-alistarh-2026-speculative,
    title = "Speculative Decoding Speed-of-Light: Optimal Lower Bounds via Branching Random Walks",
    author = "Pankratov, Sergey  and
      Alistarh, Dan",
    editor = "Demberg, Vera  and
      Inui, Kentaro  and
      Marquez, Llu{\'i}s",
    booktitle = "Proceedings of the 19th Conference of the {E}uropean Chapter of the {A}ssociation for {C}omputational {L}inguistics (Volume 1: Long Papers)",
    month = mar,
    year = "2026",
    address = "Rabat, Morocco",
    publisher = "Association for Computational Linguistics",
    url = "https://aclanthology.org/2026.eacl-long.301/",
    doi = "10.18653/v1/2026.eacl-long.301",
    pages = "6404--6418",
    ISBN = "979-8-89176-380-7"
}

@inproceedings{
zhong2025speeding,
title={Speeding up Speculative Decoding via Sequential Approximate Verification},
author={Meiyu Zhong and Noel Teku and Ravi Tandon},
booktitle={ES-FoMo III: 3rd Workshop on Efficient Systems for Foundation Models},
year={2025},
url={https://openreview.net/forum?id=Y4KcfotBkf}
}

@inproceedings{
liu2026speculative,
title={Speculative Decoding: Performance or Illusion?},
author={Xiaoxuan Liu and Jiaxiang Yu and Jongseok Park and Ion Stoica and Alvin Cheung},
booktitle={Ninth Conference on Machine Learning and Systems},
year={2026},
url={https://openreview.net/forum?id=fzkqtezFEi}
}

@article{DBLP:journals/ai/NockN07,
  author       = {Richard Nock and
                  Frank Nielsen},
  title        = {A $\mathbb{R}$eal generalization of discrete {AdaBoost}},
  journal      = {Artif. Intell.},
  volume       = {171},
  number       = {1},
  pages        = {25--41},
  year         = {2007},
  url          = {https://doi.org/10.1016/j.artint.2006.10.014},
  doi          = {10.1016/J.ARTINT.2006.10.014},
  bibsource    = {dblp computer science bibliography, https://dblp.org}
}

@article{kTN,
 author = "D.-E. Knuth",
 journal = {The American Mathematical Monthly},
 number = {5},
 pages = {403--422},
 publisher = {Mathematical Association of America},
 title = {Two Notes on Notation},
 volume = {99},
 year = {1992}
}

@article{10.1214/aos/1024691352,
author = {Peter Bartlett and Yoav Freund and Wee Sun Lee and Robert E. Schapire},
title = {{Boosting the margin: a new explanation for the effectiveness of voting methods}},
volume = {26},
journal = {The Annals of Statistics},
number = {5},
publisher = {Institute of Mathematical Statistics},
pages = {1651 -- 1686},
year = {1998},
doi = {10.1214/aos/1024691352},
URL = {https://doi.org/10.1214/aos/1024691352}
}

@article{DBLP:journals/theoretics/AlonGHM23,
  author       = {Noga Alon and
                  Alon Gonen and
                  Elad Hazan and
                  Shay Moran},
  title        = {Boosting Simple Learners},
  journal      = {TheoretiCS},
  volume       = {2},
  eid          = {8},
  year         = {2023},
  url          = {https://doi.org/10.46298/theoretics.23.8},
  doi          = {10.46298/THEORETICS.23.8},
  bibsource    = {dblp computer science bibliography, https://dblp.org}
}

@inproceedings{DBLP:conf/aaai/AmidNNW24,
  author       = {Ehsan Amid and
                  Frank Nielsen and
                  Richard Nock and
                  Manfred K. Warmuth},
  editor       = {Michael J. Wooldridge and
                  Jennifer G. Dy and
                  Sriraam Natarajan},
  title        = {Optimal Transport with Tempered Exponential Measures},
  booktitle    = {Thirty-Eighth {AAAI} Conference on Artificial Intelligence, {AAAI}
                  2024, Thirty-Sixth Conference on Innovative Applications of Artificial
                  Intelligence, {IAAI} 2024, Fourteenth Symposium on Educational Advances
                  in Artificial Intelligence, {EAAI} 2024, February 20-27, 2024, Vancouver,
                  Canada},
  pages        = {10838--10846},
  publisher    = {{AAAI} Press},
  year         = {2024},
  url          = {https://doi.org/10.1609/aaai.v38i10.28957},
  doi          = {10.1609/AAAI.V38I10.28957},
  bibsource    = {dblp computer science bibliography, https://dblp.org}
}

@inproceedings{DBLP:conf/aistats/AmidNW23,
  author       = {Ehsan Amid and
                  Richard Nock and
                  Manfred K. Warmuth},
  editor       = {Francisco J. R. Ruiz and
                  Jennifer G. Dy and
                  Jan{-}Willem van de Meent},
  title        = {Clustering above Exponential Families with Tempered Exponential Measures},
  booktitle    = {International Conference on Artificial Intelligence and Statistics,
                  25-27 April 2023, Palau de Congressos, Valencia, Spain},
  series       = {Proceedings of Machine Learning Research},
  volume       = {206},
  pages        = {2994--3017},
  publisher    = {{PMLR}},
  year         = {2023},
  url          = {https://proceedings.mlr.press/v206/amid23a.html},
  bibsource    = {dblp computer science bibliography, https://dblp.org}
}

@inproceedings{DBLP:conf/nips/NockAW23,
  author       = {Richard Nock and
                  Ehsan Amid and
                  Manfred K. Warmuth},
  editor       = {Alice Oh and
                  Tristan Naumann and
                  Amir Globerson and
                  Kate Saenko and
                  Moritz Hardt and
                  Sergey Levine},
  title        = {Boosting with Tempered Exponential Measures},
  booktitle    = {Advances in Neural Information Processing Systems 36: Annual Conference
                  on Neural Information Processing Systems 2023, NeurIPS 2023, New Orleans,
                  LA, USA, December 10 - 16, 2023},
  year         = {2023},
  url          = {http://papers.nips.cc/paper\_files/paper/2023/hash/82d3258eb58ceac31744a88005b7ddef-Abstract-Conference.html},
  bibsource    = {dblp computer science bibliography, https://dblp.org}
}

@BOOK{nGT,
      TITLE = "Generalized thermostatistics",
      AUTHOR = "J. Naudts",
      YEAR = "2011",
      PUBLISHER = "Springer"}

@article{fdiv-AliSilvey-1966,
  title={A general class of coefficients of divergence of one distribution from another},
  author={Ali, Syed Mumtaz and Silvey, Samuel D},
  journal={Journal of the Royal Statistical Society: Series B (Methodological)},
  volume={28},
  number={1},
  pages={131--142},
  year={1966},
  publisher={Wiley Online Library}
}

@ARTICLE{cEI,
      TITLE = "Eine informationstheoretische Ungleichung und ihre Anwendung auf den Beweis der Ergodizitat von {M}arkoffschen Ketten",
      AUTHOR = "I. Csisz{\'a}r",
      JOURNAL = "Magyar. Tud. Akad. Mat. Kutato Int. Kozl.",
      VOLUME = "8",
      PAGES = "85--108",
      YEAR = "1963"}

@ARTICLE{cAC,
      TITLE = "A class of measures of informativity of observation channels",
      AUTHOR = "I. Csisz{\'a}r",
      JOURNAL = "Periodica Mathematica Hungarica",
      VOLUME = "2",
      PAGES = "191--213",
      YEAR = "1972"}

@BOOK{sfBF,
       TITLE = "Boosting, Foundations and Algorithms",
       AUTHOR = "R.-E. Schapire and Y. Freund",
       YEAR = "2012",
       PUBLISHER = "MIT Press"}

@BOOK{rCA,
      TITLE = "Convex {A}nalysis",
      AUTHOR = "R. T. Rockafellar",
      YEAR = "1970",
      PUBLISHER = "Princeton University Press"}

@article{stern2018blockwise,
  title={Blockwise parallel decoding for deep autoregressive models},
  author={Stern, Mitchell and Shazeer, Noam and Uszkoreit, Jakob},
  journal={Advances in Neural Information Processing Systems},
  volume={31},
  year={2018}
}

@article{xia2022speculative,
  title={Speculative Decoding: Exploiting Speculative Execution for Accelerating Seq2seq Generation},
  author={Xia, Heming and Ge, Tao and Wang, Peiyi and Chen, Si-Qing and Wei, Furu and Sui, Zhifang},
  journal={Findings of the Association for Computational Linguistics: EMNLP 2023},
  pages={3909--3925},
  year={2023}
}

@article{yin2024theoretical,
  title={A theoretical perspective for speculative decoding algorithm},
  author={Yin, Ming and Chen, Minshuo and Huang, Kaixuan and Wang, Mengdi},
  journal={Advances in Neural Information Processing Systems},
  volume={37},
  pages={128082--128117},
  year={2024}
}

@article{sun2023spectr,
  title={SpecTr: Fast Speculative Decoding via Optimal Transport},
  author={Sun, Ziteng and Suresh, Ananda Theertha and Ro, Jae Hun and Beirami, Ahmad and Jain, Himanshu and Yu, Felix},
  journal={Advances in Neural Information Processing Systems},
  volume={36},
  pages={30222--30242},
  year={2023}
}

@inproceedings{hu2025towards,
  title={Towards Optimal Multi-Draft Speculative Decoding},
  author={Hu, Zhengmian and Zheng, Tong and Viswanathan, Vignesh and Chen, Ziyi and Rossi, Ryan and Wu, Yihan and Manocha, Dinesh and Huang, Heng},
  booktitle={International Conference on Learning Representations (ICLR)},
  volume={2025},
  pages={3181--3203},
  year={2025}
}

@inproceedings{miao2024specinfer,
  title={Specinfer: Accelerating large language model serving with tree-based speculative inference and verification},
  author={Miao, Xupeng and Oliaro, Gabriele and Zhang, Zhihao and Cheng, Xinhao and Wang, Zeyu and Zhang, Zhengxin and Wong, Rae Ying Yee and Zhu, Alan and Yang, Lijie and Shi, Xiaoxiang and others},
  booktitle={Proceedings of the 29th ACM International Conference on Architectural Support for Programming Languages and Operating Systems, Volume 3},
  pages={932--949},
  year={2024}
}

@article{chen2024sequoia,
  title={Sequoia: Scalable and robust speculative decoding},
  author={Chen, Zhuoming and May, Avner and Svirschevski, Ruslan and Huang, Yuhsun and Ryabinin, Max and Jia, Zhihao and Chen, Beidi},
  journal={Advances in Neural Information Processing Systems},
  volume={37},
  pages={129531--129563},
  year={2024}
}

@article{cai2024medusa,
  title={Medusa: Simple LLM Inference Acceleration Framework with Multiple Decoding Heads},
  author={Cai, Tianle and Li, Yuhong and Geng, Zhengyang and Peng, Hongwu and Lee, Jason D and Chen, Deming and Dao, Tri},
  journal={arXiv preprint arXiv:2401.10774},
  year={2024}
}

@article{li2024eagle,
  title={EAGLE: Speculative Sampling Requires Rethinking Feature Uncertainty},
  author={Li, Yuhui and Wei, Fangyun and Zhang, Chao and Zhang, Hongyang},
  journal={International Conference on Machine Learning (ICML)},
  year={2024}
}

@inproceedings{li2024eagle2,
  title={Eagle-2: Faster inference of language models with dynamic draft trees},
  author={Li, Yuhui and Wei, Fangyun and Zhang, Chao and Zhang, Hongyang},
  booktitle={Proceedings of the 2024 conference on empirical methods in natural language processing},
  pages={7421--7432},
  year={2024}
}

@article{li2026eagle,
  title={Eagle-3: Scaling up inference acceleration of large language models via training-time test},
  author={Li, Yuhui and Wei, Fangyun and Zhang, Chao and Zhang, Hongyang},
  journal={Advances in Neural Information Processing Systems},
  volume={38},
  pages={136737--136756},
  year={2026}
}

@article{gloeckle2024better,
  title={Better \& faster large language models via multi-token prediction},
  author={Gloeckle, Fabian and Idrissi, Badr Youbi and Rozi{\`e}re, Baptiste and Lopez-Paz, David and Synnaeve, Gabriel},
  journal={arXiv preprint arXiv:2404.19737},
  year={2024}
}

@inproceedings{zhang2024draft,
  title={Draft\& verify: Lossless large language model acceleration via self-speculative decoding},
  author={Zhang, Jun and Wang, Jue and Li, Huan and Shou, Lidan and Chen, Ke and Chen, Gang and Mehrotra, Sharad},
  booktitle={Proceedings of the 62nd Annual Meeting of the Association for Computational Linguistics (Volume 1: Long Papers)},
  pages={11263--11282},
  year={2024}
}

@article{liu2024kangaroo,
  title={Kangaroo: Lossless self-speculative decoding for accelerating llms via double early exiting},
  author={Liu, Fangcheng and Tang, Yehui and Liu, Zhenhua and Ni, Yunsheng and Tang, Duyu and Han, Kai and Wang, Yunhe},
  journal={Advances in Neural Information Processing Systems},
  volume={37},
  pages={11946--11965},
  year={2024}
}

@article{kim2023speculative,
  title={Speculative decoding with big little decoder},
  author={Kim, Sehoon and Mangalam, Karttikeya and Moon, Suhong and Malik, Jitendra and Mahoney, Michael W and Gholami, Amir and Keutzer, Kurt},
  journal={Advances in Neural Information Processing Systems},
  volume={36},
  pages={39236--39256},
  year={2023}
}

@article{fu2024break,
  title={Break the sequential dependency of llm inference using lookahead decoding},
  author={Fu, Yichao and Bailis, Peter and Stoica, Ion and Zhang, Hao},
  journal={arXiv preprint arXiv:2402.02057},
  year={2024}
}

@inproceedings{he2024rest,
  title={Rest: Retrieval-based speculative decoding},
  author={He, Zhenyu and Zhong, Zexuan and Cai, Tianle and Lee, Jason and He, Di},
  booktitle={Proceedings of the 2024 Conference of the North American Chapter of the Association for Computational Linguistics: Human Language Technologies (Volume 1: Long Papers)},
  pages={1582--1595},
  year={2024}
}

@article{yang2023inference,
  title={Inference with reference: Lossless acceleration of large language models},
  author={Yang, Nan and Ge, Tao and Wang, Liang and Jiao, Binxing and Jiang, Daxin and Yang, Linjun and Majumder, Rangan and Wei, Furu},
  journal={arXiv preprint arXiv:2304.04487},
  year={2023}
}

@article{chen2026dflash,
  title={DFlash: Block Diffusion for Flash Speculative Decoding},
  author={Chen, Jian and Liang, Yesheng and Liu, Zhijian},
  journal={arXiv preprint arXiv:2602.06036},
  year={2026}
}

@inproceedings{zhou2024distillspec,
  title={DistillSpec: Improving Speculative Decoding via Knowledge Distillation},
  author={Zhou, Yongchao and Lyu, Kaifeng and Rawat, Ankit Singh and Menon, Aditya Krishna and Rostamizadeh, Afshin and Kumar, Sanjiv and Kagy, Jean-Fran{\c{c}}ois and Agarwal, Rishabh},
  booktitle={International Conference on Learning Representations (ICLR)},
  year={2024}
}

@misc{Tran-Thien_2023,
  title={An optimal lossy variant of speculative decoding},
  author={Tran-Thien, Vivien},
  howpublished={\url{https://vivien000.github.io/blog/journal/a-provably-optimal-lossy-variant-of-speculative-decoding.html}},
  journal={Unsupervised Thoughts (blog)},
  year={2023}
}

@inproceedings{yuan2024speculative,
  title={Speculative Contrastive Decoding},
  author={Yuan, Hongyi and Lu, Keming and Huang, Fei and Yuan, Zheng and Zhou, Chang},
  booktitle={Proceedings of the 62nd Annual Meeting of the Association for Computational Linguistics (Volume 2: Short Papers)},
  pages={56--64},
  year={2024}
}

@inproceedings{narasimhan2025faster,
  title={Faster cascades via speculative decoding},
  author={Narasimhan, Harikrishna and Jitkrittum, Wittawat and Rawat, Ankit Singh and Kim, Seungyeon and Gupta, Neha and Menon, Aditya Krishna and Kumar, Sanjiv},
  booktitle={International Conference on Learning Representations},
  volume={2025},
  pages={44949--44987},
  year={2025}
}

@article{hao2026cactus,
  title={Cactus: Accelerating Auto-Regressive Decoding with Constrained Acceptance Speculative Sampling},
  author={Hao, Yongchang and Mou, Lili},
  journal={International Conference on Learning Representations (ICLR)},
  year={2026}
}

@article{xia2026practical,
  title={A Practical Investigation of Training-free Relaxed Speculative Decoding},
  author={Xia, Guoxuan and Ribar, Luka and Balanca, Paul},
  journal={arXiv preprint arXiv:2607.08690},
  year={2026}
}

@inproceedings{holsman2025fuzzy,
  title={Fuzzy speculative decoding for a tunable accuracy-runtime tradeoff},
  author={Holsman, Maximilian and Huang, Yukun and Dhingra, Bhuwan},
  booktitle={Findings of the Association for Computational Linguistics: ACL 2025},
  pages={26257--26273},
  year={2025}
}

@inproceedings{wang2025diversed,
title={{DIVERSED}: Relaxed Speculative Decoding via Dynamic Ensemble Verification},
author={Ziyi Wang and Siva Rajesh Kasa and Ankith M S and Santhosh Kumar Kasa and Jiaru Zou and Nan Jiang and Sumit Negi and Ruqi Zhang and Qifan Song},
booktitle={NeurIPS 2025 Workshop on Efficient Reasoning},
year={2025},
url={https://openreview.net/forum?id=yrkf0GxTe7}
}

@inproceedings{DBLP:conf/icml/LiaoXD0MSSX25,
  author       = {Baohao Liao and
                  Yuhui Xu and
                  Hanze Dong and
                  Junnan Li and
                  Christof Monz and
                  Silvio Savarese and
                  Doyen Sahoo and
                  Caiming Xiong},
  editor       = {Aarti Singh and
                  Maryam Fazel and
                  Daniel Hsu and
                  Simon Lacoste{-}Julien and
                  Felix Berkenkamp and
                  Tegan Maharaj and
                  Kiri Wagstaff and
                  Jerry Zhu},
  title        = {Reward-Guided Speculative Decoding for Efficient {LLM} Reasoning},
  booktitle    = {Forty-second International Conference on Machine Learning, {ICML}
                  2025, Vancouver, BC, Canada, July 13-19, 2025},
  series       = {Proceedings of Machine Learning Research},
  volume       = {267},
  publisher    = {{PMLR} / OpenReview.net},
  year         = {2025},
  url          = {https://proceedings.mlr.press/v267/liao25f.html},
  bibsource    = {dblp computer science bibliography, https://dblp.org}
}

@inproceedings{DBLP:conf/aaai/QinHPCS25,
  author       = {Zongyue Qin and
                  Zifan He and
                  Neha Prakriya and
                  Jason Cong and
                  Yizhou Sun},
  editor       = {Toby Walsh and
                  Julie Shah and
                  Zico Kolter},
  title        = {Dynamic-Width Speculative Beam Decoding for {LLM} Inference},
  booktitle    = {Thirty-Ninth {AAAI} Conference on Artificial Intelligence, Thirty-Seventh
                  Conference on Innovative Applications of Artificial Intelligence,
                  Fifteenth Symposium on Educational Advances in Artificial Intelligence,
                  {AAAI} 2025, Philadelphia, PA, USA, February 25 - March 4, 2025},
  pages        = {25056--25064},
  publisher    = {{AAAI} Press},
  year         = {2025},
  url          = {https://doi.org/10.1609/aaai.v39i23.34690},
  doi          = {10.1609/AAAI.V39I23.34690},
  bibsource    = {dblp computer science bibliography, https://dblp.org}
}

@article{zzrsMC,
  author    = {J. Zhu and H. Zou and S. Rosset and T. Hastie},
  title     = "Multi-class {A}daboost",
  journal   = "Statistics and Its Interface",
  volume    = 2,
  pages = "349--360",
  year      = {2009}
}

\newpage
\clearpage
\renewcommand\thesection{\Roman{section}}
\renewcommand\thesubsection{\thesection.\arabic{subsection}}
\renewcommand\thesubsubsection{\thesection.\thesubsection.\arabic{subsubsection}}

\renewcommand*{\thetheorem}{\Alph{theorem}}
\renewcommand*{\thelemma}{\Alph{lemma}}
\renewcommand*{\thecorollary}{\Alph{corollary}}

\renewcommand{\thetable}{A\arabic{table}}


\begin{center}
\Huge{\supplementLong}
\end{center}

This is the Appendix to paper "\papertitle". To
differentiate with the numberings in the main file, the numbering of
Theorems, etc. is letter-based (A, B, ...).

\section*{Table of contents}

\noindent \textbf{Proofs} \hrulefill Pg \pageref{sec-app-proofs}\\

\noindent $\hookrightarrow$ Proof of Lemma \ref{lem-equiv-sol} \hrulefill Pg \pageref{sec-proof-lem-equiv-sol}\\
\noindent $\hookrightarrow$ Proof of Lemma \ref{lem-slater} \hrulefill Pg \pageref{sec-proof-lem-slater}\\
\noindent $\hookrightarrow$ Proof of Theorem \ref{th-opt1} and Lemma \ref{lemLAlphaBeta} \hrulefill Pg \pageref{sec-proof-th-opt1}\\
\noindent $\hookrightarrow$ Proof of Lemma \ref{lemPOST} \hrulefill Pg \pageref{sec-proof-lemPOST}\\
\noindent $\hookrightarrow$ Proof of Theorem \ref{thmAPPROX1} \hrulefill Pg \pageref{sec-proof-thmAPPROX1}\\
\noindent $\hookrightarrow$ Proof of Theorem \ref{thmUniversal} \hrulefill Pg \pageref{sec-proof-thmUniversal}\\
\noindent $\hookrightarrow$ Proof of Theorem \ref{th-boost-P} \hrulefill Pg \pageref{sec-proof-th-boost-P}\\
\noindent $\hookrightarrow$ Proof of Lemma \ref{lem-tempered-mentor} \hrulefill Pg \pageref{sec-proof-lem-tempered-mentor}\\
\noindent $\hookrightarrow$ Proof of Theorem \ref{thm-boost-and-mentor-TV} \hrulefill Pg \pageref{sec-proof-thm-boost-and-mentor-TV}\\
\noindent $\hookrightarrow$ Proof of Theorem \ref{thm-boost-plus-md-general-f} \hrulefill Pg \pageref{sec-proof-thm-boost-plus-md-general-f}\\
\noindent $\hookrightarrow$ Proof of Theorem \ref{thm-cab-from-c-breakpoints} \hrulefill Pg \pageref{sec-proof-thm-cab-from-c-breakpoints}\\
\noindent $\hookrightarrow$ Proof of Lemma \ref{lem-b-from-a-and-c-breakpoints} \hrulefill Pg \pageref{sec-proof-lem-b-from-a-and-c-breakpoints}\\
\noindent $\hookrightarrow$ Proof of Lemma \ref{lem-cont-ab} \hrulefill Pg \pageref{sec-proof-lem-cont-ab}\\
\noindent $\hookrightarrow$ Proof of Theorem \ref{thm-opt-f} \hrulefill Pg \pageref{sec-proof-thm-opt-f}\\
\noindent $\hookrightarrow$ Proof of Lemma \ref{lem-top-k} \hrulefill Pg \pageref{sec-proof-lem-top-k}\\
\noindent $\hookrightarrow$ Proof of Lemma \ref{lem-ba-exp} \hrulefill Pg \pageref{sec-proof-lem-ba-exp}\\
\noindent $\hookrightarrow$ Proof of Lemma \ref{lem-subopt-greedy} \hrulefill Pg \pageref{sec-proof-lem-subopt-greedy}\\

\newpage

\section{Proofs}\label{sec-app-proofs}

\subsection{Proof of Lemma \ref{lem-equiv-sol}}\label{sec-proof-lem-equiv-sol}

  Suppose $\ve{\pi} \in \textsc{md}^1_f(\ve{p}, \ve{q}; D)$. Then with the choice $\ve{r} \defeq \min \{\ve{1}, \ve{\pi} \oslash \ve{p}\}\in [0,1]^n$ \eqref{eq-r-s-from-pi}, we get $\ve{p}^\top \ve{r} = \ve{1}^\top \min\{\ve{\pi}, \ve{p}\} = 1 - D_{\mathrm{TV}} (\ve{\pi} \|\ve{p})$. We also trivially have $\ve{s} \in \Delta_n$ so the couple $(\ve{r}, \ve{s})$ is feasible for \eqref{def-pb-f-md-2-general}. Suppose it is not optimal and build $\ve{\pi}'$ from a better solution $(\ve{r}', \ve{s}')$ -- thus with $\ve{p}^\top \ve{r}' > \ve{p}^\top \ve{r}$ -- via \eqref{eq-pi-from-r-s}. For any $i\in [n]$, we have $p_i r'_i \leq p_i$ but also $p_i r'_i \leq \pi'_i$ because of \eqref{eq-pi-from-r-s}. So we have $D_{\mathrm{TV}} (\ve{\pi}' \|\ve{p}) = 1 - \ve{1}^\top \min\{\ve{\pi}', \ve{p}\} \leq 1 - \ve{p}^\top \ve{r}' < 1 - \ve{p}^\top \ve{r} = D_{\mathrm{TV}} (\ve{\pi} \|\ve{p})$, which contradicts the fact that $\ve{\pi} \in \textsc{md}^1_f(\ve{p}, \ve{q}; D)$. So we have $(\ve{r}, \ve{s}) \in \textsc{md}^2_f(\ve{p}, \ve{q}; D)$.

  Respectively, suppose $(\ve{r}, \ve{s}) \in \textsc{md}^2_f(\ve{p}, \ve{q}; D)$. Clearly $\ve{\pi}$ as per \eqref{eq-pi-from-r-s} is feasible for \eqref{def-pb-f-md-1-general}. Suppose it is not optimal and build this time $(\ve{r}', \ve{s}')$ from a better solution $\ve{\pi}'$ -- thus with $D_{\mathrm{TV}} (\ve{\pi}' \|\ve{p}) < D_{\mathrm{TV}} (\ve{\pi} \|\ve{p})$ -- via \eqref{eq-r-s-from-pi}. This time, we directly have from the construction of $\ve{r}'$ the chain of (in)equalities $1 - \ve{p}^\top \ve{r}' = D_{\mathrm{TV}} (\ve{\pi}' \|\ve{p}) < D_{\mathrm{TV}} (\ve{\pi} \|\ve{p}) = 1 - \ve{p}^\top \ve{r}$, resulting in $- \ve{p}^\top \ve{r}' <  - \ve{p}^\top \ve{r}$, a contradiction with the fact that $(\ve{r}, \ve{s}) \in \textsc{md}^2_f(\ve{p}, \ve{q}; D)$. So we have $\ve{\pi} \in \textsc{md}^1_f(\ve{p}, \ve{q}; D)$, which ends the main part of the proof of Lemma \ref{lem-equiv-sol}. We easily check the two equivalent formulations for $\ve{s}$ since $\pi_i - p_i r_i = \pi_i - \min\{\pi_i, p_i\} = \max\{0, \pi_i - p_i\}, \forall i \in [n]$.

\subsection{Proof of Lemma  \ref{lem-slater}}\label{sec-proof-lem-slater}

We consider \eqref{def-pb-f-md-1-general} (there is no difficulty in reparameterizing the proof using Lemma \ref{lem-equiv-sol} for \eqref{def-pb-f-md-2-general}). Since $f$-divergences satisfy the identity of indiscernibles, $\ve{q} > \ve{0}$ implies the existence of $\ve{q} \neq \tilde{\ve{\pi}} \in \Delta_n$ such that $\tilde{\ve{\pi}} > 0$ and $D_{f} (\tilde{\ve{\pi}} \|\ve{q}) \leq D$ (and we also have $\tilde{\ve{\pi}} \neq \ve{p}$). We then check Slater's constraint qualification by picking any $0 < \delta < \min_i \{\min\{p_i, \tilde{\pi}_i\}/(p_i + \tilde{\pi}_i)\}$, and choosing
  \begin{eqnarray}
    \ve{t} & \defeq & \delta \cdot (\ve{p} + \tilde{\ve{\pi}}).
  \end{eqnarray}
  For this choice and that of $\delta$ we get $\ve{t} < \ve{p}$ and for the choice (note that $\delta < 1/2$)
  \begin{eqnarray*}
\ve{s} & \defeq & \frac{1-\delta}{1-2\delta} \cdot \tilde{\ve{\pi}} - \frac{\delta}{1-2\delta} \cdot \ve{p},
  \end{eqnarray*}
we have $\ve{1}^\top \ve{s} = 1$ but more importantly $\ve{s}>\ve{0}$, so Slater's constraint qualification are satisfied.

\subsection{Proof of Theorem  \ref{th-opt1} and Lemma  \ref{lemLAlphaBeta}}\label{sec-proof-th-opt1}

For readability reasons, we reparameterize \eqref{def-pb-f-md-2-general} as:
\begin{eqnarray}
      \textsc{md}'_f(\ve{p}, \ve{q}; D) \defeq \arg\min _{\ve{0}\leq \ve{t} \leq \ve{p}, \ve{s} \in  \Delta_n} -\ve{1}^\top \ve{t}  \quad \mbox{s.t. } D_{f} (\ve{t} + (1-\ve{1}^\top \ve{t})\cdot \ve{s} \|\ve{q}) \leq D. \label{def-pb-f-md-general-alt}
\end{eqnarray}
Mentored decoding's $\ve{r}$ in \eqref{def-pb-f-md-2-general} is obtained as $\ve{r} \defeq \ve{t} \oslash \ve{p}$.\\

\noindent \textbf{($\textsc{md}'_f\rightarrow\textsc{is}_f$)} We have the Lagrangian,
\begin{eqnarray}
\mathcal{L}_1(\ve{t}, \ve{s} ; \mu, \ve{\chi}, \ve{\nu}) & = & -\ve{1}^\top \ve{t} + \lambda \cdot( D_{f} (\ve{t} + (1-\ve{1}^\top \ve{t})\cdot \ve{s}  \|\ve{q}) - D) + \mu\cdot(\ve{1}^\top \ve{s} - 1) \nonumber\\
 & & + \ve{\chi}^\top -\ve{s} + \ve{\nu}^\top (\ve{t} - \ve{p}). \label{defLag}
\end{eqnarray}
We have KKT the conditions (using notations from \eqref{def-pb-f-md-2-general} and \eqref{defLag})
  \begin{eqnarray}
    \ve{t} & \leq & \ve{p}, \label{kkt1}\\
    \ve{s} & \geq & \ve{0}, \label{kkt2}\\
    \ve{1}^\top \ve{s} & = & 1, \label{kkt3}\\
    \ve{\chi} & \geq & \ve{0}, \label{kkt4}\\
    \ve{\nu} & \geq & \ve{0}, \label{kkt5}\\
    \ve{\chi} \odot \ve{s} & = & \ve{0}, \label{kkt6}\\
    \ve{\nu} \odot (\ve{t} - \ve{p}) & = & \ve{0}, \label{kkt7}\\
    \nabla_{\ve{t}} \mathcal{L}_1 = \nabla_{\ve{s}} \mathcal{L}_1 & = & \ve{0}, \label{kkt8}\\
    \lambda & \geq & 0, \label{kkt9}\\
    \lambda \cdot( D_{f} (\ve{t} + (1-\ve{1}^\top \ve{t})\cdot \ve{s}  \|\ve{q}) - D) & = & 0, \label{kkt85}\\
    D_{f} (\ve{t} + (1-\ve{1}^\top \ve{t})\cdot \ve{s}  \|\ve{q}) & \leq & D.  \label{kkt10}
  \end{eqnarray}
Denote for short 
\begin{eqnarray}
\pi_i & \defeq & t_i + (1-\ve{1}^\top \ve{t}) \cdot s_i. \label{defPI1}
\end{eqnarray}
\eqref{kkt8} is equivalent to:
\begin{eqnarray}
  \frac{\partial \mathcal{L}_1}{\partial t_i} = \lambda \cdot \sum_j (-f') \left(\frac{\pi_j}{q_j}\right)\cdot s_j - \lambda \cdot (-f') \left(\frac{\pi_i}{q_i}\right)  - 1 +\nu_i & = & 0, \forall i, \label{eq-stat1}\\
  \frac{\partial \mathcal{L}_1}{\partial s_i} = \mu - \lambda \cdot (-f') \left(\frac{\pi_i}{q_i}\right) \cdot (1-\ve{1}^\top \ve{t}) - \chi_i & = & 0, \forall i. \label{eq-stat2}
\end{eqnarray}
We have a first Lemma.
\begin{lemma}\label{lemR1}
If $D < D_{f}(\ve{p} \| \ve{q})$ then $\ve{t} \neq \ve{p}$ and $\lambda > 0$ at the optimum.
\end{lemma}
\begin{proof}
Proof immediate for $\ve{t} \neq \ve{p}$ because for $\ve{t} = \ve{p}$, $D_{f}(\ve{\pi} \| \ve{q}) = D_{f}(\ve{p} \| \ve{q}) > D$, not feasible in this case. If $\lambda = 0$, we get from \eqref{eq-stat1} $\nu_i = 1 \neq 0, \forall i$ and thus complementary slackness \eqref{kkt7} imposes $\ve{t} = \ve{p}$, impossible since $D < D_{f}(\ve{p} \| \ve{q})$.
\end{proof}
We can thus reorganize \eqref{eq-stat1} and \eqref{eq-stat2} with the complementary slackness conditions \eqref{kkt6}, \eqref{kkt7} to give ($\iver{.}$ is Iverson's bracket):
\begin{eqnarray}
(-f') \left(\frac{\pi_i}{q_i}\right) & = &  \alpha + \underbrace{\frac{\nu_i}{\lambda}\cdot \color{red}{\iver{t_i = p_i}}}_{\geq 0}, \forall i, \quad \mbox{with } \alpha \defeq \expect_{i \sim \ve{s}} \left[(-f') \left(\frac{\pi_i}{q_i}\right)\right] - \frac{1}{\lambda}.\label{eq22bis}\\
   (-f') \left(\frac{\pi_i}{q_i}\right) & = & \beta - \underbrace{\frac{ \chi_i }{\lambda \cdot (1 - \ve{1}^\top \ve{t})}\cdot \color{red}{\iver{s_i = 0}}}_{\geq 0}, \forall i, \quad \mbox{with } \beta \defeq \frac{\mu}{\lambda \cdot (1 -\ve{1}^\top \ve{t})},\label{eq11bis}
\end{eqnarray}

\begin{lemma}\label{lemAB}
 At the optimum, $\alpha = \beta - (1/\lambda)$; hence $\alpha < \beta$. Furthermore, $\alpha$ and $\beta$ also satisfy:
\begin{eqnarray}
\alpha & = & \expect_{i \sim \ve{u}} \left[(-f') \left(\frac{\pi_i}{q_i}\right)\right], \quad\mbox{ with }\ve{u} \defeq \frac{1}{1-\ve{1}^\top\ve{t}} \cdot (\ve{p}-\ve{t}) \in \Delta_n.\label{mdisAlphaL}\\
\beta & = & \expect_{i \sim \ve{s}} \left[(-f') \left(\frac{\pi_i}{q_i}\right)\right]\label{mdisBetaL}
\end{eqnarray}
 \end{lemma}
\begin{proof}
Sum \eqref{eq11bis} times $s_i$ and we get
\begin{eqnarray*}
\sum_i (-f') \left(\frac{\pi_i}{q_i}\right) \cdot s_i & = & \beta \cdot \underbrace{\sum_i s_i}_{=1} - \sum_i \frac{ \chi_i }{\lambda \cdot (1 - \ve{1}^\top \ve{t})} \cdot \underbrace{{\color{red}{\iver{s_i = 0} \cdot s_i}}}_{=0, \forall i} = \beta\label{alterbetabis},
\end{eqnarray*}
and we reorganize using \eqref{eq22bis} to get 
\begin{eqnarray*}
\beta & = & \expect_{i \sim \ve{s}} \left[(-f') \left(\frac{\pi_i}{q_i}\right)\right],
\end{eqnarray*}
and  $\alpha = \beta - (1/\lambda)$ because of the definition of $\alpha$ in \eqref{eq22bis}. Now, sum \eqref{eq22bis} times $p_i - t_i$ and we get
\begin{eqnarray*}
\sum_i (-f') \left(\frac{\pi_i}{q_i}\right) \cdot (p_i - t_i) & = & \alpha \cdot \underbrace{\sum_i p_i - t_i}_{=1 - \ve{1}^\top \ve{t}} - \sum_i \frac{\nu_i}{\lambda}\cdot \underbrace{\color{red}{\iver{t_i = p_i} \cdot (t_i - p_i)}}_{=0, \forall i}\label{alterbetabis2},
\end{eqnarray*}
and rearrange to find the expression of $\alpha$ in \eqref{mdisAlphaL}.
\end{proof}

Pick any $i$ such that $t_i < p_i$. \eqref{eq22bis} yields $(-f') \left(\frac{\pi_i}{q_i}\right) = \alpha$ and so, in any optimal solution,
\begin{eqnarray}
t_i < p_i & \Rightarrow & \pi_i \in q_i \cdot L_\alpha(-f').\label{eqdec}
\end{eqnarray}
Pick any $i$ such that $s_i>0$. \eqref{eq11bis} yields $(-f') \left(\frac{\pi_i}{q_i}\right) = \beta$ and so, in any optimal solution,
\begin{eqnarray}
s_i > 0 & \Rightarrow & \pi_i \in q_i \cdot L_\beta(-f').\label{eqinc}
\end{eqnarray}
All other cases must meet $t_i = p_i$ and $s_i = 0$, hence $\pi_i = p_i$. Since $\alpha < \beta$, we must have $L_\alpha(-f') > L_\beta(-f')$ ($f$ is convex), so it is impossible that $1 > L_\alpha(-f')$ or $L_\beta(-f') > 1$ otherwise $\ve{\pi}$ would not be a distribution. We thus have simultaneously
\begin{eqnarray}
  L_\beta(-f') & < & L_\alpha(-f'),\label{lalphabeta1}\\
  1 & \leq & \max L_\alpha(-f'), \label{lalphabeta2}\\
  \min L_\beta(-f') & \leq & 1. \label{lalphabeta3}
\end{eqnarray}
We go back to \eqref{eqdec}: for any $i$ such that $t_i < p_i$ and since $\alpha < \beta$ \eqref{eq11bis} yields that for all these indices $s_i = 0$ so $\pi_i = t_i < p_i$. Since otherwise $t_i = p_i$, we get that in all cases, 
\begin{eqnarray}
t_i & \in & \mins(p_i, q_i \cdot L_\alpha(-f')) = \{p_i\} + \mins(0, q_i \cdot L_\alpha(-f') - p_i) \label{eqfinti}
\end{eqnarray}
and while this guarantees $\ve{t}  \preceq \ve{p}$, we must also ensure $\ve{t}  \neq \ve{p}$ (Lemma \ref{lemR1}). We separately note that
\begin{eqnarray}
\min L_\alpha(-f') & < & \max_j p_j / q_j \label{lalphabeta4}
\end{eqnarray}
otherwise the only feasible solution is $\ve{t} = \ve{p} = \ve{\pi}$, impossible (Lemma \ref{lemR1}). 

We go back to \eqref{eqinc}: for any $i$ such that $s_i>0$ and since $\alpha < \beta$ \eqref{eq22bis} yields that for all these indices $t_i = p_i$. Since otherwise $s_i = 0$, we get from $\pi_i \defeq t_i + s_i (1-\ve{1}^\top \ve{t})$:
\begin{eqnarray}
s_i(1-\ve{1}^\top \ve{t}) & \in & \maxs(0, q_i \cdot L_\beta(-f') - p_i).  \label{eqfinsi}
\end{eqnarray}
We separately note that
\begin{eqnarray}
\max L_\beta(-f') & > & \min_j p_j / q_j \label{lalphabeta5}
\end{eqnarray}
otherwise $\ve{s} = \ve{0}$, not admissible. From \eqref{lalphabeta1}, we get $q_i \cdot L_\beta(-f') - p_i <  q_i \cdot L_\alpha(-f') - p_i$ and so we get the final expression for $\ve{\pi}$ from its definition and \eqref{eqfinti}, \eqref{eqfinsi}:
\begin{eqnarray*}
  \pi_i & \in & \{p_i\} + \mins(0, q_i \cdot L_\alpha(-f') - p_i) + \maxs(0, q_i \cdot L_\beta(-f') - p_i)\\
           & & = \clamps(p_i, q_i \cdot L_\beta(-f'), q_i \cdot L_\alpha(-f')),                   
\end{eqnarray*}
where the equality comes by definition of $\clamps$. Thus, all optimal solutions satisfy
\begin{eqnarray}
  \ve{\pi} & \in & \clamps(\ve{p}, L_\beta(-f') \cdot \ve{q}, \cdot L_\alpha(-f') \cdot \ve{q}) \cap \Delta_n, \label{defPICHI}
  \end{eqnarray}
At this stage, we have explicitly satisfied KKT conditions \eqref{kkt1}, \eqref{kkt2}, \eqref{kkt3}, \eqref{kkt8}, \eqref{kkt9}.

To check \eqref{kkt4} and \eqref{kkt6}, we compute
\begin{eqnarray}
\chi_i & = & \left\{
\begin{array}{ccl}
\lambda \cdot (1-\ve{1}^\top \ve{t}) \cdot \left( \beta - (-f') \left(\frac{\pi_i}{q_i}\right)\right) & \mbox{ if } & t_i < p_i \vee \pi_i = p_i\\
0 & \multicolumn{2}{l}{\mbox{ otherwise }}
\end{array}
\right. . \label{defCHII}
\end{eqnarray}
Note that the first condition is equivalent to $\pi_i \leq p_i$. 

Suppose first that $t_i < p_i$. Then we know that $s_i = 0$ so that \eqref{defCHII} is in fact \eqref{eq11bis}. In this case, \eqref{eqdec} yields the second identity in
\begin{eqnarray}
\chi_i & = & \lambda \cdot (1-\ve{1}^\top \ve{t}) \cdot \left( \beta - (-f') \left(\frac{\pi_i}{q_i}\right)\right) \label{chi1}\\
& = & \lambda \cdot (1-\ve{1}^\top \ve{t}) \cdot \left( \beta - (-f') \left(L_\alpha(-f')\right)\right)\nonumber\\
& = & \lambda \cdot (1-\ve{1}^\top \ve{t}) \cdot \left( \beta - \alpha\right)\nonumber\\
& = & 1-\ve{1}^\top \ve{t} \geq 0.\nonumber
\end{eqnarray}
the penultimate identity comes from the definition of $L_\alpha(-f')$ and the last one is Lemma \ref{lemAB}. So $t_i < p_i$ implies $\chi_i \geq 0$ and $\chi_i s_i = 0$.

Suppose now that $\pi_i = p_i (=t_i)$. In this case, we have
\begin{eqnarray*}
(-f') \left(\frac{\pi_i}{q_i}\right) = (-f') \left(\frac{p_i}{q_i}\right) \in [\alpha, \beta],
\end{eqnarray*}
Hence from \eqref{chi1} $\chi_i \geq 0$ again while, since we also have $s_i = 0$, we have $\chi_i s_i = 0$.

Finally, if $\neg(t_i < p_i \vee \pi_i = p_i) \equiv \pi_i > p_i$, we have $s_i > 0$ but $\chi_i s_i = 0$ and still $\chi_i \geq 0$.

Hence, KKT \eqref{kkt4} and \eqref{kkt6} are satisfied.

We now check KKT \eqref{kkt5} and \eqref{kkt7} and for that we let:
\begin{eqnarray*}
\nu_i & = & \left\{
\begin{array}{ccl}
\lambda \cdot \left((-f') \left(\frac{\pi_i}{q_i}\right) - \alpha \right) & \mbox{ if } & \pi_i \geq p_i\\
0 & \multicolumn{2}{l}{\mbox{ otherwise }}
\end{array}
\right. . \label{defNUI}
\end{eqnarray*}
In the topmost case, it comes that because of \eqref{defPICHI} and $f$ is convex (thus $-f'$ is non-increasing), we always have $(-f') \left(\pi_i/q_i\right) \geq \alpha$ so $\ve{\nu} \ge \ve{0}$, which is \eqref{kkt5}. We also know that $t_i < p_i \Rightarrow \pi_i < p_i$, hence $\pi_i \geq p_i$ implies $t_i = p_i$ and KKT \eqref{kkt7} is satisfied.

At this stage, we have checked KKT conditions \eqref{kkt1}, \eqref{kkt2}, \eqref{kkt3}, \eqref{kkt4}, \eqref{kkt5}, \eqref{kkt6}, \eqref{kkt7}, \eqref{kkt8}, \eqref{kkt9}. To get optimality, we only need the last KKT conditions \eqref{kkt85} and \eqref{kkt10} to be satisfied and since $\lambda > 0$, this implies 
\begin{eqnarray*}
D_{f} (\ve{\pi} \|\ve{q}) & = & D.
\end{eqnarray*}
Hence, via the construction \eqref{eqfinti} and \ref{eqfinsi} and $\ve{\pi} \defeq \ve{t} + (1-\ve{1}^\top \ve{t})\cdot \ve{s}$, any optimal solution satisfies
\begin{eqnarray}
\left\{
\begin{array}{l}
\ve{\pi} \in \clamps(\ve{p}, L_\beta(-f') \cdot \ve{q}, L_\alpha(-f') \cdot \ve{q}) \cap \Delta_n\\
D_{f} (\ve{\pi} \|\ve{q}) = D
\end{array}
\right. . \label{eqCONDPI}
\end{eqnarray}

This ends the proof of ($\textsc{md}'_f\rightarrow\textsc{is}_f$).\\

\noindent \textbf{($\textsc{is}_f\rightarrow\textsc{md}'_f$)} Pick any $\ve{\pi}$ satisfying \eqref{eqCONDPI} for $\alpha < \beta \in \mathrm{Im}(-f')$. We craft
\begin{eqnarray}
\ve{t} & \defeq & \min\{\ve{p}, \ve{\pi}\}, \label{eqfintivec}\\
\ve{s} & \defeq & (1-\ve{1}^\top \ve{t})^{-1} \cdot \max\{\ve{0}, \ve{\pi}- \ve{p}\}
\end{eqnarray}
(Because of Assumption \ref{assum-mda}, we have $\ve{p}^\top\ve{r} < 1$, so there is one choice only for $\ve{s}$ as per \eqref{eq-r-s-from-pi}). Lemma \ref{lemR1} implies $\ve{1}^\top \ve{t} < 1$ so $\ve{s}$ is positive and finite, and $1-\ve{1}^\top \ve{t} = \sum_i \pi_i - \min\{\pi_i, p_i\} = \sum_i \max\{0, \pi_i - p_i\}$, so $\ve{s} \in \Delta_n$.

At this point, we easily check KKT \eqref{kkt1}, \eqref{kkt2}, \eqref{kkt3}. Furthermore, we observe from \eqref{pMMC3}
\begin{eqnarray*}
\{\min\{\ve{p}, \ve{\pi}\}: \ve{\pi} \in \clamps(\ve{p}, L_\beta(-f') \cdot \ve{q}, L_\alpha(-f') \cdot \ve{q})\} & = & \mins(\ve{p}, L_\alpha(-f') \cdot \ve{q}),
\end{eqnarray*}
which is just the way we built $\ve{t}$ in \eqref{eqfinti}. Also, 
\begin{eqnarray*}
\{\max\{\ve{0}, \ve{\pi}- \ve{p}\}: \ve{\pi} \in \clamps(\ve{p}, L_\beta(-f') \cdot \ve{q}, L_\alpha(-f') \cdot \ve{q})\} & = & \maxs(\ve{0}, L_\beta(-f') \cdot \ve{q}  - \ve{p})
\end{eqnarray*}
from \eqref{pMMC4} which is just the way we built $\ve{s}$ from \eqref{eqfinsi}. So the way we craft $\ve{t}$ and $\ve{s}$ is the same as for the first step and ends the proof of Theorem \ref{th-opt1}.

\begin{remark}
Note that Lemma \ref{lemLAlphaBeta} follows from \eqref{lalphabeta1}, \eqref{lalphabeta2}, \eqref{lalphabeta3}, \eqref{lalphabeta4}, \eqref{lalphabeta5}.
\end{remark}

\subsection{Proof of Lemma \ref{lemPOST}}\label{sec-proof-lemPOST}

Since $p_i>0$, $r_i p_i \leq 0$ would imply $0 \in L_{\alpha}(-f')$ and then $\alpha = \beta$ since we would be forced to also have $0 \in L_{\beta}(-f')$ \eqref{mdisBetaL}, which is impossible (Lemma \ref{lemAB}).

\subsection{Proof of Theorem  \ref{thmAPPROX1}}\label{sec-proof-thmAPPROX1}

We show that all KKT conditions of \eqref{def-pb-f-md-2-general} are satisfied except eventually one, the $f$-divergence constraint \eqref{kkt85} and then compute bounds for the corresponding $D$, so we start with the assumption that $f$ is differentiable and then remove it. We reuse the KKT conditions in \eqref{kkt1}, \eqref{kkt2}, \eqref{kkt3}, \eqref{kkt4}, \eqref{kkt5}, \eqref{kkt6}, \eqref{kkt7}, \eqref{kkt8}, \eqref{kkt9}, \eqref{kkt85}. We consider $\ve{\pi}$ defined by
\begin{eqnarray}
 \pi_i \defeq p_i \mbox{ if } i \in \mathbb{I}, \mbox{ else } (1+a) q_i \mbox{ if } i \in \mathbb{A}, \mbox{ else } (1-b) q_i \mbox{ if } i \in \mathbb{B} \label{defPI}.
\end{eqnarray}
We then have $\ve{\pi} = \ve{t} + (1-\ve{1}^\top \ve{t})\cdot \ve{s}$, letting $\ve{t} \defeq \ve{p} \odot \ve{r}$, for the choices:
\begin{eqnarray}
  t_i & \defeq & p_i \mbox{ if } i \in \mathbb{B} \cup \mathbb{I} \mbox{ else } t_i \defeq (1+a) q_i \mbox{ if } i \in \mathbb{A}, \label{defTI}\\
  s_i & \defeq & 0 \mbox{ if } i \in \mathbb{A} \cup \mathbb{I} \mbox{ else } s_i \defeq (1-\ve{1}^\top \ve{t})^{-1} \cdot ((1-b) q_i - p_i) \mbox{ if } i \in \mathbb{B}. \label{defSI}
\end{eqnarray}
(if $b \geq 1 - \min_i p_i/q_i$, we pick any distribution $\ve{s} \in \Delta_n$ since $\mathbb{B} = \emptyset$). KKT \eqref{kkt1} holds because all indexes in $\mathbb{A}$ satisfiy $p_i / q_i > 1+a$. KKT \eqref{kkt2} is satisfied because all indexes in $\mathbb{B}$ satisfy $p_i / q_i < 1-b$. We note from the taxonomy \eqref{defAbis}, \eqref{defBbis}, \eqref{defIbis},
\begin{eqnarray}
  \ve{1}^\top \ve{t} & = & (1+a) q(\mathbb{A}) + p(\mathbb{B}) + p(\mathbb{I}) = 1 - (p(\mathbb{A})  - (1+a) q(\mathbb{A})), \label{optcost}
\end{eqnarray}
and \eqref{kktconst3} also yields
\begin{eqnarray}
\ve{1}^\top \ve{t} = 1 - ( (1-b) q(\mathbb{B}) - p(\mathbb{B}) ).  \label{optcost2}
\end{eqnarray}
We also have from the taxonomy $(1-b)q(\mathbb{B}) = \sum_{i \in \mathbb{B}} \pi_i = \sum_{i \in \mathbb{B}} p_i + (1-\ve{1}^\top \ve{t}) \cdot s_i = p(\mathbb{B}) + (1-\ve{1}^\top \ve{t}) \cdot \ve{1}^\top \ve{s}$ which, via identification with \eqref{optcost2}, yields  $\ve{1}^\top \ve{s} = 1$ because $ (1-b) q(\mathbb{B}) - p(\mathbb{B}) > 0$ from the definition of $\mathbb{B}$. Hence KKT \eqref{kkt3} holds. We now compute $\ve{\chi}$ and $\ve{\nu}$ as
\begin{eqnarray}
  \chi_i & \defeq & \lambda (1-\ve{1}^\top{t}) \cdot \left( (a+b) \mbox{ if } i \in \mathbb{A}, \mbox{ else } \left(b - 1 + \frac{p_i}{q_i}\right) \mbox{ if } i \in \mathbb{I}, \mbox{ else } 0 \mbox{ if } i \in \mathbb{B} \right),\label{defCHI} \\
  \nu_i & \defeq & \lambda \cdot \left( (a+b) \mbox{ if } i \in \mathbb{B}, \mbox{ else } \left(1 - \frac{p_i}{q_i} + a\right) \mbox{ if } i \in \mathbb{I}, \mbox{ else } 0 \mbox{ if } i \in \mathbb{A} \right). \label{defNU}
\end{eqnarray}
We easily check KKT \eqref{kkt4} and \eqref{kkt5} ($a, b \geq 0$ and the taxonomy \eqref{defIbis}). \eqref{kkt6} is checked from \eqref{defSI} and \eqref{defCHI}, \eqref{kkt7} is checked from \eqref{defTI} and \eqref{defNU}, and finally \eqref{kkt8} is just \eqref{eq11bis} and \eqref{eq22bis}. We finally let $\alpha \defeq (-f)'(1+a), \beta \defeq (-f')(1-b)$ and
\begin{eqnarray}
\lambda & \defeq & \frac{1}{\beta - \alpha} > 0, 
\end{eqnarray}
so \eqref{kkt9} is satisfied; there is only \eqref{kkt85} which is eventually not satisfied. Note that if $f$ is not differentiable, we just switch to $\alpha \defeq (-g)(1+a), \beta \defeq (-h)(1-b)$ for $g, h \in \partial f$.

Hence, \textit{any} values $a,b$ as in \eqref{kktconst1}, \eqref{kktconst2} define the optimum of \eqref{def-pb-f-md-2-general} for \textit{some} $\tilde{D}$ that we can compute:
\begin{eqnarray*}
  \tilde{D} = D_f(\ve{\pi} \| \ve{q}) & = & \sum_{\mathbb{A}} f \left(\frac{\pi_i}{q_i}\right) \cdot q_i  + \sum_{\mathbb{B}} f \left(\frac{\pi_i}{q_i}\right) \cdot q_i + \underbrace{\sum_{\mathbb{I}} f \left(\frac{\pi_i}{q_i}\right) \cdot q_i }_{\defeq D_f(\mathbb{I})}\\
  & = & q(\mathbb{A}) \cdot f(1+a)  + q(\mathbb{B}) \cdot f(1-b)  + D_f(\mathbb{I}).
\end{eqnarray*}
Because of the definition of $\mathbb{I}$ in \eqref{defBbis}, $\ve{\pi}$ in \eqref{defPI} and the fact that $f$ is convex (therefore continuous), we can upperbound $D_f(\mathbb{I})$ as
\begin{eqnarray*}
D_f(\mathbb{I}) & = & \sum_{i \in \mathbb{I}} f \left(\frac{\pi_i}{q_i}\right) \cdot q_i \\
& = & \sum_{i \in \mathbb{I}} f \left(\frac{p_i}{q_i}\right) \cdot q_i \\
& = & q(\mathbb{I}) \cdot \sum_{i \in \mathbb{I}} \frac{q_i}{q(\mathbb{I})} \cdot f \left(\frac{p_i}{q_i}\right) = q(\mathbb{I}) \cdot \expect_{i \sim \ve{q}_{|\mathbb{I}}}\left[ f\left(\frac{p_i}{q_i}\right)\right]\\
& = & q(\mathbb{I}) \cdot f(u)
\end{eqnarray*}
for some $u\in [1-b, 1+a]$ (here, $\ve{q}_{|\mathbb{I}}$ is $\ve{q}$ restricted to set $\mathbb{I}$). Summarizing, 
\begin{eqnarray*}
\exists u\in [1-b, 1+a]: \tilde{D} & = & q(\mathbb{A}) \cdot f(1+a)  + q(\mathbb{B}) \cdot f(1-b) + q(\mathbb{I}) \cdot f(u).
\end{eqnarray*}
We finally compute the acceptance probability as
    \begin{eqnarray*}
\pacc(MD) & \defeq & \ve{1}^\top \ve{t}\\
& = & \sum_i \min\{p_i, (1+a) q_i\}\\
& =& \pacc(SD) + a q(\mathbb{A}) + (p(\mathbb{I}_{>1})-q(\mathbb{I}_{>1})),
    \end{eqnarray*}
    where we have let
    \begin{eqnarray}
\mathbb{I}_{>1} \defeq \left\{i : p_i\in q_i \cdot (1, 1+a]\right\},
    \end{eqnarray}
    which ends the proof of Theorem \ref{thmAPPROX1}.

\subsection{Proof of Theorem \ref{thmUniversal}}\label{sec-proof-thmUniversal}

Let $\tilde{\textsc{is}}_f(\ve{p}, \ve{q}; D) \subseteq \textsc{is}_f(\ve{p}, \ve{q}; D)$ be defined as:
\begin{eqnarray}
\lefteqn{\tilde{\textsc{is}}_f(\ve{p}, \ve{q}; D)}\nonumber\\
& \defeq & \left\{\ve{\pi} \in \Delta_n: \left\{
\begin{array}{l}
\exists L_\beta(-h) \leq 1 \leq L_\alpha(-g): \ve{\pi} \in\clamps(\ve{p}, L_\beta(-h) \cdot \ve{q}, L_\alpha(-g) \cdot \ve{q})\\
D_{f} (\ve{\pi} \|\ve{q}) = D
\end{array}
\right.\right\}, \label{isABTILDE}
  \end{eqnarray}
with the additional constraint $\alpha < \beta$ (we remind $g, h \in \partial f$, \eqref{isAB-nonconvex}). Using the definition of level sets and $\clamps$, we get that in this case and coordinate-wise,
\begin{eqnarray*}
\lefteqn{\clamps(p_i, L_\beta(-h) q_i, L_\alpha(-g) q_i)}\\
& = & 
\left\{
\begin{array}{ccll}
L_\beta(-h) q_i & \mbox{ if } & p_i < L_\beta(-h) q_i & (I)\\
\{z \in L_\beta(-h) q_i : z > p_i\} & \mbox{ if } & p_i \in L_\beta(-h) q_i \wedge p_i < \max L_\beta(-h) q_i & (II)\\
p_i & \mbox{ if } & p_i \in [\max L_\beta(-h) q_i, \min L_\alpha(-g) q_i] & (III)\\
\{z \in L_\alpha(-g) q_i : z < p_i\} & \mbox{ if } & p_i \in L_\alpha(-g) q_i \wedge p_i > \min L_\alpha(-g) q_i & (IV)\\
L_\alpha(-g) q_i & \mbox{ if } & p_i > L_\alpha(-g) q_i & (V)
\end{array}
\right.
  \end{eqnarray*}
We analyze case by case, noting that $L_\beta(-h) \leq 1$ implies $L_\beta(-h) q_i \leq q_i$, and $L_\alpha(-g) \geq 1$ implies $L_\alpha(-g) q_i \geq q_i$:
\begin{itemize}
\item [Case (I)] Here, $L_\beta(-h) q_i \subseteq [p_i, q_i] = [\min\{p_i, q_i\}, \max\{p_i, q_i\}]$;
\item [Case (II)] is a subset of (I) still with $p_i < q_i$;
\item [Case (III)] in this case, $[\min\{p_i, q_i\}, \max\{p_i, q_i\}] = [q_i, p_i]$ and we clearly have $p_i \in [q_i, p_i]$;
\item [Case (IV)] we observe $q_i < \{z \in L_\alpha(-g) q_i : z < p_i\} < p_i$ so $\{z \in L_\alpha(-g) q_i : z < p_i\} \subset [\min\{p_i, q_i\}, \max\{p_i, q_i\}]$;
\item [Case (V)] we observe again $q_i < L_\alpha(-g) q_i < p_i$, so same conclusion as in (IV).
\end{itemize}
  To summarize, we have shown that there exists $D'$ such that
  \begin{eqnarray*}
\tilde{\textsc{is}}_f(\ve{p}, \ve{q}; D) & \subseteq & \textsc{is}_{f_{\mathrm{TV}}}(\ve{p}, \ve{q}; D'),
  \end{eqnarray*}
Now denote $\tilde{\textsc{md}}^2_f(\ve{p}, \ve{q}; D)$ the set of optimal solutions in $\textsc{md}^2_f(\ve{p}, \ve{q}; D)$ having $L_\beta(-h) \leq 1 \leq L_\alpha(-g)$ with $\alpha, \beta$ in \eqref{mdisAlpha}, \eqref{mdisBeta}. What we have shown above make the following mappings connections, also using the bijection of Theorem \ref{th-opt1}:
\begin{eqnarray*}
\tilde{\textsc{is}}_f(\ve{p}, \ve{q}; D) \rightarrow \textsc{is}_{f_{\mathrm{TV}}}(\ve{p}, \ve{q}; D') \rightarrow \textsc{md}^2_{f_{\mathrm{TV}}}(\ve{p}, \ve{q}; D'),
\end{eqnarray*}
which shows the first part of Theorem \ref{thmUniversal}.\\

The second part and \eqref{eq-opt-f-g-md} follows from the proof of Theorem \ref{thmAPPROX1} and the fact that (i) for strictly convex generators, level sets $L_\beta, L_\alpha$ are singletons and (ii) the set of couples $\mathcal{C}(\ve{p}, \ve{q})$ in Definition \ref{def-cab} define optimal solution irrespectively of the generator.

\subsection{Proof of Theorem \ref{th-boost-P}}\label{sec-proof-th-boost-P}

We first need two technical Lemmata.
\begin{lemma}\label{lem-Ratio}
  For any $x_1, x_2, ... x_T \geq 0$ and any real $k \geq 1$, it holds that
  \begin{eqnarray*}
\frac{\sum_{i\in [T]} x_i^2}{\sum_{i\in [T]} x_i} & \leq & \frac{\sum_{i\in [T]} x_i^{2k}}{\sum_{i\in [T]} x_i^{2k-1}}.
    \end{eqnarray*}
  \end{lemma}
  \begin{proof}
    We note that this is equivalent to showing
    \begin{eqnarray}
\underbrace{\left(\sum_{i\in [T]} x_i\right) \cdot \left( \sum_{i\in [T]} x_i^{2k} \right)  - \left(\sum_{i\in [T]} x_i^2 \right) \cdot \left(\sum_{i\in [T]} x_i^{2k-1}\right)}_{\defeq A} & \geq & 0.
    \end{eqnarray}
    We develop the sums in $A$ in two equivalent forms (swapping indexes):
    \begin{eqnarray*}
      A = \sum_{i\in [T]} \sum_{j \in [T]} x_i x_j^{2k} - x_i^2 x_j^{2k-1} = \sum_{i\in [T]} \sum_{j \in [T]} x_j x_i^{2k} - x_j^2 x_i^{2k-1}.
    \end{eqnarray*}
    We then write $A$ as the arithmetic average of both expressions and factor:
     \begin{eqnarray}
       2 A & = & \sum_{i\in [T]} \sum_{j \in [T]} x_i x_j^{2k} - x_i^2 x_j^{2k-1} +  x_j x_i^{2k} - x_j^2 x_i^{2k-1}\nonumber \\
           & = & \sum_{i\in [T]} \sum_{j \in [T]} x_ix_j \cdot (\underbrace{x_j^{2k-1} - x_i x_j^{2k-2}}_{(x_j - x_i)x_j^{2k-2}} + \underbrace{x_i^{2k-1} - x_j x_i^{2k-2}}_{- (x_j - x_i)x_i^{2k-2}}) \nonumber \\
       & = &  \sum_{i\in [T]} \sum_{j \in [T]} x_ix_j \cdot (x_j - x_i) \cdot (x_j^{2k-2} - x_i^{2k-2}). \label{eq-Alast}
    \end{eqnarray}
    Since $z \mapsto z^{u}$ is strictly increasing for $z\geq 0$ and $u>0$, we get that for $k\geq 1$ we always have $ (x_j - x_i) \cdot (x_j^{2k-2} - x_i^{2k-2}) \geq 0$ for any $x_i, x_j \geq 0$ while $ x_ix_j \geq 0$; hence, all terms in \eqref{eq-Alast} are $\geq 0$, so $A \geq 0$ and the Lemma is proven.
    \end{proof}

\begin{lemma}\label{lem-bsup}
  for any $x_1, x_2, ... x_T \in [0,1)$ and any
  \begin{eqnarray}
y & \in & \left[0, \frac{1}{3} \cdot  \frac{\sum_{i\in [T]} x_i^2}{\sum_{i\in [T]} x_i}\right],\label{eq-int-y}
    \end{eqnarray}
it holds that
  \begin{eqnarray*}
\left(\prod_{i\in [T]} (1+x_i)\right)^{1+y} \cdot \left(\prod_{i\in [T]} (1-x_i)\right)^{1-y} & \leq & \exp\left(-y \cdot \sum_{i\in [T]} x_i\right).
    \end{eqnarray*}
\end{lemma}
\begin{proof}
    Take the logs and reorganize: we want equivalently
    \begin{eqnarray*}
      y \cdot \sum_{i\in [T]} x_i  + \sum_{i\in [T]} \log(1-x_i^2) + y\cdot \sum_{i\in [T]} \log\left(\frac{1+x_i}{1-x_i}\right) & \leq & 0.
    \end{eqnarray*}
Since $|x_i| < 1, \forall i$, we consider the (convergent) Taylor-MacLaurin series $\log(1-x^2) = \sum_{k\geq 1} (-x^{2k}/k)$ and $\log((1+x)/(1-x)) = \sum_{k\geq 1} (2/(2k-1)) \cdot x^{2k-1}$ and plug them in the desired inequality and isolating the term for $k=1$:
\begin{eqnarray*}
      \underbrace{y \cdot \sum_{i\in [T]} x_i -\sum_{i\in [T]} x_i^2 + 2y \sum_{i\in [T]} x_i}_{\defeq A} + \sum_{k\geq 2} \underbrace{\sum_{i\in [T]} \left(\frac{2}{2k-1} \cdot y x_i^{2k-1} - \frac{1}{k} \cdot x_i^{2k}\right)}_{\defeq B_k}& \leq & 0.
    \end{eqnarray*}
    We now show that all of $A$ and $B_k, k\geq 2$ are non-positive. First, we factor $A$:
    \begin{eqnarray*}
      A & = & \left( 3y - \frac{\sum_{i\in [T]} x_i^2}{\sum_{i\in [T]} x_i} \right) \cdot \sum_{i\in [T]} x_i,
    \end{eqnarray*}
    so $A \leq 0$ iff
    \begin{eqnarray}
y & \leq & \frac{1}{3} \cdot  \frac{\sum_{i\in [T]} x_i^2}{\sum_{i\in [T]} x_i} .\label{eq-const-A}
    \end{eqnarray}
    and to have each $B_k\leq 0$ we must observe equivalently
    \begin{eqnarray*}
      y & \leq & \frac{2k-1}{2k} \cdot \frac{\sum_{i\in [T]} x_i^{2k}}{\sum_{i\in [T]} x_i^{2k-1}}, \forall k \geq 2,
    \end{eqnarray*}
    and from Lemma \ref{lem-Ratio} and the fact that $k \mapsto 1 - (1/(2k))$ is strictly increasing, it is sufficient to require
    \begin{eqnarray*}
      y & \leq & \frac{1}{2} \cdot  \frac{\sum_{i\in [T]} x_i^2}{\sum_{i\in [T]} x_i},
    \end{eqnarray*}
    and we observe that this is satisfied if \eqref{eq-const-A} holds, which brings the statement of the Lemma.
    \end{proof}
We now embark on the proof of Theorem \ref{th-boost-P}. We use the following inequality \citet[Lemma 2]{DBLP:journals/ai/NockN07}:
  \begin{eqnarray}
    1-ab & \geq & \sqrt{1-a^2} \cdot \exp\left(-\frac{b}{2}\cdot \log\left(\frac{1+a}{1-a}\right)\right), \forall a, b \in [-1,1]. \label{eqb1}
\end{eqnarray}
Consider prediction for $i$-th training example with associated next token vector $\ve{y}_{i} \in \mathcal{Y}$, fix $a\defeq \mu_t \in [-1,1]$ and 
\begin{eqnarray*}
  b & \defeq & \frac{\ve{y}_{i}^\top \ve{h}_t (\bm{x}_i) }{2 h_{t, \infty}}.
\end{eqnarray*}
We note that H{\"o}lder's inequality and the definition of $\mathcal{Y}$ imply $|\ve{y}_{i}^\top \ve{h}_t (\bm{x}_i)| \leq \|\ve{y}_{i}\|_1 \cdot h_{t, \infty}$ and $\|\ve{y}_{i}\|_1 = n_i * (1/n_i) + (n-n_i) * (1/(n-n_i)) = 2$, so $|b| \leq 1$. \eqref{eqb1} brings for these choices:
\begin{eqnarray}
  \lefteqn{1 - \frac{\mu_t}{2 h_{t, \infty}}\cdot \ve{y}_{i}^\top \ve{h}_t(\bm{x}_i)}\nonumber\\
  & \geq & \sqrt{1-\mu^2_t} \cdot \exp\left(-\frac{\ve{y}_{i}^\top \ve{h}_t(\bm{x}_i)}{4 h_{t, \infty}} \ln\left( \frac{1+\mu_t}{1-\mu_t}\right)\right)\nonumber\\
   & & = \sqrt{1-\mu^2_t} \cdot \exp\left(- \ve{y}_{i}^\top \left(c_t \cdot \ve{h}^\star_t (\bm{x}_i)\right)\right), \: c_t \defeq \frac{1}{4}\cdot \ln\left( \frac{1+\mu_t}{1-\mu_t}\right), \: \ve{h}^\star_t \defeq \frac{1}{h_{t, \infty}} \cdot \ve{h}_t\:\:.\label{ubound}
\end{eqnarray}
Unraveling the weight update rule, we also obtain:
\begin{eqnarray}
w_{T+1, i} \cdot \prod_{j=1}^{T} {(1 - \mu^2_{t})} & = & w_{1,i}\cdot \prod_{j=1}^{T} {\left(1 -\frac{\mu_t}{2 h_{t, \infty}}\cdot \ve{y}_{i}^\top \ve{h}_t(\bm{x}_i)\right)}\:\:. \label{onex}
\end{eqnarray}
Using $T$ times (\ref{ubound}) on the right-hand side of (\ref{onex}) and simplifying yields:
\begin{eqnarray}
w_{1,i} \cdot \exp\left( -\ve{y}_{i}^\top \sum_{t=1}^{T} c_t \cdot \ve{h}^\star_t (\bm{x}_i) \right) & \leq & w_{T+1, i} \cdot \prod_{t=1}^{T} {\sqrt{1 - \mu^2_{t}}} \:\:,\label{funeq0}
\end{eqnarray}
which we then sum for $i\in [m]$ and simplify ($\ve{w}_. \in \Delta_n$):
\begin{eqnarray}
  \sum_{i=1}^m w_{1,i} \cdot \exp\left( -\ve{y}_{i}^\top \sum_{t=1}^{T} c_t \cdot \ve{h}^\star_t (\bm{x}_i) \right) & \leq & \left(\sum_{i=1}^m w_{T+1, i}\right) \cdot \prod_{t=1}^{T} {\sqrt{1 - \mu^2_{t}}} \nonumber\\
   & & = \prod_{t=1}^{T} {\sqrt{1 - \mu^2_{t}}} \:\:,\label{funeq1}
\end{eqnarray}
Denote
\begin{eqnarray*}
  \ve{H}_T(\bm{x}) & \defeq & \frac{1}{\sum_{t=1}^{T} c_t} \cdot \sum_{t=1}^{T} c_t \cdot \ve{h}^\star_t (\bm{x}).
\end{eqnarray*}
We get from $\iver{q\leq 0} \leq \exp(-q), \forall q$ the first inequality and from \eqref{funeq1} the last inequality, of
 \begin{eqnarray}
   \sum_{i=1}^m w_{1,i} \cdot  \bigiver{\ve{y}_{i}^\top \ve{H}_T(\bm{x}_i) \leq \theta}  & = & \sum_{i=1}^m w_{1,i} \cdot  \bigiver{\ve{y}_{i}^\top  \sum_{t=1}^{T} c_t \cdot \ve{h}^\star_t (\bm{x}) - \theta \cdot \sum_{t=1}^{T} c_t \leq 0} \nonumber\\
                                                                                              & \leq &  \sum_{i=1}^m w_{1,i} \cdot \exp\left(-\ve{y}_{i}^\top  \sum_{t=1}^{T} c_t \cdot \ve{h}^\star_t (\bm{x}) + \theta \cdot \sum_{t=1}^{T} c_t  \right)\nonumber\\
                                                                                              & & = \exp\left(\theta \cdot \sum_{t=1}^{T} c_t \right) \cdot \sum_{i=1}^m w_{1,i} \cdot \exp\left(-\ve{y}_{i}^\top  \sum_{t=1}^{T} c_t \cdot \ve{h}^\star_t (\bm{x}) \right)\nonumber\\
    & \leq & \exp\left(\theta \cdot \sum_{t=1}^{T} c_t \right) \cdot \prod_{t=1}^{T} {\sqrt{1 - \mu^2_{t}}}, \forall \theta \in \mathbb{R}. \label{eq-bM1}
 \end{eqnarray}
 We simplify the RHS using the expression of $c_t$ in \eqref{ubound}:
 \begin{eqnarray*}
   \exp\left(\theta \cdot \sum_{t=1}^{T} c_t \right) \cdot \prod_{t=1}^{T} {\sqrt{1 - \mu^2_{t}}} & = & \prod_{t=1}^{T} \left( \frac{1+\mu_t}{1-\mu_t}\right)^{\frac{\theta}{4}}\cdot \prod_{t=1}^{T} {\sqrt{1 - \mu^2_{t}}} \\
   & = & \sqrt{\left(\prod_{t=1}^{T} (1+\mu_t)\right)^{1+\frac{\theta}{2}}\cdot  \left(\prod_{t=1}^{T} (1-\mu_t)\right)^{1-\frac{\theta}{2}}},
 \end{eqnarray*}
 and we use Lemma \ref{lem-bsup}: assuming $0 \leq \mu_t <1, \forall t$ (note that we necessarily have $|\mu_t|\leq 1, \forall t$), we get
 \begin{eqnarray*}
   \exp\left(\theta \cdot \sum_{t=1}^{T} c_t \right) \cdot \prod_{t=1}^{T} {\sqrt{1 - \mu^2_{t}}} & \leq & \exp \left(-\frac{\theta }{4} \cdot \sum_{t=1}^T \mu_t \right), \forall \theta \in \left[0, \frac{2}{3} \cdot \frac{\sum_{t=1}^T \mu_t^2}{\sum_{t=1}^T \mu_t}\right),
 \end{eqnarray*}
 which we connect to \eqref{eq-bM1} and finally get
 \begin{eqnarray}
\pr_{i \sim \ve{w}}\left[\ve{y}_{i}^\top \ve{H}_T(\bm{x}_i) \leq \theta\right] & \leq & \exp \left(-\frac{\theta }{4} \cdot \sum_{t=1}^T \mu_t \right). \label{eq-boosted-bsup}
 \end{eqnarray}
We then remark that the LHS is a non-decreasing function of $\theta$ while the RHS is a strictly decreasing function of $\theta$, and so we get
  \begin{eqnarray}
    \forall \theta \leq \frac{2}{3} \cdot \frac{\sum_{t=1}^T \mu_t^2}{\sum_{t=1}^T \mu_t},  \quad\pr_{i \sim \ve{w}}\left[\ve{y}_{i}^\top \ve{H}_T(\bm{x}_i) \leq \theta\right] & \leq & \exp \left(-\frac{1}{6} \cdot \frac{\sum_{t=1}^T \mu_t^2}{\sum_{t=1}^T \mu_t} \cdot \sum_{t=1}^T \mu_t \right)\nonumber\\
     & & = \exp \left(-\frac{1}{6} \cdot \sum_{t=1}^T \mu_t^2 \right) .\label{eq-boosted-bsup-gen}
 \end{eqnarray}
 We finally process the event: we remark that
 \begin{eqnarray}
   \ve{y}_{i}^\top \ve{H}_T(\bm{x}_i) & \defeq & \frac{1}{\sum_{t=1}^{T} c_t} \cdot \sum_{t=1}^{T} c_t \cdot \ve{y}_{i}^\top \ve{h}^\star_t (\bm{x}_i)\nonumber\\
                                           & = & \frac{1}{\sum_{t=1}^{T} c_t} \cdot \sum_{t=1}^{T} c_t \cdot \frac{1}{h_{t,\infty}} \cdot \ve{y}_{i}^\top \ve{h}_t(\bm{x}_i) \nonumber\\
   & = &  \frac{1}{\sum_{t=1}^{T} c_t} \cdot \sum_{t=1}^{T}  \left( \frac{1}{n_i} \cdot \sum_{j\in \mathcal{Y}_i} \frac{c_t h_{t, j}}{h_{t,\infty}} - \frac{1}{n-n_i}\cdot \sum_{j\in \overline{\mathcal{Y}}_i}  \frac{c_t h_{t, j}}{h_{t,\infty}}\right)  \label{eq-latent-c}
 \end{eqnarray}
 where we have used the definition of $\ve{h}^\star_t$ and the fact that $\ve{y}_. \in \mathcal{Y}$, reminding that $\mathcal{Y}_i \defeq \{j : y_{ij} = 1/n_i\}$ denotes the set of potential next tokens, while $\overline{\mathcal{Y}}_i \defeq \{j : y_{ij} = -1/(n-n_i)\}$ denotes the rest of the tokens (because of the definition of $\mathcal{Y}$). 
   Recall that the mentored boosted probability vector is defined as 
   \begin{eqnarray*}
     \tilde{\ve{\pi}}_T (\ve{x}) & \defeq &\frac{1}{Z_T} \cdot \prod_{t=1}^{T} \left(\ve{p}_t(\ve{x})\right)^{\frac{c_t}{\sum_{u \in [T]} c_u}}, \quad \ve{p}_t(\ve{x}) \defeq  \frac{1}{Z_t} \cdot \exp\left(\frac{1}{h_{t,\infty}}\cdot \ve{h}_t(\ve{x})\right) 
   \end{eqnarray*}
   so that the event "$\ve{y}_{i}^\top \ve{H}_T(\bm{x}_i) \leq \theta$" equivalently states, after taking exponentials and normalizing by $Z_T \cdot \prod_{t=1}^{T} Z_t^{\frac{c_t}{\sum_{u \in [T]} c_u}}$,
   \begin{eqnarray*}
     \overline{\{\tilde{\pi}_{T,j}(\ve{x}_i), j \in \mathcal{Y}_i\}}^G& \leq & \overline{\{\tilde{\pi}_{T,j}(\ve{x}_i), j \in \overline{\mathcal{Y}}_i\}}^G \cdot \exp(\theta)
   \end{eqnarray*}
   where $\tilde{\pi}_{T,k}(\ve{x})$ is coordinate $k$ in $\ve{p}_T(\ve{x})$ and for any set of non negative reals $\mathcal{A}$, $\overline{\mathcal{A}}^G$ denotes the geometric average of the elements of $\mathcal{A}$. There remains to put this event in \eqref{eq-boosted-bsup-gen} to get that $\forall \rho \leq \exp\left(\frac{2}{3} \cdot \frac{\sum_{t=1}^T \mu_t^2}{\sum_{t=1}^T \mu_t}\right)$,
    \begin{eqnarray*}
     \pr_{i \sim \ve{w}}\left[ \overline{\{\tilde{\pi}_{T,j}(\ve{x}_i), j \in \mathcal{Y}_i\}}^G \leq \rho \cdot \overline{\{\tilde{\pi}_{T,j}(\ve{x}_i), j \in \overline{\mathcal{Y}}_i\}}^G \right] & \leq & \exp \left(-\frac{1}{6} \cdot \sum_{t=1}^T \mu_t^2\right).
    \end{eqnarray*}
    Finally, we remark that $(\sum_{t=1}^T \mu_t^2) / \sum_{t=1}^T \mu_t = \expect[\mu] + \var[\mu] / \expect[\mu]$, and conclude with the statement of the Theorem.

  \subsection{Proof of Lemma \ref{lem-tempered-mentor}}\label{sec-proof-lem-tempered-mentor}

Coordinates of $\ve{v}_\alpha$ are $v_{\alpha, i} = p_i^\alpha q_i^{1-\alpha} / Z$ with $Z = \sum_{j \in [n]} p_j^\alpha q_j^{1-\alpha}$. So we want
  \begin{eqnarray}
    \min\{p_i, q_i\}  \leq \frac{p_i^\alpha q_i^{1-\alpha}}{Z} \leq  \max\{p_i, q_i\}, \forall i \in [n], \label{ineqGM}
  \end{eqnarray}
The AGH inequality yields $Z \leq \sum_{j \in [n]} \alpha p_j + (1-\alpha) q_j = 1$ so to get \eqref{ineqGM} we only have to guarantee $Z \geq p_i^\alpha q_i^{1-\alpha} / \max\{p_i, q_i\}, \forall i \in [n]$, or equivalently,
  \begin{eqnarray}
    Z & \geq & \max_i \left(\frac{q_i}{p_i}\right)^{\iver{p_i > q_i}-\alpha}\label{eq-B-Z}.
  \end{eqnarray}
  Introducing $i_* \in [n]$ the index realizing the max (assuming it is unique for simplicity), the RHS can be reformulated as (we recall that $\alpha \in (0,1)$)
  \begin{eqnarray*}
    \max_i \left(\frac{q_i}{p_i}\right)^{\iver{p_i > q_i}-\alpha} & = & \max_i \min \left\{\left(\frac{q_i}{p_i}\right)^{1-\alpha}, \left(\frac{p_i}{q_i}\right)^{\alpha}\right\}\\
     & = & \left(\max_i \min \left\{\frac{q_i}{p_i}, \frac{p_i}{q_i}\right\}\right)^{\iver{p_{i_*} > q_{i_*}}\cdot(1-\alpha) + \iver{p_{i_*} < q_{i_*}}\cdot \alpha}.
  \end{eqnarray*}
We use the tempered logarithm and tempered exponential as \cite[Chapter 7]{nGT}:
  \begin{eqnarray*}
\log_t (z) \defeq \frac{1}{1-t}\cdot\left(z^{1-t}-1\right) & , & \exp_t (z) \defeq \left[1+(1-t) z\right]^{1/(1-t)}_+ \quad ([z]_+ \defeq \max\{0,z\}),
  \end{eqnarray*}
  where the case $t=1$ is supposed to be the extension by continuity to the $\log$ and $\exp$ functions, respectively. We shall consider $t\in (0,1)$, values for which the concavity / convexity of functions is the same as for $t=1$, see also \citet{DBLP:conf/aaai/AmidNNW24,DBLP:conf/aistats/AmidNW23,DBLP:conf/nips/NockAW23}.
  
  Suppose first that $p_{i_*} > q_{i_*}$. Remark that $Z^{\frac{1}{1-\alpha}}$ can be conveniently rewritten as (we recall that $\alpha \in (0,1)$ and $Z \in [0,1]$)
  \begin{eqnarray*}
    Z^{\frac{1}{1-\alpha}} & = & \left[1 + (1-\alpha)\cdot \left( \frac{1}{1-\alpha} \cdot \sum_{i\in [n]} p_i^\alpha q_i^{1-\alpha} -p_i\right)\right]^{\frac{1}{1-\alpha}}\\
                           & = & \left[1 + (1-\alpha)\cdot \left( \sum_{i\in [n]} p_i \cdot \left\{\frac{1}{1-\alpha} \cdot \left(\left(\frac{q_i}{p_i}\right)^{1-\alpha} -1\right)\right\}\right)\right]^{\frac{1}{1-\alpha}}\\
                                 & = & \exp_t \expect_{i \sim \ve{p}} \log_t \frac{q_i}{p_i}, \quad \mbox{ with $t \defeq \alpha$},
  \end{eqnarray*}
  and if $p_{i_*} < q_{i_*}$, we rewrite $Z^{\frac{1}{\alpha}}$ as
  \begin{eqnarray*}
    Z^{\frac{1}{\alpha}} & = & \left[1 + (1-(1-\alpha))\cdot \left( \frac{1}{1-(1-\alpha)} \cdot \sum_{i\in [n]} p_i^\alpha q_i^{1-\alpha} -q_i\right)\right]^{\frac{1}{1-(1-\alpha)}}\\
                           & = & \left[1 + (1-(1-\alpha))\cdot \left( \sum_{i\in [n]} q_i \cdot \left\{\frac{1}{1-(1-\alpha)} \cdot \left(\left(\frac{p_i}{q_i}\right)^{1-(1-\alpha)} -1\right)\right\}\right)\right]^{\frac{1}{1-(1-\alpha)}}\\
                                 & = & \exp_t \expect_{i \sim \ve{q}} \log_t \frac{p_i}{q_i}, \quad \mbox{ with $t \defeq 1-\alpha$}.
  \end{eqnarray*}
  To summarize, \eqref{eq-B-Z} is equivalent to having
  \begin{eqnarray*}
    \exp_\alpha \expect_{i \sim \ve{p}} \log_\alpha \frac{q_i}{p_i} & \geq & \max_i \min \left\{\frac{q_i}{p_i}, \frac{p_i}{q_i}\right\} \quad \mbox{ if $p_{i_*} > q_{i_*}$},\\
    \exp_{1-\alpha} \expect_{i \sim \ve{q}} \log_{1-\alpha} \frac{p_i}{q_i} & \geq & \max_i \min \left\{\frac{q_i}{p_i}, \frac{p_i}{q_i}\right\} \quad \mbox{ if $p_{i_*} < q_{i_*}$},
  \end{eqnarray*}
  and provided these hold, we get
  \begin{eqnarray}
    \ve{v}_\alpha \in [\min\{\ve{p}, \ve{q}\}, \max\{\ve{p}, \ve{q}\}],\label{eq-const-tv-md-proof}
  \end{eqnarray}
  which is the statement of the Lemma.

  \subsection{Proof of Theorem \ref{thm-boost-and-mentor-TV}}\label{sec-proof-thm-boost-and-mentor-TV}

  We proceed in two steps, first showing the bullet elements of the statement, and then showing the bound on $D_{f_{\mathrm{TV}}} (\ve{v}_\alpha \|\ve{q})$ in  \eqref{eq-bound-tv-alpha}. Our first step proof uses the following technical Lemma.
  \begin{lemma}\label{lem-exp-esp-log}
    let $X$ be a random variable with values in an interval $[u,v]$ satisfying $u>0$ and $\expect[X] = 1$. Then for any $t \in [0,1]$ it holds that
    \begin{eqnarray}
\exp_t \expect[\log_t X] & \geq & \left(\frac{u}{v}\right)^t \label{eq-binf-expt}.
      \end{eqnarray}
    \end{lemma}
    \begin{proof}
      We first prove the result for $t \in (0,1)$. Since $\log_t$ is concave for any $t \in [0,1]$, it lies above any of its secants in the interval defined by the intersections, so we get
      \begin{eqnarray*}
        \log_t z & \geq & \frac{v-z}{v-u} \cdot \log_t(u) + \frac{z-u}{v-u} \cdot \log_t(v),
      \end{eqnarray*}
      which yields, since $\expect[X] = 1$,
      \begin{eqnarray*}
        \expect[\log_t X] & \geq & \frac{v-1}{v-u} \cdot \log_t(u) + \frac{1-u}{v-u} \cdot \log_t(v),
      \end{eqnarray*}
      and letting $p \defeq (v-1)/(v-u) \in [0,1]$, yields the equivalent inequality after expressing $\log_t$:
      \begin{eqnarray*}
        \expect[\log_t X] & \geq & \frac{p u^{1-t} + (1-p) v^{1-t} - 1}{1-t},
      \end{eqnarray*}
      and since $\exp_t$ is non-decreasing for any $t \in [0,1]$, ensures \eqref{eq-binf-expt} provided the sufficient condition holds: $\exp_t\left(\frac{p u^{1-t} + (1-p) v^{1-t} - 1}{1-t}\right) \geq \left(\frac{u}{v}\right)^t$, which simplifies to checking the condition:
      \begin{eqnarray}
        p u^{1-t} + (1-p) v^{1-t} & \geq & \left(\frac{u}{v}\right)^{t(1-t)}. \label{eq-suff-cond}
      \end{eqnarray}
      Since $u\leq \expect[X] = 1 \leq v$, it is enough to check this inequality for any $0<u<1, v>1$ with $p \defeq (v-1)/(v-u) \in [0,1]$. Let us simplify it once more. Define $x \defeq v/u \geq 1$. The RHS of \eqref{eq-suff-cond} only depends on $x$ and $t$, and it turns out the LHS simplifies:
      \begin{eqnarray*}
        p u^{1-t} + (1-p) v^{1-t} & = & u^{1-t} \cdot \left( p + (1-p) x^{1-t}\right)\\
        & = & \frac{p + (1-p) x^{1-t}}{(p+(1-p)x)^{1-t}},
      \end{eqnarray*}
      since $pu + (1-p)v = 1$ and $v=ux$ which yields $u=1/(p+(1-p)x)$. Using these expressions depending on $p,x,t$ and taking logs in \eqref{eq-suff-cond}, we want to show
      \begin{eqnarray}
        \underbrace{\log(p + (1-p) x^{1-t}) - (1-t)\log(p+(1-p)x)}_{\defeq g_x(p)} & \geq & -t(1-t) \log (x). \label{eq-suff-cond-log}
      \end{eqnarray}
      To show this, let us first analyze $g_x(p)$. We have
      \begin{eqnarray*}
        g'_x(p) & = & \frac{(1-t)(x-1)}{p+(1-p)x} -\frac{x^{1-t} -1}{p + (1-p) x^{1-t}} ,
      \end{eqnarray*}
      and we obtain $g'_x(p) \leq 0$ iff
      \begin{eqnarray*}
        p & \leq & Q \defeq \frac{(1-t)x^{1-t}+tx^{2-t}-x}{t(x-1)(x^{1-t}-1)},
      \end{eqnarray*}
      and we remark that
      \begin{eqnarray*}
        Q & = & 1 + \frac{x^{1-t} + (t+1)x - t}{t(x-1)(x^{1-t}-1)}\\
         & \geq & 1,
      \end{eqnarray*}
      since $x\geq 1, t \in [0,1]$, so the minimum of $g_x(p)$ for $p\in [0,1]$ is obtained at $g_x(1) = 0$, showing $g_x(p)\geq 0, \forall x\geq 1,  t \in [0,1]$. Since the RHS of \eqref{eq-suff-cond-log} is $\leq 0$ under these conditions, \eqref{eq-suff-cond-log} is proven and so is Lemma \ref{lem-exp-esp-log} for $t\in (0,1)$. We then remark the continuity of both functions in \eqref{eq-binf-expt} for $t\in [0,1]$ so taking the limits in $0^+$ and $1^-$ completes the proof.
    \end{proof}
We now shift to analyzing boosting \textit{under the constraint} that we must keep \eqref{eq-const-tv-md-proof} true. Our first step consists of showing constraints on the exponent $\alpha$, and for this we distinguish two cases.\\
    
\noindent \textbf{Case 1:}  we first assume $p_{i_*} > q_{i_*}$, so we work with the constraint
  \begin{eqnarray*}
    \exp_\alpha \expect_{i \sim \ve{p}} \log_\alpha \frac{q_i}{p_i} & \geq & \max_i \min \left\{\frac{q_i}{p_i}, \frac{p_i}{q_i}\right\},
  \end{eqnarray*}
  which, from Lemma \ref{lem-exp-esp-log}, holds if we have $\alpha \log\left(\frac{\min_i \frac{q_i}{p_i}}{\max_i \frac{q_i}{p_i}}\right) \geq \log\left(\max_i \min \left\{\frac{q_i}{p_i}, \frac{p_i}{q_i}\right\}\right)$, that is,
  \begin{eqnarray}
\alpha & \leq & \frac{\log\left(\min_i \max \left\{\frac{q_i}{p_i}, \frac{p_i}{q_i}\right\}\right)}{\log\left(\frac {\max_i \frac{q_i}{p_i}}{\min_i \frac{q_i}{p_i}}\right)}.\label{eq-alpha1}
  \end{eqnarray}
  This provides our first constraint on $\alpha$. We move to the alternative one.\\

\noindent \textbf{Case 2:}  if instead $p_{i_*} < q_{i_*}$, we want from Lemma \ref{lem-tempered-mentor}
  \begin{eqnarray*}
    \exp_{1-\alpha} \expect_{i \sim \ve{q}} \log_{1-\alpha} \frac{p_i}{q_i}  & \geq & \max_i \min \left\{\frac{q_i}{p_i}, \frac{p_i}{q_i}\right\},
  \end{eqnarray*}
  and so we want from Lemma \ref{lem-exp-esp-log}
  \begin{eqnarray}
1-\alpha & \leq & \frac{\log\left(\min_i \max \left\{\frac{q_i}{p_i}, \frac{p_i}{q_i}\right\}\right)}{\log\left(\frac {\max_i \frac{q_i}{p_i}}{\min_i \frac{q_i}{p_i}}\right)}. \label{eq-alpha2}
  \end{eqnarray}

  At this point, if we can provide boosting coefficients $c_1 = f(\mu_1)$ and $c_2=f(\mu_2)$ where $f$ is given in \eqref{eq-def-ct} (main file), such that (a)
  \begin{eqnarray}
    \alpha & \defeq & \frac{c_u}{c_u+c_v}\label{eq-alpha-uv}
  \end{eqnarray}
  satisfies whichever \eqref{eq-alpha1} or \eqref{eq-alpha2} is relevant (where distinct $u,v$ are in $\{1,2\}$), \textit{and} (b) the associated boosting advantage is large enough to show that the combination of two models does satisfy the exponential decrease associated in \eqref{eq-bound-boosted-2} (main file), then we are done: in all cases, the boosted solution is also optimal for the TV-MD problem. What we need to do is find the sequence of models (among drafter and target) to include in the boosted model, and find the edges $\mu_1$ and $\mu_2$ such that the related boosting advantage is large enough.

  In the context of boosting, we want the best guarantee from the boosting advantage standpoint, so let us assume that the first model we include is the target model, so that $\mu_1 = \mu_\target$, and show how to collapse both cases above as a single one that controls $\mu_2$ as an eventual clamping of $\mu_\drafter$. If $p_{i_*} > q_{i_*}$, we fix $u=1, v=2$ in \eqref{eq-alpha-uv} and need to show \eqref{eq-alpha1} for $\alpha = c_1 / (c_1 + c_2)$. Otherwise if $p_{i_*} < q_{i_*}$, we permute $u=2, v=1$ in \eqref{eq-alpha-uv} so that $1-\alpha = c_1 / (c_1 + c_2)$ and \eqref{eq-alpha2} is the same condition as for the first case.\\

We now include \eqref{eq-def-eps}. The RHS of \eqref{eq-alpha1} is $\geq \log(1+\rhopq)/\log(1/\epsilonpq^2)$ and since $\alpha = c_2/(c_1 + c_2)$, \eqref{eq-alpha1} is equivalent to having
  \begin{eqnarray*}
    c_2 & \leq & \frac{\log(1+\rhopq)}{\log(1/\epsilonpq^2)-\log(1+\rhopq)} \cdot c_1,
  \end{eqnarray*}
  which, predictably, prevents a too large boosting coefficient for the drafter model. We now need another technical Lemma
  \begin{lemma}\label{lem-zx}
    For any $z,x$ such that $z \in (0,1]$, $x\geq 0$ and $z (1+x) \leq 1$,
    \begin{eqnarray}
      \frac{\ln(1+x)}{\ln\left(\frac{1}{z^2}\right) - \ln(1+x)} & \geq & z \cdot x. \label{eq-binf-zx}
    \end{eqnarray}
  \end{lemma}
  \begin{proof}
    The denominator in the LHS of \eqref{eq-binf-zx} being non negative, we reformulate the inequality as $f_x(z) \defeq (1+zx)\ln(1+x) + 2zx \ln(z) \geq 0$, where we treat $x$ as a constant. We easily get that $f_x(z)$ is strictly convex and has a global minimum at $z_* \defeq 1/(e \sqrt{1+x})$. If the minimum falls in the set of constraints for the value of $x$ (we can show that this happens iff $x \leq e^2 - 1$), we get that
    \begin{eqnarray*}
      f_x(z_*) & = & \ln(1+x) - \frac{2x}{e\sqrt{1+x}},
    \end{eqnarray*}
    and so, letting $y \defeq \sqrt{1+x}$, we need to show $e y\ln y - y^2 +1 \geq 0, \forall y \in [1, e]$, which is easily checked (the function is 0 in $y=1$ and its derivative is $\geq 0$ on $[1,e]$). A similar proof holds if $z_*$ is not in the set of constraints.
\end{proof}
Lemma \ref{lem-zx} yields, since $\epsilonpq \in (0,1]$, $\rhopq\geq 0$ and $\epsilonpq (1+\rhopq) \leq 1$, 
  \begin{eqnarray*}
\frac{\log(1+\rhopq)}{\log(1/\epsilonpq^2)-\log(1+\rhopq)} & \geq & \rhopq \cdot \epsilonpq,
  \end{eqnarray*}
  so we can simplify the requirement to $c_2 \leq \rhopq \epsilonpq c_1$, which equivalently reads
  \begin{eqnarray}
\frac{1+\mu_2}{1-\mu_2} & \leq & \left(\frac{1+\mu_\target}{1-\mu_\target} \right)^{\rhopq \epsilonpq},\label{eq-b-comp-p}
  \end{eqnarray}
  and thus yields
  \begin{eqnarray}
\mu_2 & \defeq & \min \left\{\mu_\drafter, \frac{\left(\frac{1+\mu_\target}{1-\mu_\target}\right)^{\rhopq \epsilonpq}-1}{\left(\frac{1+\mu_\target}{1-\mu_\target}\right)^{\rhopq \epsilonpq}+1}\right\}.\label{eq-def-mu2-app}
  \end{eqnarray}
  Note that this biases the computation of boosting coefficients but since we do not include further models after $t=2$, this does not change the analysis of the convergence, and the current analysis in the proof of Theorem \ref{th-boost-P} accommodates for the case where the \textit{last} boosting coefficient is eventually reduced.\\
  
  We now need another simple technical Lemma.
  \begin{lemma}\label{lem-perspective}
Let $f$ be convex such that $f(0) = 0$. Then for any $z \in \mathrm{dom}f$ and any $t\in [0,1]$, $t \cdot f(z) \geq f(t\cdot z)$.
\end{lemma}
\begin{proof}
For $t\in (0,1]$, we just remark that since $f$ is convex, $f(tz + (1-t)\cdot 0) \leq t\cdot f(z) + (1-t)\cdot f(0)$, which, since $f(0) = 0$, simplifies into the Lemma's statement. The case $t=0$ is immediate.
\end{proof}
Now, pick
\begin{eqnarray*}
f(z) & \defeq & \log\left(\frac{1+z}{1-z}\right)
\end{eqnarray*}
restricted to $[0,1]$, in which it is convex. Using Lemma \ref{lem-perspective}, we can expand the inequality $t \cdot \log\left(\frac{1+z}{1-z}\right) \geq \log\left(\frac{1+tz}{1-tz}\right)$ (for any $z, t \in [0,1]$) into the equivalent one in which we substitute $t \defeq \rhopq \epsilonpq$:
\begin{eqnarray*}
  \left(\frac{\left(\frac{1+z}{1-z}\right)^{\rhopq \epsilonpq}-1}{\left(\frac{1+z}{1-z}\right)^{\rhopq \epsilonpq}+1}\right)^2 & \geq & \rhopq^2 \epsilonpq^2 z^2, \forall z \in [0,1],
\end{eqnarray*}
so getting back to \eqref{eq-def-mu2-app}, this translates into a guaranteed boosting advantage
\begin{eqnarray*}
  A\left(\{\ve{h}_1 \defeq \ve{h}_\target, \ve{h}_2 \defeq \ve{h}_\drafter\}\right) & \geq & \mu^2_\target + \rhopq^2 \epsilonpq^2 \mu^2_\target\\
  &  & = (1 + \rhopq^2 \epsilonpq^2)\cdot \mu^2_\target.
\end{eqnarray*}
This guarantee holds for the $\epsilonpq$ corresponding to the current outputs of the drafter and target models. We just need to replace it by the minimal $\rhopq\epsilonpq > 0$ that would be satisfied for any outputs, and we get the bullet statements of Theorem \ref{thm-boost-and-mentor-TV}.\\

We now proceed to showing the bound on $D_{f_{\mathrm{TV}}} (\ve{v}_\alpha \|\ve{q})$ in  \eqref{eq-bound-tv-alpha}. Let us rewrite the TV distance:
\begin{eqnarray*}
  D_{f_{\mathrm{TV}}} (\ve{v}_\alpha \|\ve{q}) & = & \frac{1}{2} \sum_{i\in n} \left| \frac{p_i^\alpha q_i^{1-\alpha}}{Z} - q_i \right|\\
                                  & = &  \frac{1}{2} \sum_{i\in n} q_i \cdot \left| \frac{1}{Z} \cdot \frac{p_i^\alpha }{q_i^{\alpha}} - 1\right|\\
  & = & \frac{1}{2} \cdot \expect \left[\left| \frac{X^\alpha}{Z} - 1 \right|\right],
\end{eqnarray*}
where $X$ is a random variable taking value $p_i/q_i$ with probability $q_i$, hence satisfying $\expect[X] = 1$. Since $\ve{p}, \ve{q}$ obey \eqref{eq-def-eps} and $z \mapsto z^\alpha$ is strictly monotonic for $\alpha \in (0,1]$, the TV is maximal iff all values taken by $X$ are at the boundary of its range, i.e. either $\epsilonpq\leq 1$ or $1/\epsilonpq\geq 1$. But $\expect[X] = 1$ and so the total mass at $\epsilonpq$, $\mu_{\epsilonpq}$ and the total mass at $1/\epsilonpq$, $\mu_{1/\epsilonpq}$ satisfy (a) $\mu_{1/\epsilonpq} + \mu_{\epsilonpq} = 1$ and (b) $(1/\epsilonpq) \cdot \mu_{1/\epsilonpq} + \epsilonpq\cdot \mu_{\epsilonpq} = 1$, a system whose solution is $\mu_{1/\epsilonpq} = \epsilonpq/(1+\epsilonpq), \mu_{\epsilonpq} = 1/(1+\epsilonpq)$, giving the normalization coefficient
\begin{eqnarray*}
Z & = & \mu_{\epsilonpq} \cdot \epsilonpq^\alpha + \mu_{1/\epsilonpq} \cdot \left(\frac{1}{\epsilonpq}\right)^\alpha = \frac{\epsilonpq^\alpha+\epsilonpq^{1-\alpha}}{1+\epsilonpq}
\end{eqnarray*}
and yielding the upperbound
\begin{eqnarray}
  D_{f_{\mathrm{TV}}} (\ve{v}_\alpha \|\ve{q}) &\leq & \frac{1}{2} \cdot \left( \mu_{\epsilonpq} \cdot \left(1-\frac{\epsilonpq^\alpha}{Z}\right) + \mu_{1/\epsilonpq} \cdot \left(\frac{\epsilonpq^{-\alpha}}{Z}-1\right) \right)\nonumber\\
   & & = \frac{1-\epsilonpq}{2\cdot (1+\epsilonpq)} + \frac{\epsilonpq^{1-\alpha} - \epsilonpq^\alpha}{2\cdot(\epsilonpq^{1-\alpha} + \epsilonpq^\alpha)}\label{eq-upper-TV}.
\end{eqnarray}
Now, remark that if we pick
\begin{eqnarray}
\mu_2 & \defeq & \frac{\left(1+\mu_\target\right)^{\epsilonpq\rhopq} - \left(1-\mu_\target\right)^{\epsilonpq\rhopq}}{\left(1+\mu_\target\right)^{\epsilonpq\rhopq} + \left(1-\mu_\target\right)^{\epsilonpq\rhopq}},
\end{eqnarray}
then $\alpha$ simplifies to
\begin{eqnarray*}
\alpha & \defeq & \frac{c_2}{c_2 + c_\target}\\  
       &  =& \frac{\frac{1}{4}\cdot \log\left(\frac{1+\frac{\left(1+\mu_\target\right)^{\epsilonpq\rhopq} - \left(1-\mu_\target\right)^{\epsilonpq\rhopq}}{\left(1+\mu_\target\right)^{\epsilonpq\rhopq} + \left(1-\mu_\target\right)^{\epsilonpq\rhopq}}}{1-\frac{\left(1+\mu_\target\right)^{\epsilonpq\rhopq} - \left(1-\mu_\target\right)^{\epsilonpq\rhopq}}{\left(1+\mu_\target\right)^{\epsilonpq\rhopq} + \left(1-\mu_\target\right)^{\epsilonpq\rhopq}}}\right)}{\frac{1}{4}\cdot\log\left(\frac{1+\frac{\left(1+\mu_\target\right)^{\epsilonpq\rhopq} - \left(1-\mu_\target\right)^{\epsilonpq\rhopq}}{\left(1+\mu_\target\right)^{\epsilonpq\rhopq} + \left(1-\mu_\target\right)^{\epsilonpq\rhopq}}}{1-\frac{\left(1+\mu_\target\right)^{\epsilonpq\rhopq} - \left(1-\mu_\target\right)^{\epsilonpq\rhopq}}{\left(1+\mu_\target\right)^{\epsilonpq\rhopq} + \left(1-\mu_\target\right)^{\epsilonpq\rhopq}}}\right) + \frac{1}{4}\cdot\log\left(\frac{1+\mu_\target}{1-\mu_\target}\right)}\\
   & = & \frac{\epsilonpq\rhopq}{1+\epsilonpq\rhopq},
\end{eqnarray*}
so that \eqref{eq-upper-TV} becomes an upperbound depending solely on $\epsilonpq$:
\begin{eqnarray}
  D_{f_{\mathrm{TV}}} (\ve{v}_\alpha \|\ve{q}) &\leq & \frac{1-\epsilonpq}{2\cdot (1+\epsilonpq)} + \frac{\exp\left(\frac{1}{1+\epsilonpq\rhopq} \cdot \log \epsilonpq\right) - \exp\left(\frac{\epsilonpq\rhopq}{1+\epsilonpq\rhopq} \cdot \log \epsilonpq\right)}{2\cdot\left(\exp\left(\frac{1}{1+\epsilonpq\rhopq} \cdot \log \epsilonpq\right) + \exp\left(\frac{\epsilonpq\rhopq}{1+\epsilonpq\rhopq} \cdot \log \epsilonpq\right)\right)}\label{eq-upper-TV2}.
\end{eqnarray}
Some tedious calculation allow to show that the RHS is $\leq \rhopq\epsilonpq $ for any $\epsilonpq \in [0,1], \rhopq \geq 0$, which gives $D_{f_{\mathrm{TV}}} (\ve{v}_\alpha \|\ve{q}) \leq \rhopq\epsilonpq$, as claimed.

\subsection{Proof of Theorem \ref{thm-boost-plus-md-general-f}}\label{sec-proof-thm-boost-plus-md-general-f}

We recall some notations: for any $i\in [m]$, $\ve{\pi}_i$ denotes the mentored distribution for example $\# i$ in $\mathcal{S}$ following the mentored decoding setting:
\begin{eqnarray}
  \pi_{ij} & \defeq & \left\{
                      \begin{array}{rcll}
                        (1-b_i) \cdot q_{ij} & \mbox{if} & p_{ij} < (1-b_i) q_{ij} & \mbox{(Case (i))}\\
                        p_{ij} & \mbox{if} & p_{ij} \in [1-b_i,1+a_i]\cdot q_{ij} & \mbox{(Case (ii))}\\
                        (1+a_i) \cdot q_{ij} & \mbox{if} & p_{ij} > (1+a_i) q_{ij} & \mbox{(Case (iii))}
                        \end{array}
                      \right., j \in [n] \label{def-eq-cases}
\end{eqnarray}
and
\begin{eqnarray*}
  \tilde{\pi}_{ij} & \defeq & \frac{p_{ij}^\alpha q_{ij}^{1-\alpha}}{Z}, j \in [n]
\end{eqnarray*}
denote the boosting distribution coordinates following the boosting combination setup in \eqref{eq-def-boosted-p}. Notice that under Assumption \ref{assum-mda}, $\tilde{\ve{\pi}}_i > \ve{0}, \forall i \in [m]$. Let us say that example $\# i$ is $\rho$-good iff the event in \eqref{eq-bound-boosted-2} (Theorem \ref{th-boost-P}) is false: in such a case, boosting guarantees
\begin{eqnarray}
  \frac{\overline{\{\tilde{\pi}_{ij}  : j \in \mathcal{Y}_i\}}^G}{\overline{\{\tilde{\pi}_{ij}  : j \in \overline{\mathcal{Y}}_i\}^G}} & \geq & \rho\label{eq-ratio-rho}.
\end{eqnarray}
We recall that $\overline{\mathcal{A}}^G$ denotes the geometric average of the elements of set $\mathcal{A}$. Denore for short $\epsilon_i \defeq \epsilonpqi$ in \eqref{eq-def-eps}. We now combine \eqref{eq-def-eps} with \eqref{def-eq-cases} to find intervals $\tilde{\pi}_{ij} \in \mathbb{I}(\pi_{ij})$ to transform \eqref{eq-ratio-rho} in an inequality involving only the mentored distribution $\ve{\pi}_i$. To simplify notations, we drop index $i$ to focus on coordinate $j$ only.\\

\noindent \textbf{Case (i)} Let us start with Case (i) \eqref{def-eq-cases}. Here, $p_{j} < (1-b_i) q_{j}$, but \eqref{eq-def-eps} guarantees $p_j \geq \epsilon_i q_j$, so to get Case (i), we must have
\begin{eqnarray}
  \epsilon_i & < & 1-b_i \label{eq-const-casei}.
\end{eqnarray}
Provided this holds, we also know $\pi_{j} = (1-b_i) \cdot q_{j}$, so the boosting coordinate $\tilde{\pi}_j$ satisfies $\tilde{\pi}_j < (1-b_i)^\alpha q_j/Z = \pi_{j} / (Z(1-b_i)^{1-\alpha})$. On the other hand, \eqref{eq-def-eps} guarantees $p_j \geq \epsilon_i q_j$, which yields a lowerbound $\tilde{\pi}_j \geq (\epsilon_i^\alpha / Z) \cdot q_i = (\epsilon_i^\alpha / (Z(1-b_i))) \cdot \pi_{j}$, and thus we overall obtain
\begin{eqnarray}
  \tilde{\pi}_j \in \frac{\pi_{j}}{(1-b_i)Z} \cdot \left[ \epsilon_i^\alpha,  (1-b_i)^{\alpha}\right) \quad\mbox{ if } p_j \in q_{j} \cdot [\epsilon_i, 1-b_i).\label{eq-interval-casei}
\end{eqnarray}

\noindent \textbf{Case (iii)} Now, we have $p_{j} > (1+a_i) q_{j}$ but \eqref{eq-def-eps} guarantees $p_j \leq q_j/\epsilon_i$, so to get Case (iii), we must have
\begin{eqnarray}
  \epsilon_i & < & \frac{1}{1+a_i} \label{eq-const-caseiii}.
\end{eqnarray}
Provided this holds, we also know $\pi_{j} = (1+a_i) \cdot q_{j}$, so the boosting coordinate $\tilde{\pi}_j$ satisfies $\tilde{\pi}_j > (1+a_i)^\alpha q_j/Z = \pi_{j} / (Z(1+a_i)^{1-\alpha})$. On the other hand, \eqref{eq-def-eps} guarantees $p_j \leq q_j/\epsilon_i$, which yields an upperbound $\tilde{\pi}_j \leq (1/ (Z \epsilon_i^\alpha)) \cdot q_i = (1 / (Z(1+a_i) \epsilon_i^\alpha)) \cdot \pi_{j}$, and thus we overall obtain
\begin{eqnarray}
  \tilde{\pi}_j \in \frac{\pi_{j}}{(1+a_i)Z} \cdot \left(  (1+a_i)^{\alpha}, \frac{1}{\epsilon_i^\alpha}\right] \quad\mbox{ if } p_j \in q_{j} \cdot (1+a_i, 1/\epsilon_i].\label{eq-interval-caseiii}
\end{eqnarray}

\noindent \textbf{Case (ii)} Now, we have simultaneously $ (1-b_i) q_{j} \leq p_j \leq (1+a_i) q_{j} $, but \eqref{eq-def-eps} guarantees $\epsilon_i q_j \leq p_j \leq q_j / \epsilon_j$ so to get Case (ii), the intervals must have a non-empty intersection and we must have $\epsilon_i  \leq 1/(1-b_i)$ or $\epsilon_i  \leq 1+a_i$ -- which always holds since $b_i\in [0, 1], a_i\geq 0$ --. Since $\pi_j = p_j$, we now have the direct expression
\begin{eqnarray*}
  \tilde{\pi}_j & = & \frac{p_j^\alpha q_j^{1-\alpha}}{Z}\\
  & = & \pi_j \cdot  \frac{1}{Z} \cdot \left(\frac{q_j}{p_j}\right)^\alpha,
\end{eqnarray*}
and we can check that with the inequalities above, we get
\begin{eqnarray}
  \tilde{\pi}_j \in \frac{\pi_{j}}{Z} \cdot \left[\frac{1}{(1+a_i)^\alpha}, \frac{1}{(1-b_i)^\alpha}\right] \quad\mbox{ if } p_j \in q_{j} \cdot [1-b_i, 1 + a_i].\label{eq-interval-caseii}
\end{eqnarray}
Using \eqref{eq-interval-casei}, \eqref{eq-interval-caseiii}, \eqref{eq-interval-caseii}, we obtain the upperbound,
\begin{eqnarray}
  \frac{\overline{\{\tilde{\pi}_{ij}  : j \in \mathcal{Y}_i\}}^G}{\overline{\{\tilde{\pi}_{ij}  : j \in \overline{\mathcal{Y}}_i\}}^G} & \leq & \frac{\overline{\{\pi_{ij}  : j \in \mathcal{Y}_i\}}^G}{\overline{\{\pi_{ij}  : j \in \overline{\mathcal{Y}}_i\}}^G} \cdot \underbrace{\frac{\max\left\{\frac{1}{(1-b_i)^{1-\alpha}}, \frac{1}{(1-b_i)^{\alpha}}, \frac{1}{(1+a_i)\cdot \epsilon_i^\alpha}\right\}}{\min\left\{\frac{1}{(1+a_i)^{1-\alpha}}, \frac{1}{(1+a_i)^{\alpha}}, \frac{\epsilon_i^\alpha}{1-b_i}\right\}}}_{\defeq R^{-1}}. \label{eq-boost-ment}
\end{eqnarray}
We now need to find a lowerbound for $R$. We distinguish two cases:\\

\noindent \textbf{Case (A)}: $\epsilon^\alpha_i < (1-b_i)/(1+a_i)$. Thus, $\epsilon^\alpha_i / (1-b_i)< 1/(1+a_i)\leq 1$ since $a_i\leq 0$, and since $0\leq \alpha \leq 1$, the numerator of $R$ satisfies
\begin{eqnarray*}
\min\left\{\frac{1}{(1+a_i)^{1-\alpha}}, \frac{1}{(1+a_i)^{\alpha}}, \frac{\epsilon_i^\alpha}{1-b_i}\right\} & \geq & \frac{\epsilon^\alpha_i}{1-b_i};
\end{eqnarray*}
We also get $1 \leq 1/(1-b_i) \leq 1/((1+a_i) \epsilon^\alpha_i)$ since $b_i\in [0,1]$, and since $0\leq \alpha \leq 1$, the denominator of $R$ satisfies
\begin{eqnarray*}
\max\left\{\frac{1}{(1-b_i)^{1-\alpha}}, \frac{1}{(1-b_i)^{\alpha}}, \frac{1}{(1+a_i)\cdot \epsilon_i^\alpha}\right\} & \leq & \frac{1}{(1+a_i) \epsilon^\alpha_i},
\end{eqnarray*}
and finally
\begin{eqnarray}
R & \geq & \frac{1+a_i}{1-b_i} \cdot \epsilon^{2\alpha}_i. \label{eq-binf-r-1}
  \end{eqnarray}
  \noindent \textbf{Case (B)}: $\epsilon^\alpha_i \geq (1-b_i)/(1+a_i)$. Thus $\epsilon^\alpha_i / (1-b_i) \geq 1/(1+a_i)$ and the numerator of $R$ satisfies
  \begin{eqnarray*}
    \min\left\{\frac{1}{(1+a_i)^{1-\alpha}}, \frac{1}{(1+a_i)^{\alpha}}, \frac{\epsilon_i^\alpha}{1-b_i}\right\} & \geq & \min\left\{\frac{1}{(1+a_i)^{1-\alpha}}, \frac{1}{(1+a_i)^{\alpha}}, \frac{1}{1+a_i}\right\}\\
    & & = \frac{1}{1+a_i}
  \end{eqnarray*}
  since $0\leq \alpha \leq 1$. Similarly for the denominator, since $1/(1-b_i) \geq 1/((1+a_i) \epsilon^\alpha_i)$, we observe
  \begin{eqnarray*}
    \max\left\{\frac{1}{(1-b_i)^{1-\alpha}}, \frac{1}{(1-b_i)^{\alpha}}, \frac{1}{(1+a_i)\cdot \epsilon_i^\alpha}\right\} & \leq & \max\left\{\frac{1}{(1-b_i)^{1-\alpha}}, \frac{1}{(1-b_i)^{\alpha}}, \frac{1}{1-b_i}\right\}\\
     & & = \frac{1}{1-b_i},
  \end{eqnarray*}
  and finally 
  \begin{eqnarray}
R & \geq & \frac{1-b_i}{1+a_i}. \label{eq-binf-r-2}
  \end{eqnarray}
  and we finally check that \eqref{eq-binf-r-1} and \eqref{eq-binf-r-2} can be folded into one:
  \begin{eqnarray*}
    R & \geq & \frac{1+a_i}{1-b_i} \cdot \min\left\{\epsilon^\alpha_i , \frac{1-b_i}{1+a_i}\right\}^2\\
     & & = \min\left\{\epsilon^\alpha_i \cdot \sqrt{\frac{1+a_i}{1-b_i} }, \sqrt{\frac{1-b_i}{1+a_i} }\right\}^2
  \end{eqnarray*}
  We can then simplify \eqref{eq-boost-ment} into a more readable uperbound:
\begin{eqnarray*}
  \frac{\overline{\{\tilde{\pi}_{ij}  : j \in \mathcal{Y}_i\}}^G}{\overline{\{\tilde{\pi}_{ij}  : j \in \overline{\mathcal{Y}}_i\}}^G} & \leq & \frac{\overline{\{\pi_{ij}  : j \in \mathcal{Y}_i\}}^G}{\overline{\{\pi_{ij}  : j \in \overline{\mathcal{Y}}_i\}}^G} \cdot \frac{1}{\min\left\{\epsilon^\alpha_i \cdot \sqrt{\frac{1+a_i}{1-b_i} }, \sqrt{\frac{1-b_i}{1+a_i} }\right\}^2},
\end{eqnarray*}
from which it comes that, for any $i\in [m]$ and $\rho \geq 0$, if
\begin{eqnarray*}
  \overline{\{\pi_{ij}  : j \in \mathcal{Y}_i\}}^G & \leq & \rho \cdot \min\left\{\epsilon^\alpha_i \cdot \sqrt{\frac{1+a_i}{1-b_i} }, \sqrt{\frac{1-b_i}{1+a_i} }\right\}^2\cdot \overline{\{\pi_{ij}  : j \in \overline{\mathcal{Y}}_i\}}^G,
\end{eqnarray*}
then
\begin{eqnarray*}
  \overline{\{\tilde{\pi}_{ij}  : j \in \mathcal{Y}_i\}}^G & \leq & \rho \cdot \overline{\{\tilde{\pi}_{ij}  : j \in \overline{\mathcal{Y}}_i\}}^G,
\end{eqnarray*}
and yields to the statement of Theorem \ref{thm-boost-plus-md-general-f} via Theorem \ref{th-boost-P}.

\subsection{Proof of Theorem \ref{thm-cab-from-c-breakpoints}}\label{sec-proof-thm-cab-from-c-breakpoints}
    
Without loss of generality, we assume all ratios $p_. / q_.$ are distinct (otherwise, we group the $p$s and $q$s) and indices are ordered in increasing ratio. Let us index in $\{1, 2, ...\}$ the breakpoints in the order they are put in $\textsc{c}$ by algorithm \cabbreakpoints.

Clearly, the list of breakpoints built by \cabbreakpoints~is built in strictly increasing order of $a$ and $b$. Clearly also, the first breakpoint, $(a,b) = (0,0)$ is in $\mathcal{C}(\ve{p}, \ve{q})$ because $ a q(\mathbb{A}) = 0 = 0 + 1 - 1 = b q(\mathbb{B}) + q(\mathbb{I}) - p(\mathbb{I})$ so all \eqref{kktconst1}, \eqref{kktconst2} and \eqref{kktconst3} are satisfied. Now take any such breakpoint $(a,b) \in \mathcal{C}(\ve{p}, \ve{q})$. Let
    \begin{eqnarray*}
      i & \defeq & \min \mathbb{A},\\
      j & \defeq & \max \mathbb{B}.
    \end{eqnarray*}

    We have three cases:\\

    \noindent \textbf{Case 1}: suppose that the current breakpoint satisfies
    \begin{eqnarray}
\left(1 - b - \frac{p_j}{q_j}\right) \cdot q(\mathbb{B}) & < & \left(\frac{p_i}{q_i} - 1 - a\right)\cdot q(\mathbb{A}). \label{eq-ab-case-i}
      \end{eqnarray}
Let $b' \defeq 1 - p_j / q_j$ and $\Delta_b \defeq b' - b > 0$ (by definition of $\mathbb{B}$ \eqref{defBbis}). Reformulate \eqref{kktconst3} as
\begin{eqnarray}
    a q(\mathbb{A}) & = & b q(\mathbb{B}) + q(\mathbb{I}) - p(\mathbb{I}),\label{eq-kktconst3-b}
\end{eqnarray}
and rewrite the RHS:
\begin{eqnarray}
  \lefteqn{b q(\mathbb{B}) + q(\mathbb{I}) - p(\mathbb{I})}\nonumber\\
  & = & \left(1-\Delta_b - \frac{p_j}{q_j}\right) (q(\mathbb{B}\backslash \{j\}) + q_j) + q(\mathbb{I}) - p(\mathbb{I}) \nonumber\\
                                                  & = & \left(1 - \frac{p_j}{q_j}\right) q(\mathbb{B}\backslash \{j\}) + q(\mathbb{I}) - p(\mathbb{I}) + (q_j - p_j) - \Delta_b \cdot (q(\mathbb{B}\backslash \{j\}) + q_j) \nonumber\\
  & = & \underbrace{\left(1 - \frac{p_j}{q_j}\right) q(\mathbb{B}\backslash \{j\}) + q(\mathbb{I} \cup \{j\}) - p(\mathbb{I}\cup \{j\})}_{\defeq R_a} - \Delta_b \cdot q(\mathbb{B}). \label{eq-newb}
\end{eqnarray}
Note that $R_a$ is the RHS of \eqref{eq-kktconst3-b} for a new solution $(a',b') \in \textsc{c}(\ve{p}, \ve{q})$ where $b'$ has already been defined and $a' \defeq a + \delta_a$ is such that
\begin{eqnarray}
    (a + \delta_a) q(\mathbb{A}) & = & R_a,\label{eq-constR-b}
\end{eqnarray}
which means we need to guarantee that there is no change in $\mathbb{A}$ in the process of moving from $b$ to $b'$, i.e. $a + \delta_a < p_i / q_i - 1$. Replacing $R_a$ in \eqref{eq-newb} by its expression in \eqref{eq-constR-b} and using \eqref{eq-kktconst3-b} yields the sufficient conditions for $(a',b') \in \textsc{c}(\ve{p}, \ve{q})$:
\begin{eqnarray}
  \delta_a & = & \Delta_b\cdot \frac{q(\mathbb{B})}{q(\mathbb{A})} = \left(1 - \frac{p_j}{q_j} - b\right) \cdot \frac{q(\mathbb{B})}{q(\mathbb{A})} , \label{eq-def-deltaa} \\
  a + \delta_a & < & \frac{p_i}{q_i} - 1.  \label{eq-def-deltab} 
\end{eqnarray}
and we check that the inequality is \eqref{eq-ab-case-i}. To summarize, if $(a,b) \in  \textsc{c}(\ve{p}, \ve{q})$ and \eqref{eq-ab-case-i} holds, then the new breakpoint
\begin{eqnarray}
(a', b') & \defeq & \left(a + \left(1 - \frac{p_j}{q_j} - b\right) \cdot \frac{q(\mathbb{B})}{q(\mathbb{A})}, 1 - \frac{p_j}{q_j}\right) \label{eq-breakpoint-casei}
\end{eqnarray}
is in $\textsc{c}(\ve{p}, \ve{q})$. Also $\mathbb{A}$ does not change but we have the updates $\mathbb{B} \leftarrow \mathbb{B}\backslash \{j\}$ (one index less) and  $\mathbb{I} \leftarrow \mathbb{I}\cup \{j\}$.

    \noindent \textbf{Case 2}: suppose that the current breakpoint satisfies
    \begin{eqnarray}
\left(\frac{p_i}{q_i} - 1 - a\right)\cdot q(\mathbb{A}) & < & \left(1 - b - \frac{p_j}{q_j}\right) \cdot q(\mathbb{B}). \label{eq-ab-case-ii}
      \end{eqnarray}
Let $a' \defeq p_i / q_i - 1$ and $\Delta_a \defeq a' - a > 0$ (by definition of $\mathbb{A}$ \eqref{defAbis}). Reformulate \eqref{kktconst3} as
\begin{eqnarray}
    b q(\mathbb{B}) & = & a q(\mathbb{A}) - q(\mathbb{I}) + p(\mathbb{I}),\label{eq-kktconst3-a}
\end{eqnarray}
and rewrite the RHS:
\begin{eqnarray}
  \lefteqn{a q(\mathbb{A}) - q(\mathbb{I}) + p(\mathbb{I})}\nonumber\\
  & = & \left(\frac{p_i}{q_i} - 1 - \Delta_a\right) (q(\mathbb{A}\backslash \{i\}) + q_i) - q(\mathbb{I}) + p(\mathbb{I})\nonumber\\
                                                  & = & \left(\frac{p_i}{q_i} - 1\right) q(\mathbb{A}\backslash \{i\}) - q(\mathbb{I}) + p(\mathbb{I})+ (p_i - q_i) - \Delta_a \cdot (q(\mathbb{A}\backslash \{i\}) + q_i) \nonumber\\
  & = & \underbrace{\left(\frac{p_i}{q_i} - 1\right) q(\mathbb{A}\backslash \{i\}) - q(\mathbb{I} \cup \{i\}) + p(\mathbb{I}\cup \{i\})}_{\defeq R_b} - \Delta_a \cdot q(\mathbb{A}). \label{eq-newa}
\end{eqnarray}
Note that $R_b$ is the RHS of \eqref{eq-kktconst3-a} for a new solution $(a',b') \in \textsc{c}(\ve{p}, \ve{q})$ where $a'$ has already been defined and $b' \defeq b + \delta_b$ is such that
\begin{eqnarray}
    (b + \delta_b) q(\mathbb{B}) & = & R_b,\label{eq-constR-a}
\end{eqnarray}
which means we need to guarantee that there is no change in $\mathbb{B}$ in the process of moving from $a$ to $a'$, i.e. $b + \delta_b < 1 - p_j / q_j$. Replacing $R_b$ in \eqref{eq-newa} by its expression in \eqref{eq-constR-a} and using \eqref{eq-kktconst3-a} yields the sufficient conditions for $(a',b') \in \textsc{c}(\ve{p}, \ve{q})$:
\begin{eqnarray}
  \delta_b & = & \Delta_a \cdot \frac{q(\mathbb{A})}{q(\mathbb{B})} = \left(\frac{p_i}{q_i} - 1 - a\right) \cdot \frac{q(\mathbb{A})}{q(\mathbb{B})} , \label{eq-def-deltba} \\
  b + \delta_b & < & 1 - \frac{p_j}{q_j}. \label{eq-def-deltbb} 
\end{eqnarray}
and we check that the inequality is \eqref{eq-ab-case-ii}. To summarize, if $(a,b) \in  \textsc{c}(\ve{p}, \ve{q})$ and \eqref{eq-ab-case-ii} holds, then the new breakpoint
\begin{eqnarray*}
(a', b') & \defeq & \left(\frac{p_i}{q_i} - 1, b + \left(\frac{p_i}{q_i} - 1 - a\right) \cdot \frac{q(\mathbb{A})}{q(\mathbb{B})}\right)
\end{eqnarray*}
is in $\textsc{c}(\ve{p}, \ve{q})$. Also $\mathbb{B}$ does not change but we have the updates $\mathbb{A} \leftarrow \mathbb{A}\backslash \{i\}$ (one index less) and  $\mathbb{I} \leftarrow \mathbb{I}\cup \{i\}$.\\

    \noindent \textbf{Case 3}: suppose that the current breakpoint satisfies
    \begin{eqnarray}
\left(\frac{p_i}{q_i} - 1 - a\right)\cdot q(\mathbb{A}) & = & \left(1 - b - \frac{p_j}{q_j}\right) \cdot q(\mathbb{B}). \label{eq-ab-case-iii}
    \end{eqnarray}
We now work with the following equivalent to \eqref{kktconst3}: 
    \begin{eqnarray}
    -a q(\mathbb{A}) + b q(\mathbb{B}) - p(\mathbb{I}) + q(\mathbb{I}) & = & 0, \label{eq-kktconst3-ab}
    \end{eqnarray}
    and we now consider $a' \defeq p_i / q_i - 1, b' \defeq p_j / q_j$ simultaneously. Rewrite the LHS of \eqref{eq-kktconst3-ab} using both \eqref{eq-newb} and \eqref{eq-newa} with their notations as
    \begin{eqnarray}
      \lefteqn{-a q(\mathbb{A}) + b q(\mathbb{B}) - p(\mathbb{I}) + q(\mathbb{I})}\nonumber\\
      & = & - \left(\frac{p_i}{q_i} - 1\right) q(\mathbb{A}\backslash \{i\}) - p(\mathbb{I})  + q(\mathbb{I}) - p_i + q_i + \Delta_a \cdot q(\mathbb{A}) \nonumber\\
       & & + \left(1 - \frac{p_j}{q_j}\right) q(\mathbb{B}\backslash \{j\})  + q_j - p_j - \Delta_b \cdot q(\mathbb{B}),\label{eq-ab}
    \end{eqnarray}
and we check that \eqref{eq-ab-case-iii} implies $\Delta_a \cdot q(\mathbb{A}) = \Delta_b \cdot q(\mathbb{B})$, so with
\begin{eqnarray*}
(a', b') & \defeq & \left(\frac{p_i}{q_i} - 1, 1 - \frac{p_j}{q_j}\right)
\end{eqnarray*}
we check that \eqref{eq-ab} becomes
\begin{eqnarray}
      -a q(\mathbb{A}) + b q(\mathbb{B}) - p(\mathbb{I}) + q(\mathbb{I}) & = & -  a' q(\mathbb{A}')  + b' q(\mathbb{B}')  - p(\mathbb{I}') + q(\mathbb{I}') ,\label{eq-ab2}
\end{eqnarray}
with $\mathbb{A}' \defeq \mathbb{A} \backslash \{i\}$, $\mathbb{B}' \defeq \mathbb{B} \backslash \{j\}$ and $\mathbb{I}' \defeq \mathbb{I} \cup \{j, i\}$ the updates to make, and since \eqref{eq-ab2} is 0 from \eqref{eq-kktconst3-ab}, $(a', b') \in \textsc{c}(\ve{p}, \ve{q})$. This achieves the proof of Theorem \ref{thm-cab-from-c-breakpoints}.

\subsection{Proof of Lemma \ref{lem-b-from-a-and-c-breakpoints}}\label{sec-proof-lem-b-from-a-and-c-breakpoints}

We prove Lemma \ref{lem-b-from-a-and-c-breakpoints} by using the proof of Theorem \ref{thm-cab-from-c-breakpoints} in Section \ref{sec-proof-thm-cab-from-c-breakpoints}. If we are not on a breakpoint (otherwise, the algorithm returns the breakpoint), we have two cases:

\noindent \textbf{Case 1}: We query $\tilde{a}$ and ask for $\tilde{b}$ such that $(\tilde{a},\tilde{b}) \in \mathcal{C}(\ve{p}, \ve{q})$. Suppose we have $a<\tilde{a}<a'$ for two consecutive breakpoints $(a,b)$ and $(a',b')$ in $\textsc{c}(\ve{p}, \ve{q})$. We have three subcases, where indexes $i,j$ are defined in the proof of Theorem \ref{thm-cab-from-c-breakpoints}:\\

\noindent \textbf{Subcase 1.1}: we have
\begin{eqnarray*}
(a', b') & \defeq & \left(a + \left(1 - \frac{p_j}{q_j} - b\right) \cdot \frac{q(\mathbb{B})}{q(\mathbb{A})}, 1 - \frac{p_j}{q_j}\right), 
\end{eqnarray*}
so we reuse the proof of Theorem \ref{thm-cab-from-c-breakpoints}, Case 1. Remark that as long as $\delta'_b = \epsilon \cdot \Delta_b$ for $b' \defeq b + \delta'_b$ with $0<\epsilon<1$, the RHS of \eqref{eq-kktconst3-b} is
\begin{eqnarray*}
  b q(\mathbb{B}) + q(\mathbb{I}) - p(\mathbb{I}) & = & b' q(\mathbb{B}) + q(\mathbb{I}) - p(\mathbb{I}) - \epsilon \cdot \Delta_b \cdot q(\mathbb{B}). \label{eq-newb-2}
\end{eqnarray*}
while the LHS becomes $(a' - \delta'_a) q(\mathbb{A}) = - \delta'_a q(\mathbb{A}) + a'  q(\mathbb{A})$. Since $\epsilon<1$, $\mathbb{A}, \mathbb{B}, \mathbb{I}$ do not change and if we ensure (i) $\delta'_a \leq \delta_a$ \eqref{eq-def-deltaa} and (ii) $\delta'_a q(\mathbb{A}) =  \epsilon \cdot \Delta_b \cdot q(\mathbb{B})$, yielding
\begin{eqnarray*}
  \delta'_a & = & \epsilon \cdot \Delta_b \cdot \frac{q(\mathbb{B})}{q(\mathbb{A})},
\end{eqnarray*}
and we check $\delta'_a \leq \delta_a$. Solving for $\epsilon$ while $\delta'_a \defeq \tilde{a} - a$ yields $\epsilon = (\tilde{a} - a) q(\mathbb{A}) / (\Delta_b \cdot q(\mathbb{B}))$ and the solution
\begin{eqnarray*}
\left(\tilde{a}, \tilde{b} \defeq b + (\tilde{a} - a) \cdot \frac{q(\mathbb{A})}{q(\mathbb{B})}\right) \in \mathcal{C}(\ve{p}, \ve{q}),
\end{eqnarray*}
where $\mathbb{A}, \mathbb{B}$ are associated to breakpoint $(a,b)$.\\

\noindent \textbf{Subcase 1.2}: we have
\begin{eqnarray*}
(a', b') & \defeq & \left(\frac{p_i}{q_i} - 1, b + \left(\frac{p_i}{q_i} - 1 - a\right) \cdot \frac{q(\mathbb{A})}{q(\mathbb{B})}\right)
\end{eqnarray*}
so we reuse the proof of Theorem \ref{thm-cab-from-c-breakpoints}, Case 2. Remark that as long as $\delta'_b = \epsilon \cdot \Delta_a q(\mathbb{A}) / q(\mathbb{B})$ for $b' \defeq b + \delta'_b$ with $0<\epsilon<1$, the RHS of \eqref{eq-kktconst3-a} is
\begin{eqnarray*}
  a q(\mathbb{A}) - q(\mathbb{I}) + p(\mathbb{I}) & = & a' q(\mathbb{A}) - q(\mathbb{I}) + p(\mathbb{I})  - \epsilon \cdot  \Delta_a \cdot q(\mathbb{A}). \label{eq-newa-2}
\end{eqnarray*}
while the LHS becomes $(b' - \delta'_b) q(\mathbb{B}) = - \delta'_b q(\mathbb{B}) + b'  q(\mathbb{B})$. Since $\epsilon<1$, $\mathbb{A}, \mathbb{B}, \mathbb{I}$ do not change and if we ensure (i) $\delta'_b \leq \delta_b$ \eqref{eq-def-deltba} and (ii) $\delta'_b q(\mathbb{B}) =  \epsilon \cdot \Delta_a \cdot q(\mathbb{A})$, yielding
\begin{eqnarray*}
  \delta'_b & = & \epsilon \cdot \Delta_a \cdot \frac{q(\mathbb{A})}{q(\mathbb{B})},
\end{eqnarray*}
and we check $\delta'_b \leq \delta_b$. This time, $\Delta_a$ maps $a$ to the next breakpoint in the proof of Theorem \ref{thm-cab-from-c-breakpoints} so we have $\epsilon \cdot \Delta_a = \tilde{a} - a$ and we get that the solution
\begin{eqnarray*}
\left(\tilde{a}, \tilde{b} \defeq b + (\tilde{a} - a) \cdot \frac{q(\mathbb{A})}{q(\mathbb{B})}\right) \in \mathcal{C}(\ve{p}, \ve{q}),
\end{eqnarray*}
where $\mathbb{A}, \mathbb{B}$ are associated to breakpoint $(a,b)$.\\

\noindent \textbf{Subcase 1.3} is Theorem \ref{thm-cab-from-c-breakpoints}, Case 3, and yields the same solution as the two preceding cases.\\

\noindent \textbf{Case 2}: We query $\tilde{b}$ and ask for $\tilde{a}$ such that $(\tilde{a},\tilde{b}) \in \mathcal{C}(\ve{p}, \ve{q})$. This time, we suppose we have $b<\tilde{b}<b'$ for two consecutive breakpoints $(a,b)$ and $(a',b')$ in $\textsc{c}(\ve{p}, \ve{q})$. Sparing all the computation, we get this time
\begin{eqnarray*}
\left(\tilde{a}\defeq a + (\tilde{b} - b) \cdot \frac{q(\mathbb{B})}{q(\mathbb{A})}, \tilde{b} \right) \in \mathcal{C}(\ve{p}, \ve{q}),
\end{eqnarray*}
where $\mathbb{A}, \mathbb{B}$ are associated to breakpoint $(a,b)$. This ends the proof of Lemma \ref{lem-b-from-a-and-c-breakpoints}.

\subsection{Proof of Lemma \ref{lem-cont-ab}}\label{sec-proof-lem-cont-ab}

All properties except continuity are immediate consequences of Lemma \ref{lem-b-from-a-and-c-breakpoints}. For the continuity part, pick any $a$ satisfying \eqref{kktconst1}. For $b=0$, we necessarily have $-a q(\mathbb{A}) + b q(\mathbb{B}) = -a q(\mathbb{A}) < p(\mathbb{I}) - q(\mathbb{I})$ since  $p(\mathbb{I}) - q(\mathbb{I}) = \sum_{i: q_i \leq p_i \leq (1+a) q_i} p_i -q_i \geq 0$. Elements in $\mathbb{A}$ have $p_i > (1+a) q_i$, which yields after summing and taking negation $-a q(\mathbb{A}) > q(\mathbb{A})  - p(\mathbb{A})$ so for the other extreme case, $b=1$, since we have in this case  $p(\mathbb{I}) - q(\mathbb{I}) = (1 - p(\mathbb{A})) - (1 - q(\mathbb{A})) = q(\mathbb{A}) - p(\mathbb{A})$, we observe $-a q(\mathbb{A}) + 0 > p(\mathbb{A}) - q(\mathbb{A}) = p(\mathbb{I}) - q(\mathbb{I})$. Now, for $b\in (0,1)$, reformulate \eqref{kktconst3} as
\begin{eqnarray}
    a q(\mathbb{A}) & = & b q(\mathbb{B}) + q(\mathbb{I}) - p(\mathbb{I}),\label{kktconst3b}
\end{eqnarray}
and remark that having chosen $a$, the LHS is fixed. Suppose the current $b$ is at $(1-b) = (p_i/q_i) + \delta$ with $\delta>0$, meaning $i\in \mathbb{B}$, and suppose no other element of $\mathbb{B}$ is closer. Suppose $\delta$ small enough so that when we increase $b$, $q(\mathbb{I}) - p(\mathbb{I})$ does not change. In this case, the RHS of \eqref{kktconst3b} equals
\begin{eqnarray*}
  b q(\mathbb{B}) + q(\mathbb{I}) - p(\mathbb{I}) & = & \left(1-\delta - \frac{p_i}{q_i}\right) (q(\mathbb{B}\backslash \{i\}) + q_i) + q(\mathbb{I}) - p(\mathbb{I}) \\
                                                  & = & \left(1 - \frac{p_i}{q_i}\right) q(\mathbb{B}\backslash \{i\}) + q(\mathbb{I}) - p(\mathbb{I}) + (q_i - p_i) - \delta \cdot (q(\mathbb{B}\backslash \{i\}) + q_i) \\
  & = & \left(1 - \frac{p_i}{q_i}\right) q(\mathbb{B}\backslash \{i\}) + q(\mathbb{I} \cup \{i\}) - p(\mathbb{I}\cup \{i\}) - \delta \cdot (q(\mathbb{B}\backslash \{i\}) + q_i),
\end{eqnarray*}
and as $b$ continuously increases further while $\delta >0$ decreases, the RHS increases, but the limit of the RHS as $\delta \rightarrow 0^+$ is $\left(1 - (p_i/q_i)\right) q(\mathbb{B}\backslash \{i\}) + q(\mathbb{I} \cup \{i\}) - p(\mathbb{I}\cup \{i\})$, which, in fact, is the RHS of \eqref{kktconst3b} as $i$ goes from $\mathbb{B}$ to $\mathbb{I}$. What we just showed is that the RHS of \eqref{kktconst3b} has continuous variations with $b\in [0,1]$. Since $\textsc{c} (\ve{p}, \ve{q}) \defeq \{(a_i, b_i) : i \in [N]\}$ has its elements indexed in strictly increasing values of both $a$ and $b$ with $(a_1, b_1) = (0,0)$, this achieves the proof of the Lemma (choosing $b$ first yields the same proof).

\subsection{Proof of Theorem \ref{thm-opt-f}}\label{sec-proof-thm-opt-f}

As it is formulated, the problem can be conveniently solved by addressing the dual of \eqref{def-pb-f-md-2-general}: 
 \begin{eqnarray*}
      \textsc{dmd}^2_f(\ve{p}, \ve{q}; D) & \defeq & \arg\min _{\ve{r} \in [0,1]^n, \ve{s} \in  \Delta_n} D_{f} (\ve{p} \odot \ve{r} + (1-\ve{p}^\top \ve{r})\cdot \ve{s} \|\ve{q})  \quad \mbox{s.t. }  \ve{p}^\top \ve{r} \geq P, 
\end{eqnarray*}  
 for $P > \pacc(SD)$.

We first prove the (strict) convexity part. For any given outputs $\ve{p}, \ve{q}$ of the drafter and target, consider three distinct acceptance probabilities $P_.$, for $P_b \defeq \gamma P_a + (1-\gamma) P_c$ and $P_a \leq P_b \leq P_c$. Denote $\ve{r}_a, \ve{s}_a$; $\ve{r}_b, \ve{s}_b$ and $\ve{r}_c, \ve{s}_c$ the respective optimal solutions components. We need to show
\begin{eqnarray}
  D_{f} (\ve{p}\odot \ve{r}_b + (1-\ve{p}^\top \ve{r}_b)\cdot \ve{s}_b \|\ve{q}) & \leq & \gamma \cdot D_{f} (\ve{p}\odot \ve{r}_a + (1-\ve{p}^\top \ve{r}_a)\cdot \ve{s}_a \|\ve{q}) \\
  & & + (1-\gamma) \cdot D_{f} (\ve{p}\odot \ve{r}_c + (1-\ve{p}^\top \ve{r}_c)\cdot \ve{s}_c \|\ve{q}),\label{eq-cvx2}
\end{eqnarray}
and this will hold if we can find a \textit{feasible} solution for $P_b$ that can be put in the LHS (the corresponding optimal one cannot increase the divergence by definition). First, pick
\begin{eqnarray}
  \tilde{\ve{r}} & \defeq & \gamma \cdot \ve{r}_a + (1-\gamma) \cdot \ve{r}_c.\label{defrtilde}
\end{eqnarray}
Since $\ve{p}^\top \ve{r}_a \geq P_a$ and $\ve{p}^\top \ve{r}_c \geq P_c$, we have $\ve{p}^\top \tilde{\ve{r}} \geq P_b$ and also $\tilde{\ve{r}} \in [0,1]^n$. To find the corresponding feasible $\tilde{\ve{s}} \in \Delta_n$, we want it to satisfy
\begin{eqnarray*}
(1 - \ve{p}^\top \tilde{\ve{r}} ) \cdot\tilde{\ve{s}} = \gamma (1-\ve{p}^\top \ve{r}_a ) \cdot\ve{s}_a + (1-\gamma) (1-\ve{p}^\top \ve{r}_c ) \cdot \ve{s}_c,
\end{eqnarray*}
but since $1 - \ve{p}^\top \tilde{\ve{r}} = \gamma (1-\ve{p}^\top \ve{r}_a ) +  (1-\gamma) (1-\ve{p}^\top \ve{r}_c ) $, we may choose $\tilde{\ve{s}} $ as:
\begin{eqnarray}
\tilde{\ve{s}} & \defeq & \epsilon \cdot \ve{s}_a + (1-\epsilon) \cdot \ve{s}_c, \quad \epsilon \defeq \gamma \cdot \frac{1-\ve{p}^\top \ve{r}_a}{1 - \ve{p}^\top \tilde{\ve{r}}}, \label{defstilde}
  \end{eqnarray}
  and we have $\tilde{\ve{s}} \in \Delta_n$. $\tilde{\ve{r}}, \tilde{\ve{s}}$ as in \eqref{defrtilde}, \eqref{defstilde} is feasible and because of the convexity of $f$, we get the inequality (strict if $\gamma \neq 0,1$ and $f$ is strictly convex) in
  \begin{eqnarray*}
    \lefteqn{D_{f} (\ve{p}\odot \tilde{\ve{r}} + (1-\ve{p}^\top \tilde{\ve{r}})\cdot \tilde{\ve{s}}  \|\ve{q})}\\
    & = & D_{f} \left(\gamma \cdot(\ve{p}\odot \ve{r}_a + (1-\ve{p}^\top \ve{r}_a)\cdot \ve{s}_a) + (1-\gamma) \cdot(\ve{p}\odot \ve{r}_c + (1-\ve{p}^\top \ve{r}_c)\cdot \ve{s}_c) \|\ve{q}\right) \\
    & \leq & \gamma \cdot D_{f} (\ve{p}\odot \ve{r}_a + (1-\ve{p}^\top \ve{r}_a)\cdot \ve{s}_a \|\ve{q})  + (1-\gamma) \cdot D_{f} (\ve{p}\odot \ve{r}_c + (1-\ve{p}^\top \ve{r}_c)\cdot \ve{s}_c \|\ve{q}),
  \end{eqnarray*}
  and of course $D_{f} (\ve{p}\odot \ve{r}_b + (1-\ve{p}^\top \ve{r}_b)\cdot \ve{s}_b \|\ve{q}) \leq D_{f} (\ve{p}\odot \tilde{\ve{r}} + (1-\ve{p}^\top \tilde{\ve{r}})\cdot \tilde{\ve{s}}  \|\ve{q})$ (the optimum cannot be worse than any feasible solution), which shows \eqref{eq-cvx2} and ends the proof of the (strict) convexity part.\\

  Let us now tackle the right derivative part. Without loss of generality, indexes are ordered such that $p_{i+1} / q_{i+1}\geq p_{i} / q_{i}, \forall i$ and all ratios are distinct (any of the $n$ distinct ratio values is called a "tick"). For a current optimal solution given by $z_\beta \in (\min_i p_i / q_i, 1] , z_\alpha \in [1, \max_i p_i / q_i)$, the corresponding value of the $f$-divergence is:
  \begin{eqnarray}
    D_f(\ve{\pi}\| \ve{q})  & = & \sum_{\mathbb{I}_\beta} q_i f\left(z_\beta\right) + \sum_{\mathbb{I}} q_i f\left(\frac{p_i}{q_i}\right) + \sum_{\mathbb{I}_\alpha} q_i f\left(z_\alpha\right),\label{eqvalf}
  \end{eqnarray}
  with the three sets of ticks
  \begin{eqnarray}
    \mathbb{I}_\beta \defeq \{i : p_i/q_i < z_\beta\} \quad; \quad  \mathbb{I} \defeq \{i : z_\beta  \leq p_i/q_i \leq z_\alpha\}   \quad; \quad \mathbb{I}_\alpha \defeq \{i : z_\alpha  < p_i/q_i \}. \label{defallI}
  \end{eqnarray}
  Suppose we pick $\delta_\beta > 0, \delta_\alpha > 0$ such that the choice
  \begin{eqnarray}
    z'_\beta & \defeq & z_\beta - \delta_\beta, \label{eq-def-zpbeta}\\
    z'_\alpha & \defeq & z_\alpha + \delta_\alpha \label{eq-def-zpalpha}
  \end{eqnarray}
  satisfies
\begin{itemize}
\item  [(i)] it yields a new optimal solution,
\item [(ii)] sets $\mathbb{I}_\beta, \mathbb{I}, \mathbb{I}_\alpha$ do not change.
\end{itemize}
Let us compute the variation of the acceptance probability $P$ and the variation of the $f$-divergence as a function of this variation. We get for $P$, 
\begin{eqnarray}
P(z'_\alpha)  = P(z_\alpha) + \delta_\alpha q(\mathbb{I}_\alpha) & ; & P(z'_\beta)  = P(z_\beta) + \delta_\beta q(\mathbb{I}_\beta), \label{defPP}
\end{eqnarray}
(where $q(\mathbb{U} \subseteq [n]) \defeq \sum_{i \in \mathbb{U}} q_i$) so for (i) to hold we must have the relationship between $\delta_\alpha$ and $\delta_\beta$
 \begin{eqnarray}
\delta_\alpha q(\mathbb{I}_\alpha) & = & \delta_\beta q(\mathbb{I}_\beta).\label{relalphabeta}
 \end{eqnarray}
(and we also need to assume $P(z'_\alpha) = P(z'_\beta) < 1$). Recall $\pacc(SD) \defeq \sum_i \min\{p_i, q_i\}$ and denote $D_f(P)$ the optimal value of the $f$-divergence for the requested $P$. Under Assumption \ref{assum-mda}, $p_i / q_i \neq 1, \forall i \in [n]$.  To compute the right derivative at $\pacc(SD) $, $(D_f)'_r(\pacc(SD))$, we start from $z_\alpha = z_\beta = 1$ (Note that $\mathbb{I}_\beta \cup \mathbb{I}_\alpha = [n]$ in \eqref{defallI}) and then compute a variation ($z'_\alpha = z_\alpha + \delta_\alpha$, negative for $z'_\beta = z_\beta - \delta_\beta$, $\delta_\alpha, \delta_\beta > 0$), small enough not to change $\mathbb{I}_\beta$ and $\mathbb{I}_\alpha$. We know from \eqref{defPP} that we must have the relationship $\delta_\beta q_\beta = \delta_\alpha(1-q_\beta)$ with $q_\beta \defeq q(\mathbb{I}_\beta )$ \eqref{defPP} for optimality to hold for a new value $P'$ of the acceptance probability. Keeping $f(1)$ (=0) in expressions for clarity, we compute the ratio
  \begin{eqnarray*}
    \frac{D_f(P')-D_f(\pacc(SD))}{P' - \pacc(SD)} & = & \frac{\sum_{\mathbb{I}_\beta} q_i f \left(1-\delta_\beta\right) + \sum_{\mathbb{I}_\alpha} q_i f \left(1+\delta_\alpha\right) - f(1)}{\pacc(SD) +  \delta_\alpha(1-q_\beta) - \pacc(SD)}\\
                                                                        & = & \frac{q_\beta \cdot \left( f\left(1 - \frac{ \delta_\alpha(1-q_\beta)}{q_\beta}\right) - f(1)\right)}{\delta_\alpha(1-q_\beta) } + \frac{(1-q_\beta) \cdot \left( f \left(1+\delta_\alpha\right) - f(1) \right)}{\delta_\alpha(1-q_\beta)}\\
     & = & - \frac{f(1) - f(1-\tilde{\delta}_\alpha)}{\tilde{\delta}_\alpha} + \frac{ f \left(1+\delta_\alpha\right) - f(1) }{\delta_\alpha}
  \end{eqnarray*}
  where $\tilde{\delta}_\alpha \defeq \delta_\alpha(1-q_\beta) / q_\beta$. Now we pass to the limit:
  \begin{eqnarray}
    (D_f)'_r(\pacc(SD)) & \defeq & \lim_{P' \searrow \pacc(SD)} \frac{D_f(P')-D_f(\pacc(SD))}{P' - \pacc(SD)}\nonumber\\
                                   & =& -\lim_{\tilde{\delta}_\alpha \searrow 0} \frac{f(1) - f(1-\tilde{\delta}_\alpha)}{\tilde{\delta}_\alpha} + \lim_{\delta_\alpha \searrow 0} \frac{ f \left(1+\delta_\alpha\right) - f(1) }{\delta_\alpha}\nonumber\\
     & = & - f'_l(1) + f'_r(1) \label{eq-diff-der}
  \end{eqnarray}
  where $f'_l$ is the left derivative and $f'_r$ the right derivative, that must exist because $f$ is convex over an open set so differentiable anywhere except maybe on a set of measure zero \citep[Theorem 25.5]{rCA}, and any point of non-differentiability $z$ has a subdifferential which is exactly $[f'_l(z), f'_r(z)]$, yielding for \eqref{eq-diff-der} $- f'_l(1) + f'_r(1) = \max \partial f(1) - \min \partial f(1)$, as claimed.

\subsection{Proof of Lemma \ref{lem-top-k}}\label{sec-proof-lem-top-k}
  
Without loss of generality, we assume masking $\ve{q}$ is accompanied by a renormalization of the top-$k$ coordinates for the sake of the proof. We immediately remark that if $\lim_{z \rightarrow +\infty} f(z)/z = +\infty$, any solution where $q_i = 0$ and $\pi_i > 0$ enforces $D_f(\ve{\pi}\| \ve{q}) = +\infty > D$ and is thus not feasible, so the Lemma is proven. Suppose $\lim_{z \rightarrow +\infty} f(z)/z = \ell \ll +\infty$. Denote $\Pi \defeq \sum_{i \in \mathcal{I}} \pi_i$ with $\mathcal{I} \defeq \{i : q_i = 0 \wedge \pi_i > 0\}$, so that
\begin{eqnarray}
  D_f(\ve{\pi}\| \ve{q}) & \defeq & \sum_{i \not\in \mathcal{I}} q_i f\left(\frac{\pi_i}{q_i}\right) + \ell \cdot \Pi.\label{eq-def-f-div-gen}
\end{eqnarray}
$\ve{q}$ being renormalized over the top-$k$ coordinates, Lemma \ref{lem-slater} still applies. We show how the proof of Theorem \ref{th-opt1} adapts via a simple change of parameter in the computation of $\alpha$ and $\beta$ in Lemma \ref{lemAB}. \eqref{kkt8} now becomes:
\begin{eqnarray*}
  \frac{\partial \mathcal{L}_1}{\partial t_i} = \lambda \cdot \sum_j C_j \cdot s_j - \lambda \cdot C_i  - 1 +\nu_i & = & 0, \forall i, \\
  \frac{\partial \mathcal{L}_1}{\partial s_i} = \mu - \lambda \cdot C_i \cdot (1-\ve{1}^\top \ve{t}) - \chi_i & = & 0, \forall i, 
\end{eqnarray*}
with
\begin{eqnarray*}
  C_i & \defeq &  (-f') \left(\frac{\pi_i}{q_i}\right) \cdot \iver{i\not\in \mathcal{I}}  - \iver{i\in \mathcal{I}},
\end{eqnarray*}
so we just have to replace $\alpha$ and $\beta$ by
\begin{eqnarray*}
\alpha & = & \expect_{i \sim \ve{u}} \left[C_i\right], \quad\mbox{ with }\ve{u} \defeq \frac{1}{1-\ve{1}^\top\ve{t}} \cdot (\ve{p}-\ve{t}) \in \Delta_n,\\
\beta & = & \expect_{i \sim \ve{s}} \left[C_i\right],
\end{eqnarray*}
and the rest of the proof of Theorem \ref{th-opt1} follows. Notice however that
\begin{eqnarray*}
\alpha & = & (1-u(\mathcal{I})) \cdot \expect_{i \sim \tilde{\ve{u}}} \left[(-f') \left(\frac{\pi_i}{q_i}\right)\right] - u(\mathcal{I}),
\end{eqnarray*}
where $u(\mathcal{I}) \defeq \sum_{i \in \mathcal{I}} u_i$ and $\tilde{\ve{u}}$ is distribution $\ve{u}$ masked and renormalized to support in the top-$k$ coordinates. A similar reformulation holds for $\beta$. So instead of $L_\alpha(-g)$ in \eqref{isAB-nonconvex}, we have to consider $L_\alpha(-\gamma g - (1-\gamma))$ for some $\gamma \in [0,1]$, which slightly extends the possible range of values to search in, and the same happens for $ L_\beta(-h)$. In the end, $\clamps(\ve{p}, L_\beta(-h) \cdot \ve{q}, L_\alpha(-g) \cdot \ve{q})$ in \eqref{isAB-nonconvex} becomes $\clamps(\ve{p}, L_\beta(-\gamma' h - (1-\gamma')) \cdot \ve{q}, L_\alpha(-\gamma g - (1-\gamma))  \cdot \ve{q})$ for some $\gamma, \gamma' \in [0,1]$. The set of optimal solutions looks more complicated but keeps the fundamental property that coordinate $i$ is necessarily $0$ if $i \in \mathcal{I}$, and we do not look for all optimal solutions but just for one whose support coincides with the mask. If $f$ is strictly convex, we still observe that the optimal solution $\ve{\pi}$ of \eqref{def-pb-f-md-1-general} has the same support as the top-$k$ coordinates of $\ve{q}$. If $f$ is not, we easily check that \textit{there exists} optimal solutions $\ve{\pi}$ of \eqref{def-pb-f-md-1-general} with the same support as the top-$k$ coordinates of $\ve{q}$. This ends the proof of Lemma \ref{lem-top-k}.

\subsection{Proof of Lemma \ref{lem-ba-exp}}\label{sec-proof-lem-ba-exp}

Fix any $t \in \{2, ..., T-1\}$. Since $\mu_{t+1} = \expect_{t+1}(t+1)$, note the dependence between $\mu_{t+1} = \expect_{t+1}(t+1)$ and $\mu_{t} = \expect_{t}(t-1)$:
\begin{eqnarray}
  \mu_{t+1} & = & \frac{\expect_{t}(t+1) - \mu_{t} \cdot \expect_{t}(t ,t+1)}{1 - \mu_{t}^2 } .\label{eq-dep-ttm1}
\end{eqnarray}
We now want a geometric progression on edges:
\begin{eqnarray}
     \mu_{t+1} & \geq&  (1+\delta) \mu_{t}.\label{eq-geom-edge}
   \end{eqnarray}
This is equivalent, from \eqref{eq-dep-ttm1}, to requesting
\begin{eqnarray*}
  \mu_{t} \cdot \expect_{t}(t ,t+1) & \leq & \expect_{t}(t+1) - (1+\delta)  \cdot \mu_{t}(1 - \mu_{t}^2),
\end{eqnarray*}
and since we assume $\expect_{t}(t+1) \geq (1-\beta) \mu_{t}$ (2.), it is sufficient to request $\mu_{t} \cdot \expect_{t}(t ,t+1) \leq (1-\beta) \mu_{t} - (1+\delta)  \cdot \mu_{t}(1 - \mu_{t}^2)$, which after simplification ($\mu_{t} > 0$ (1.)) gives the sufficient condition
\begin{eqnarray*}
  \expect_{t}(t ,t+1) & \leq & (1+\delta)  \cdot (1 - \mu_{t}^2) - (\beta + \delta)\\
   & & = 1 - \beta -  (1+\delta) \mu_{t}^2,
\end{eqnarray*}
which is (3.). We then get for the boosting advantage:
\begin{eqnarray*}
  A\left(\{\ve{h}_{t}\}_{t\in [T]}\right) & = & \sum_{t=1}^T \mu_t^2\\
                                          & \geq & \mu_1^2 \cdot \sum_{t=1}^T (1+\delta)^{2t} \\
   &  & =  \mu_1^2 \cdot \frac{(1+\delta)^{2T} - 1}{2 \delta + \delta^2},
\end{eqnarray*}
as claimed.

\subsection{Proof of Lemma \ref{lem-subopt-greedy}}\label{sec-proof-lem-subopt-greedy}

  Denote for short $a \defeq \mu_{t}$ and $b \defeq \tilde{\mu}_{t}$. Remark that
  \begin{eqnarray*}
  B(c,d) & \defeq & \mu^2_{t} + \mu^2_{t+1}\\
                      & = & \mu^2_{t} + \left(\frac{\expect_{t}(d) - \mu_{t} \expect_{t}(c, d)}{1 - \mu_{t}^2}\right)^2\\
                      & = & \mu^2_{t} + \left(\frac{\tilde{\mu}_{t} - \mu_{t} \expect_{t}(c, d)}{1 - \mu_{t}^2}\right)^2\\
                      & = & a^2 + \left(\frac{b-a \cdot \expect_{t}(c, d)}{1-a^2}\right)^2,\\
  \end{eqnarray*}
  and similarly
    \begin{eqnarray*}
  B(d,c) & = & b^2 + \left(\frac{a - b\cdot \expect_{t}(c, d)}{1-b^2}\right)^2.
    \end{eqnarray*}
Let $C \defeq \expect_{t}(c, d) - ab$ be the covariance between the sequence of $\ve{y}_{i}^\top \tilde{\ve{h}}_{c}(\bm{x}_i)$ and $\ve{y}_{i}^\top \tilde{\ve{h}}_d(\bm{x}_i)$ computed using $\ve{w}_{t}$. We remark the simplification:
\begin{eqnarray*}
  \left(\frac{b-a \cdot \expect_{t}(c, d)}{1-a^2}\right)^2 & = & \left(\frac{b-aC - a^2b}{1-a^2}\right)^2\\
  & = & \left(b - \frac{aC}{1-a^2}\right)^2
\end{eqnarray*}
and similarly $\left(\frac{a - b\cdot \expect_{t}(c, d)}{1-b^2}\right)^2 = \left(a - \frac{bC}{1-b^2}\right)^2$. We compute the difference and get after factoring
\begin{eqnarray*}
  \lefteqn{  B(c,d) - B(d,c))}\\
  & = &  a^2 - b^2 + \left(b - \frac{aC}{1-a^2}\right)^2 - \left(a - \frac{bC}{1-b^2}\right)^2\\
  & = &  \underbrace{\frac{1}{(1-a^2b^2)(1-a^2)^2(1-b^2)^2} \cdot (a^2 - b^2)}_{\defeq A} \cdot \underbrace{C \cdot \left(C-\frac{2ab (1-a^2)(1-b^2)}{(1-a^2b^2)} \right)}_{\defeq D}.
\end{eqnarray*}
Under the greedy fit scenario, $A > 0$ so the difference is $<0$ iff $D<0$. We also remark 
\begin{eqnarray*}
 0 \leq \rho \defeq \frac{2a|b| (1-a^2)(1-b^2)}{(1-a^2b^2)}\leq 1, \forall a \in (0,1], b \in [-1,1],
\end{eqnarray*} 
and with a bit more analytical analysis, we get that the upperbound can be replaced by $6 - 4\sqrt{2} < 0.35$. We know that $a>0$, so if $b>0$ then $D<0$ iff $C \in (0, \rho)$ while if $b<0$,  then $D<0$ iff $C \in (-\rho, 0)$, as claimed. 
 
\end{document}